\documentclass{article}
\usepackage[utf8]{inputenc}

\title{Learning GMMs with Nearly Optimal Robustness Guarantees}
\author{Allen Liu \thanks{Email: \texttt{cliu568@mit.edu}. This work was supported in part by an NSF Graduate Research Fellowship, a Fannie and John Hertz Foundation Fellowship and Ankur Moitra's NSF CAREER Award CCF-1453261 and NSF Large CCF1565235.}\and Ankur Moitra \thanks{Email: \texttt{moitra@mit.edu}. This work was
supported in part by a Microsoft Trustworthy AI Grant, NSF CAREER Award CCF-1453261, NSF Large CCF1565235, a David and Lucile Packard Fellowship and an ONR Young Investigator
Award.}}
\date{\today}

\input{preamble.sty}

\begin{document}


\maketitle

\begin{abstract}
    In this work we solve the problem of robustly learning a high-dimensional Gaussian mixture model with $k$ components from $\epsilon$-corrupted samples up to accuracy $\widetilde{O}(\epsilon)$ in total variation distance for any constant $k$ and with mild assumptions on the mixture. This robustness guarantee is optimal up to polylogarithmic factors. The main challenge is that most earlier works rely on learning individual components in the mixture, but this is impossible in our setting, at least for the types of strong robustness guarantees we are aiming for. Instead we introduce a new framework which we call {\em strong observability} that gives us a route to circumvent this obstacle.

\end{abstract}

\newpage

\section{Introduction}

\subsection{Background}



Gaussian mixture models have a long and storied history. They were first introduced in the pioneering work of Karl Pearson \cite{Pea94} in 1894 and have found wide-ranging applications ever since, as a natural model for data believed to be coming from two or more heterogeneous sources. Early works focused on the statistical complexity \cite{Tei61}, namely bounding the number of samples needed to estimate the Gaussian mixture model to within some desired accuracy. More recently, these problems have been revisited with an emphasis on giving computationally efficient algorithms that work in high dimensions and with minimal assumptions \cite{Das99,kalai2010efficiently, moitra2010settling, BS10,ge2015learning}. 

There are different types of learning goals we could ask for:
\begin{enumerate}
    \item[(1)] In {\em parameter learning}, we want to estimate the mixture on a component-by-component basis. We ask that there is a matching between the components in our hypothesis and those of the true mixture so that across the matching we are close in total variation distance and get the mixing weights approximately correct. Alternatively we could ask to be close in an appropriate parameter distance instead. 
    
    \item[(2)]  In {\em proper density estimation}, we relax the goal of estimating the individual components. Rather, we want to output a hypothesis from the correct family (i.e. a Gaussian mixture model) and we require that it is statistically close as a distribution to the true mixture. 
    
    \item[(3)] Finally in {\em improper density estimation} the setup is the same as above except that we allow ourselves to output any hypothesis, even if it is not from the family we are trying to learn. 
    
\end{enumerate} 
\noindent The distinctions between these notions of learning will play a key role in this work. 

Recently, researchers have begun to revisit many of the key problems in high-dimensional learning from the perspective of robust statistics \cite{diakonikolas2019robust, diakonikolas2017being,lai2016agnostic, charikar2017learning, balakrishnan2017computationally, klivans2018efficient, diakonikolas2019sever, hopkins2018mixture, kothari2018robust, li2018principled, steinhardt2018robust, diakonikolas2019recent, bakshi2020robust, chen2020online}. In particular, we allow an adversary to arbitrarily corrupt an $\epsilon$-fraction of the samples. In this setting, it is no longer possible to learn the original distribution to any desired accuracy. In fact the algorithmic problems associated with working in high-dimensions become even more acute in the sense that many algorithms that work in the non-robust setting turn out to only be able to tolerate a fraction of corruptions that decays inverse polynomially with the dimension. 

For the most part, the emphasis in algorithmic robust statistics has been on getting some dimension-independent robustness guarantee. And only for simpler problems, like estimating a high-dimensional Gaussian \cite{diakonikolas2019robust, diakonikolas2018robustly} and linear regression \cite{bakshi2020robust} are nearly optimally robust algorithms known. In this work, we will take aim at the problem of giving algorithms with nearly optimal robustness guarantees for the challenging task of learning Gaussian mixture models. 
Most relevant to us are the recent works of  Liu and Moitra \cite{liu2020settling} and Bakshi et al. \cite{bakshi2020robustly} who gave the first robust algorithms for learning Gaussian mixture models that achieve dimension-independent robustness guarantees. Let $k$ be the number of components in the mixture.  These works achieve error rates of $\eps^{\Omega_k(1)}$. However in terms of the quantitative dependence on $\epsilon$, these works are far from optimal, and here we will ask for much more: 




\begin{quote}
{\em Given $\epsilon$-corrupted samples, is there an efficient algorithm for estimating the true mixture within $\widetilde{O}(\epsilon)$ in total variation distance for any constant $k$?}
\end{quote}

\noindent Such a bound would be optimal up to polylogarithmic factors. And even in the case of a single Gaussian it is known that there are fundamental tensions between robustness and computational efficiency, and there is evidence that it might not be possible to obtain $O(\epsilon)$ accuracy (at least in a subtractive model of noise) \cite{diakonikolas2017statistical}. As we will discuss below, in order to solve this problem we will need new frameworks and strategies that avoid trying to learn individual components. 

\subsection{Observability}

Our main conceptual contribution is a new framework, which we call {\em observability}, for framing high-dimensional density estimation problems.  We use this notion as a building block for how to design algorithms for robustly learning a mixture of Gaussians even when it is impossible to learn it on a component-by-component basis.  Observability involves having a set of test functions $f_1, \dots , f_n$ that are used to measure a distribution in a family $\mcl{F}$.  It has many parallels, but also important differences, with the commonly used notion of identifiability.  We begin with a definition.
\begin{definition}[Observability]
Given a family of distributions $\mcl{F}$ and a set of test functions $f_1, \dots , f_n$, we say that the family $\mcl{F}$ is observable through the test functions $f_1, \dots , f_n$ if any two distributions $\mcl{M}, \mcl{M}' \in \mcl{F}$ that produce identical measurements (i.e. $ \E_{\mcl{M}}[f_i] = \E_{\mcl{M}'}[f_i] $), they must be equivalent as distributions.
\end{definition}
In other words, $\mcl{F}$ is observable through a family of test functions $f_1, \dots , f_n$ if these test function measurements uniquely determine a distribution in $\mcl{F}$.  Of course we need a more algorithmically useful version of observability with quantitative guarantees.  Since our goal is to achieve nearly optimal robustness guarantees, we will need the test function discrepancy and the TV distance to be proxies for each other up to logarithmic factors.  We call this  \textit{strong observability}.
\begin{definition}[Strong Observability]\label{def:strong-observability}
Given a family of distributions $\mcl{F}$ and a set of test functions $f_1, \dots , f_n$, we say that the family $\mcl{F}$ is strongly observable through the test functions $f_1, \dots , f_n$ if for any two distributions $\mcl{M}, \mcl{M}' \in \mcl{F}$, we have
\[  
d_{TV}(\mcl{M}, \mcl{M}') \cong  \sum_i \left \lvert \E_{\mcl{M}}[f_i] - \E_{\mcl{M}'}[f_i]  \right \rvert 
\]
where the $\cong$ means that the two sides are equivalent up to logarithmic factors.
\end{definition}

Strong observability is a subtle property. There are related facts that are much easier to establish: For any two mixtures $\mcl{M}$ and $\mcl{M}'$ that are $\epsilon$-far in total variation distance, there is a test function $f$ (whose values are bounded between zero and one) where $|\E_{\mcl{M}}[f] - \E_{\mcl{M'}}[f]| \geq \epsilon$.  However the quantifiers are in the wrong order. In particular, the particular test function that distinguishes $\mcl{M}$ and $\mcl{M}'$ could vary arbitrarily and pathologically as we vary the two mixtures. In contrast, what makes our notion of observability algorithmically useful is that the family of test functions is fixed in advance and we will show that a polynomial number of them suffice.  Thus, for a density estimation algorithm, our problem is reduced to measuring the test functions on the true mixture and then computing any distribution in $\mcl{F}$ that matches these measurements. Strong observability implies that this strategy will obtain nearly optimal robustness guarantees, even in the face of adversarial corruptions.  The notion of observability seems natural and fundamental to high-dimensional learning, but as far as we are aware it has not appeared in the literature before. 

Our main result is in establishing that strong observability is possible for GMMs with a polynomial number of test functions, provided that the components are in regular form (see Definition \ref{def:regular-form}).  Roughly, a mixture is in regular form if all components are not too poorly conditioned and not too separated from each other.  Moreover we can reduce the general learning problem to the regular form case by invoking recent results on robust clustering.  Our result can be summarized as:
\begin{quote}
{ \em A mixture of a constant number of Gaussians  in regular form is strongly observable through constant degree  Hermite moments \footnote{Hermite moments will be defined more formally in the next section, but for now can be thought of modifications of standard moments that can be robustly estimated.}. }
\end{quote}
More specifically we prove:
\begin{theorem}[Informal, see Theorem \ref{thm:identifiability}]\label{thm:observability-informal}
For two mixtures $\mcl{M}, \mcl{M}'$ of a constant number of Gaussians in regular form, the distance between their first $O_k(1)$ Hermite moments and their TV distance are equivalent up to logarithmic factors.  
\end{theorem}

While our overall algorithm builds on the line of previous works on robustly learning GMMs \cite{diakonikolas2020robustly, bakshi2020outlier, liu2020settling}, our key contribution is the proof of strong observability. It leverages the recent generating function technology in \cite{liu2020settling} but in new ways that avoid using the sum-of-squares hierarchy. 


\subsubsection{Comparison to Identifiability}  It will be helpful to compare observability with the more familiar concept of identifiability, which is usually thought of as the crucial ingredient in parameter learning algorithms.  Recall the definition of identifiability.
\begin{definition}[Identifiability, informal]
Given a family of distributions $\mcl{F}$ parameterized by parameters $\theta$, we say that the family $\mcl{F}$ is identifiable if any two distributions $\mcl{F}(\theta), \mcl{F}(\theta') \in \mcl{F}$ that are close in TV they must also be close in terms of their parameters (for some appropriate parameter distance).
\end{definition}
\noindent In the case of parameter distance for GMMs, usually we require that there be a matching between the components in $\mcl{M}$ and those in $\mcl{M}'$ so that across the matching the components are close in TV and have similar mixing weights. 

However identifiability is just not the right notion to use for density estimation, at least when it comes to achieving nearly optimal robustness guarantees. The issue is that the relationship between component-wise distance and TV distance is quantitatively too weak. There are explicit constructions of GMMs $\mcl{M}$ and $\mcl{M}'$ that are $\epsilon$-close in TV distance but where all the components in both mixtures are all separated by at least $\epsilon^{O(1/k)}$ (see \cite{moitra2010settling, hardt2015tight}) in TV.  The key point is that the mixtures are so close to each other that we cannot distinguish between them in the setting where an $\epsilon$-fraction of our samples can be arbitrarily corrupted. Hence we have:

\begin{proposition}
When an $\epsilon$-fraction of the samples are arbitrarily corrupted, it is not information-theoretically possible to learn the components of a GMM to accuracy better than $\epsilon^{O(1/k)}$. 
\end{proposition}

This is a serious issue because it means that many of the standard techniques for learning mixture models that work by learning individual components are trying to do too much and will get stuck at the above barrier.  In particular, while we are using the same family of test functions as \cite{liu2020settling}, the techniques in \cite{liu2020settling} rely on trying to isolate the parameters of each of the components and thus will run into the $\epsilon^{O(1/k)}$ barrier.

\subsection{Our Results and Techniques}

Our main result is an algorithm whose robustness guarantees are optimal up to logarithmic factors for any constant $k$.  The formal statement can be found in Theorem \ref{thm:main}.
\begin{theorem}[Informal]
Let $k$ be a constant.  Let $\mcl{M} = w_1G_1 + \dots + w_kG_k$ be a mixture of Gaussians in $\R^d$ whose components have variances lower and upper bounded in all directions and such that the mixing weights are lower bounded (both of these bounds can be any function of $k$).  Given $\poly(d/\eps)$ samples from $\mcl{M}$ that are $\eps$-corrupted, there is an algorithm that runs in time $\poly(n)$ and with high probability outputs a distribution $f$ that is a mixture of $k$ polynomial Gaussians (see below for formal definition) such that
\[
d_{\TV}( \mcl{M}, f) \leq \wt{O}(\eps) \,.
\]
\end{theorem}

\begin{remark}
While weaker than the familiar goal of parameter learning (which as we discussed is \textbf{impossible} for the quantitative robustness guarantees we are aiming for), improper density estimation still has important applications. For example, our estimate $f$ will have the property that we can efficiently sample from it, which means we can compute probabilities of arbitrary events and various statistics without needing new samples from $\mcl{M}$.
\end{remark}

As we discussed in the previous subsection, the main challenge is that in our setting parameter learning is not information-theoretically possible, at least not with the kinds of robustness guarantees that we are aiming for. Our algorithm and its analysis revolve around showing that a Gaussian mixture model is strongly observable through its Hermite moments. We now sketch the proof. 


First, we give some background. For a distribution $D$, the characteristic function is defined as $\wh{D}(X) = \E_{z \sim D}[ e^{iz \cdot X}]$ (where $i = \sqrt{-1}$) and can be expanded as a power series whose terms are the moments of $D$
\[
\wh{D}(X) = \sum_{j = 0}^{\infty} \frac{\E_{z \sim D}[ (z \cdot X)^n]}{n!} i^n  \,.
\]
We can also invert the characteristic function to translate from the moments back to the actual density function.  For our purposes we will want to work with the Hermite moments instead, because they can be robustly estimated using existing techniques \cite{kane2020robust, liu2020settling}. It turns out that there is an analog of the relationship between moments and characteristic functions, but for Hermite moments instead. Specifically we define the adjusted characteristic function (for details see Definition \ref{def:adj-char-func}) and show that the terms in the power series expansion of the adjusted characteristic function are exactly the Hermite moments (see Corollary \ref{coro:adj-char-motivation}).  The key is we can give quantitative estimates for inverting the adjusted characteristic function that allow us to relate the size of terms in the power series (which are Hermite moments) to the $L^1$ norm of the density function.

This gives us some relation between the Hermite moments and the TV distance but it is still far from strong observability.  In particular, the power series has infinitely many terms, but in order to prove strong observability, we must restrict to a constant ($O_k(1)$)  number of test functions.  The key is to prove that, for mixtures of Gaussians, the Hermite moments satisfy a recurrence relation of order $O_k(1)$ and further that this recurrence has bounded coefficients.  For context, the moments of a single Gaussian satisfy a simple recurrence so it is reasonable to expect that the moments of a mixture satisfy a higher-order recurrence.  This framework of working with Hermite moments through their recurrence relations (instead of through their parameters as in previous works) is crucial to circumventing the $\eps^{O(1/k)}$ barrier.  This is only a sketch and the full proof has additional complexities. 

 However strong observability alone does not get us anywhere because, after robustly estimating the Hermite moments, we would still need to solve a large system of polynomial equations to find a good (proper) estimate. In fact this system does not appear to have any useful structure that can be exploited algorithmically, in part because it has many disconnected solutions due to the failure of robust parameter learning. Instead we circumvent this obstacle by showing how to solve a relaxed version of the polynomial system that corresponds to allowing ourselves to output a Gaussian mixture model whose mixing weights are low degree polynomials \--- i.e. it has the form
\[
Q_1(x) G_1 + \dots + Q_k(x) G_k
\]
where $Q_1, \dots , Q_k$ are low-degree polynomials that are nonnegative everywhere. We call this a mixture of polynomial Gaussians (MPG). As it turns out, our strong observability result (Theorem \ref{thm:identifiability}), that being close in Hermite moments implies closeness in total variation distance, extends to MPGs as well. We emphasize that this is only a high-level description of the proof, and there are many subtleties such as the crucial fact that the degrees of the polynomials $Q_1, \dots , Q_k$ does not need to grow with $m$, the number of Hermite moments that we want to match, which is essential in order to make our strategy work. We give a detailed technical overview in Section \ref{sec:tech-overview}.

\subsection{Other Related Work}

There has also been a line of work \cite{ashtiani2018sample, ashtiani2020near} towards achieving optimal sample complexity for robustly learning mixtures of Gaussians.  However, these works are not algorithmic i.e. their algorithms are based on discretization and brute-force search and run in exponential time.  In our work, sample complexity is not our primary concern (as long as it is polynomial) and in order to develop a computationally efficient learning algorithm we will need entirely new techniques.

\section{Technical Overview}\label{sec:tech-overview}
\subsection{Problem Setup}
We use $N(\mu, \Sigma)$ to denote a Gaussian with mean $\mu$ and covariance $\Sigma$.  We use $d_{\TV}(\mcl{D}, \mcl{D'})$ to denote the total variation distance between two distributions $\mcl{D}, \mcl{D'}$.  When there is no ambiguity, we will slightly abuse notation and for a distribution $\mcl{D}$ on $\R^d$, we use $\mcl{D}(x)$ to denote the density function of $\mcl{D}$ at $x$.  

Throughout this paper, we use the following shorthand notation: 
\begin{itemize}
    \item $X$ denotes a $d$-tuple of variables $(X_1, \dots , X_d)$
    \item For a vector $\mu \in \R^d$ and matrix $\Sigma \in \R^{d \times d}$ we set $\mu(X) = \mu^T X$ and $\Sigma(X) =  X^T \Sigma X$
\end{itemize}

We begin by formally defining the problem that we will study.  First we define the contamination model.  This is a standard definition from robust learning (see e.g. \cite{diakonikolas2020robustly}).
\begin{definition}[Strong Contamination Model]\label{def:corruption}
We say that a set of vectors $Y_1, \dots , Y_n$ is an $\eps$-corrupted sample from a distribution $\mcl{D}$ over $\R^d$ if it is generated as follows.  First $X_1, \dots , X_n$ are sampled i.i.d. from $\mcl{D}$.  Then a (malicious, computationally unbounded) adversary observes $X_1, \dots , X_n$ and replaces up to $\eps n$ of them with any vectors it chooses.  The adversary may then reorder the vectors arbitrarily and output them as $Y_1, \dots , Y_n$
\end{definition}

\noindent In this paper, we study the following problem.  There is an unknown mixture of Gaussians 
\[
\mcl{M} = w_1G_1 + \dots + w_kG_k
\]
where $G_i = N(\mu_i, \Sigma_i)$.  We receive an $\eps$-corrupted sample $Y_1, \dots , Y_n $ from $\mcl{M}$ where $n = \poly(d/\eps)$ (we treat $k$ as a constant).  The goal is to output a density function of a distribution, say $f$, such that 
\[
d_{\TV}(f, \mcl{M}) \leq \wt{O}(\eps) \,.
\]
In our main result, Theorem \ref{thm:main}, we give an algorithm that computes such a function $f$ of the form
\[
f(x) = Q_1(x)\ovl{G}_1(x) + \dots +  Q_k(x)\ovl{G}_k(x)
\]
where $Q_1, \dots , Q_k$ are polynomials of constant (possibly depending on $k$) degree that are nonnegative everywhere and $\ovl{G}_1, \dots , \ovl{G}_k$ are Gaussians.  We call such functions mixtures of polynomial Gaussians (MPG) distributions for short (see Definition \ref{def:poly-gaussian-ditribution}).

Throughout our paper, we will assume that all of the Gaussians that we consider have variance at least $\poly(\eps/d)$ and at most $\poly(d/\eps)$ in all directions i.e. they are not too flat.  This implies that their covariance matrices are invertible so we  may write expressions such as $\Sigma_i^{-1}$.  
\begin{remark}
Our main results for nearly optimal density estimation require a stronger assumption that the variances are between $\poly(\log 1/\eps)^{-1}$ and $\poly(\log 1/\eps)$ in each direction.  However, en route to these results, we first prove a few simple generalizations of the the results in \cite{liu2020settling} and these results hold under the same assumptions as in \cite{liu2020settling} i.e. components have variance between  $\poly(\eps/d)$ and $\poly(d/\eps)$ in all directions.
\end{remark}

We will also  assume that the $w_i$ are at least $A^{-1}$ for some constant $A$.  While a lower bound on the mixing weights is not technically necessary for density estimation, we need such an assumption in our paper because we need to first run  a parameter estimation algorithm (see \cite{liu2020settling}) to obtain rough estimates for all of the components.

Throughout this paper, we treat $k,A$ as constants \--- i.e. $A$ could be any function of $k$ \--- and when we say polynomial, the exponent may depend on these parameters.  We are primarily interested in dependence on $\eps$ and $d$ (the dimension of the space).

\subsection{Proof Overview}\label{sec:proof-overview}
We now give an overview of the proof of our main result.

\subsubsection{Regular Form Mixtures}\label{sec:regular-form-overview}
We will first consider the case when the components of the mixture are in a convenient form, which we call regular form, meaning that all of the components are not too far from each other and not too poorly conditioned.  
\begin{definition}[Informal, see Definition \ref{def:regular-form})]
We say a mixture of Gaussians $\mcl{M} = w_1G_1 + \dots + w_kG_k$ is in regular-form if all of the components can be written in the form $G_j = N(\mu_j, I + \Sigma_j)$ where
\begin{align*}
&\norm{\mu_j}, \norm{\Sigma_j}_2 \leq \poly(\log 1/\eps) \\
&\poly(\log 1/\eps)^{-1}I \leq I + \Sigma_j \leq \poly(\log 1/\eps)I
\end{align*}
\end{definition}

In the next subsection (Section \ref{sec:reduce-to-regular}), we sketch how we reduce from a general mixture to a mixture in regular form.  As mentioned previously, the key ingredient in our algorithm is an observability statement, that we have a family of test functions such that if two mixtures are close on this family of test functions, then they must be close in TV distance.  Our learning algorithm will then work by measuring these test functions using the samples and then solving for a distribution that matches these test function measurements.

For many learning problems, such as learning mixtures of Gaussians in the non-robust setting \cite{moitra2010settling,kalai2010efficiently}, using low-degree moments as the set of test functions suffices.  However, for robustly learning mixtures of Gaussians, using the standard moments would lose factors of $\poly(d)$ in the error guarantee.  Similar to previous papers on robustly learning mixtures of Gaussians \cite{kane2020robust, liu2020settling}, we use the family of low-degree Hermite moments as our test functions.  Of course, as mentioned previously, there are still many additional obstacles to proving strong observability and circumventing the $\eps^{1/k}$ barrier to parameter learning.   
First, we make a few definitions.  See Section \ref{sec:hermite-basics} for more details.
\begin{definition}
Let $\mcl{H}_m(x)$ denote the univariate Hermite polynomials, e.g. $\mcl{H}_2(x) = x^2 - 1,\mcl{H}_3(x) = x^3 - 3x  $.  Let $\mcl{H}_m(x, y^2)$ be the homogenized Hermite polynomials e.g. $ \mcl{H}_2(x,y^2) = x^2 - y^2, \mcl{H}_3(x,y^2) = x^3 - 3xy^2$.
\end{definition}
\begin{definition}[Multivariate Hermite Polynomials]
Let $H_m(X,z)$ be a formal polynomial in variables $X = X_1, \dots , X_d$ whose coefficients are polynomials in $d$ variables $z_1, \dots , z_d$ that is given by
\[
H_m(X,z) = \mcl{H}_m( z_1X_1 + \dots + z_dX_d , X_1^2 + \dots + X_d^2) \,.
\]
\end{definition}

\begin{definition}[Hermite Moment Polynomials]
For a distribution $D$ on $\R^d$, we let
\[
h_{m,D}(X) = \E_{(z_1, \dots , z_d) \sim D}[H_m(X,z)] \,.
\]
We omit the subscript $D$ when it is clear from context.  We refer to $h_{m,D}(X)$ as the Hermite moment polynomials of $D$. 
\end{definition}
\begin{remark}
We will use the term Hermite moment polynomial (instead of just Hermite moment) to emphasize the fact that we view $h_{m,D}(X)$ as a polynomial in the formal variables $X$.  This polynomial representation will be useful for the machinery that we introduce later on for manipulating these Hermite moment polynomials.
 \end{remark}
\begin{remark}
If instead of $\E_{(z_1, \dots , z_d) \sim D}[H_m(X,z)]$ we had $\E_{(z_1, \dots , z_d) \sim D}[(z_1X_1 + \dots + z_dX_d)^m]$ then we would get the standard moments.
\end{remark}


Recall that a key step to getting optimal accuracy is proving strong observability, which involves relating distance between Hermite moment polynomials to TV distance directly.  The Hermite moment polynomials turn out to be a particularly nice object to work with because we can work with their $L^2$ norm (after reorganizing the coefficients into a vector) without losing dimension dependent factors \footnote{For, say, standard moments, we would instead need to work with the tensor injective norm.}.

The key observability result in our proof (Theorem \ref{thm:identifiability}) implies that for regular-form mixtures, the distance between Hermite moment polynomials (in terms of $L^2$ norm of coefficients) is equivalent to TV distance up to a $\poly(\log 1/\eps)$ factor.

\paragraph{Algorithm Summary:} We now summarize our algorithm for learning regular-form mixtures. The main parts are 
\begin{enumerate}
    \item \textbf{Strong Observability through Hermite Moment Polynomials:} we prove that for two mixtures of $k$ Gaussians, and more generally mixtures of $k$ polynomial Gaussians, if their first $O_k(1)$ Hermite moment polynomials are $\eps$-close in coefficient $L^2$ distance, then the two mixtures are  $\wt{O}(\eps)$-close in TV distance (Theorem \ref{thm:identifiability}). 
    
    \item \textbf{Estimate the Hermite Moment Polynomials Optimally:} we estimate the Hermite moment polynomials of the mixture to optimal accuracy (Theorem \ref{thm:estimate-hermite})
    \item \textbf{Compute Rough Component Estimates:} we compute $\eps^{\Omega_k(1)}$-accurate estimates for all of the components (Theorem \ref{thm:list-learning-poly-acc})
    \item \textbf{Estimate Density Function Optimally:} we bootstrap the rough component estimates using the Hermite moment polynomial estimates to compute the density function of the mixture to optimal accuracy (Theorem \ref{thm:estimate-regularform})
\end{enumerate}

The next figure shows how the parts fit together in our algorithm.  We will then explain each of the parts in more detail.  We focus on parts $1,2,4$ because part $3$ follows easily from the results in \cite{liu2020settling}.    

\begin{figure}[H]
\centering
\begin{tikzpicture}[node distance=2cm]
\tikzstyle{block} = [rectangle, draw=black, thick,  text width=15 em,align=center, rounded corners, minimum height=2em]
\node (input)[block] {\textbf{Input:} $\eps$-corrupted sample $X_1, \dots , X_n$ from mixture $\mcl{M} = w_1G_1 + \dots + w_kG_k$};
\node (est-poly-acc) [block, yshift = -3cm, xshift = 4cm]{Estimate components $\ovl{G}_1, \dots , \ovl{G}_k$ to $\eps^{\Omega(1)}$ accuracy};
\node (est-hermite) [block, yshift = -3cm, xshift = -4cm]{Estimate Hermite moment polynomials $h_{1, \mcl{M}}, \dots , h_{m,\mcl{M}}$ to $\wt{O}(\eps)$ accuracy};
\node (middle) [yshift = -5cm] {};
\node (est-poly-comb) [block, yshift = -8cm]{Compute mixture of polynomial Gaussians $f = Q_1\ovl{G}_1 + \dots + Q_k\ovl{G}_k$ that matches Hermite moment polynomials  $h_{1, \mcl{M}}, \dots , h_{m,\mcl{M}}$ to within $\wt{O}(\eps)$};
\node (est-density-function) [block, yshift = -11cm]{$f$ must be close to the density function of $\mcl{M}$ in $L^1$ distance};
\draw [thick , ->] (input) -- (est-hermite) node[midway, above left] { Theorem \ref{thm:estimate-hermite}};
\draw [thick , ->] (input) -- (est-poly-acc) node[midway, above right] {Theorem \ref{thm:list-learning-poly-acc}};
\draw [thick] (est-hermite) -- (middle);
\draw [thick] (est-poly-acc) -- (middle);
\draw [thick, ->] (middle) -- (est-poly-comb) node[midway, left] { Theorem \ref{thm:estimate-regularform}};
\draw [thick, ->]  (est-poly-comb) -- (est-density-function) node[midway, left] {Theorem \ref{thm:identifiability}};
\end{tikzpicture}
\caption{Overview of our algorithm for regular-form mixtures}
\end{figure}
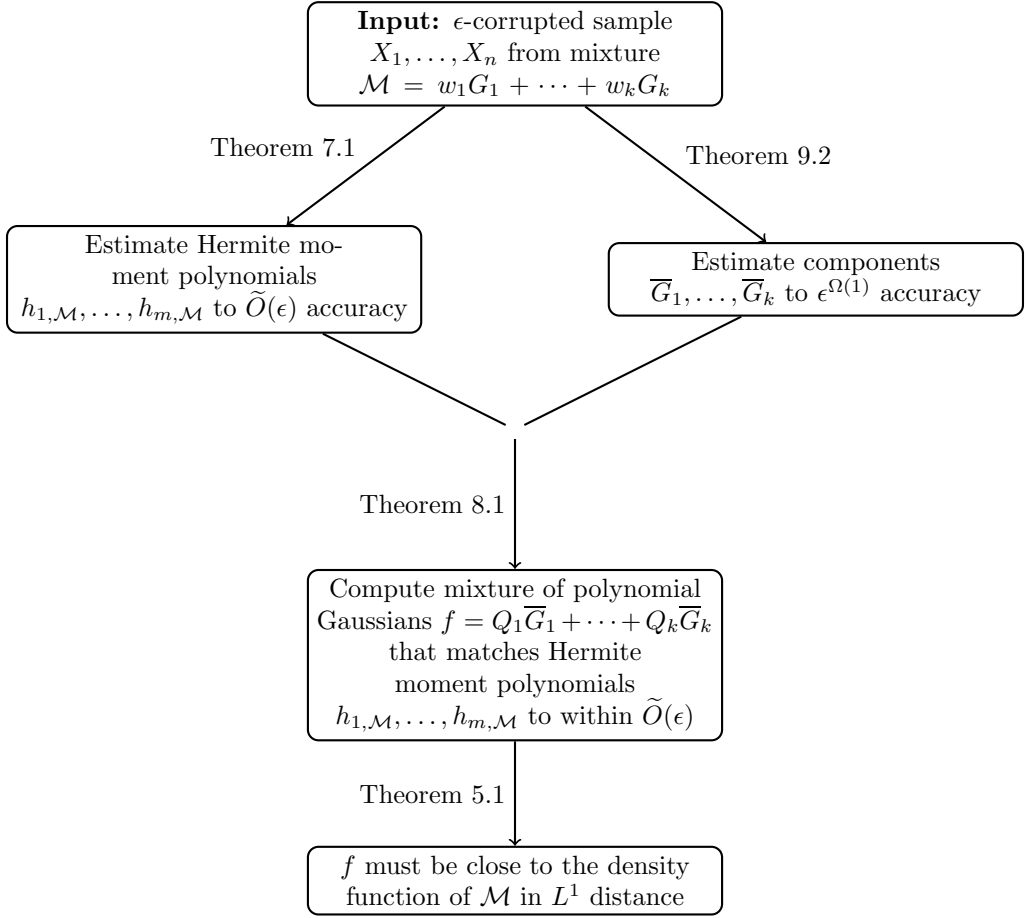

\paragraph{Strong Observability through Hermite Moment Polynomials:}
To help build intuition, here we will sketch a proof of observability via Hermite moment polynomials in the infinite limit, i.e. if two mixtures match exactly on their first $O_k(1)$ Hermite moment polynomials, then they must be exactly the same.  For simplicity, in this discussion, we will restrict ourselves to mixtures of Gaussians.  In our analysis later on, we will need observability for mixtures of polynomial Gaussians because the density function that we output is in this more general class.  

Here, we sketch a proof of the following (informal) theorem.  The full version is in Theorem \ref{thm:identifiability}.   
\begin{theorem}[Informal]\label{thm:informal-observability}
If two regular-form mixtures of Gaussians $\mcl{M} = w_1G_1 + \dots + w_kG_k$ and $\mcl{M}' = w_1'G_1' + \dots + w_k'G_k'$ match on their first $O_k(1)$ Hermite moment polynomials then $\mcl{M} = \mcl{M}'$.
\end{theorem}

A key ingredient will be understanding recurrence relations between Hermite moment polynomials.  To obtain these recurrence relations, we write down a generating function  for the Hermite moment polynomials and then manipulate this generating function using differential operators.  By writing down a differential operator that annihilates the generating function, we then obtain the coefficients of a recurrence relation that the Hermite moment polynomials must satisfy.  First, we have the following identity.
\begin{claim}[See Corollary \ref{coro:hermite-of-mixture}]\label{claim:hermite-of-mixture-intro}
Let $\mcl{M} = w_1 G_1 + \dots w_kG_k$ be a mixture of Gaussians where $G_j = N(\mu_j, I + \Sigma_j)$.  Then
\begin{equation}\label{eq:hermite-identity-intro}
\sum_{m=0}^{\infty} \frac{1}{m!} \cdot h_{m}(X) y^m = w_1e^{\mu_1(X) + \frac{1}{2}\Sigma_1(X)y^2} + \dots + w_ke^{\mu_k(X) + \frac{1}{2}\Sigma_k(X)y^2}
\end{equation}
where $h_m(X)$ are the Hermite moment polynomials of the mixture $\mcl{M}$.
\end{claim}

Let $f(y)$ be the function on the right hand side of (\ref{eq:hermite-identity-intro}), viewed as a function of $y$ with formal variables $X$.  For each $j \in [k]$, consider the differential operator $\mcl{D}_j = \partial - (\mu_j(X) + \Sigma_j(X)y)$ where the partial derivative is taken with respect to $y$.  Then if we let $\mcl{D} = \mcl{D}_k^{2^{k-1}} \mcl{D}_{k-1}^{2^{k-2}}\cdots \mcl{D}_1$, we can verify that
\[
\mcl{D}(f(y)) = 0 \,.
\]
On the other hand, by using the product rule, we can expand the differential operator $\mcl{D}$ in the form
\[
\mcl{D} = Q_{2^k - 1}(X,y)\partial^{2^{k} - 1} +  Q_{2^k - 2}(X,y)\partial^{2^{k} - 2} + \dots + Q_0(X,y)
\]
where $Q_0 , \dots , Q_{2^k - 1}$ are polynomials in $y$ whose coefficients are polynomials in the formal variables $X$.  It is immediate to verify that $Q_j$ has degree at most $2^k - 1 - j$ in $y$ so for all $0 \leq j \leq 2^k - 1$ we can write
\[
Q_j(X,y) = R_{j,2^k - 1 - j}(X)y^{2^k - 1 - j} + \dots + R_{j,0}(X)
\]
for some polynomials $R_{j,0}(X), \dots , R_{j, 2^k - 1 - j}(X)$.

Now consider what happens when we apply $\mcl{D}$ to the left hand side of (\ref{eq:hermite-identity-intro}).  We will get a power series in $y$ whose coefficients are polynomials in $X$.  The coefficient of $y^a$ will be
\[
\sum_{j=0}^{2^k - 1}\sum_{l = 0}^{2^k - 1 - j} \frac{h_{a + j - l}(X)R_{j,l}(X)}{(a - l)!} 
\]
and since $\mcl{D}(f(y)) = 0$ we get the following conclusion.
\begin{claim}
Let $\mcl{M} = w_1 G_1 + \dots w_kG_k$ be a mixture of Gaussians where $G_j = N(\mu_j, I + \Sigma_j)$.  Then there are polynomials $R_{j,l}(X)$ such that for all $a$, the Hermite moment polynomials of $\mcl{M}$ satisfy 
\begin{equation}\label{eq:intro-hermite-recurrence}
\sum_{j=0}^{2^k - 1}\sum_{l = 0}^{2^k - 1 - j} \frac{h_{a + j - l}(X)R_{j,l}(X)}{(a - l)!}  = 0 \,.
\end{equation}
\end{claim}
This means that the Hermite moment polynomials satisfy a recurrence of order $O_k(1)$.  It is straight-forward to extend the above argument to the difference of two mixtures say $\mcl{M} = w_1G_1 + \dots + w_kG_k$ and $\mcl{M}' = w_1'G_1' + \dots  +w_k'G_k'$  and we deduce that the polynomials 
\[
h_{m, \mcl{M}}(X) -  h_{m, \mcl{M}'}(X)
\]
must satisfy a similar recurrence of order $O_k(1)$.  Thus, if the first $O_k(1)$ Hermite moment polynomials of two mixtures are the same, then all of their Hermite moment polynomials must be the same.  Note that this statement suffices for the proof of observability in the infinite limit, but in the full analysis, we need several additional steps to prove quantitative bounds relating the distance between higher-degree Hermite moment polynomials to the  distance between the first $O_k(1)$ Hermite moment polynomials.

The argument above implies that 
\begin{equation}\label{eq:matching-gen-funcs-intro}
\sum_{m=0}^{\infty} \frac{1}{m!} \cdot h_{m, \mcl{M}}(X) y^m = \sum_{m=0}^{\infty} \frac{1}{m!} \cdot h_{m, \mcl{M}'}(X) y^m \,.
\end{equation}
It remains to show how to transform from these generating functions back to the original distributions.  This can be done through the adjusted characteristic function.  We define 
\begin{definition}[Adjusted Characteristic Function]
For a distribution $D$ on $\R^d$, we define its adjusted characteristic function $\wt{D}: \R^d \rightarrow \C$ as 
\[
\wt{D}(X) = \E_{z \sim D}\left[ e^{i z \cdot X + \frac{1}{2} \norm{X}^2}\right]
\]
where $i = \sqrt{-1}$.
\end{definition}
It suffices to note that 
\begin{claim}[Restatement of Claim \ref{claim:invchar-formula}]\label{claim:inv-char-intro}
Let $D$ be a distribution on $\R^d$.  Then
\[
\wt{D}(X) = \sum_{m = 0}^{\infty} \frac{i^m}{m!}h_{m,D}(X) \,.
\]
\end{claim}
Thus, plugging $y = i$ into (\ref{eq:matching-gen-funcs-intro}), we get that 
\[
\wt{\mcl{M}}(X) = \wt{\mcl{M}'}(X) \,.
\]
However, note that the adjusted characteristic function is an invertible transformation (since we can multiply by $e^{-\frac{1}{2} \norm{X}^2}$ and then invert the characteristic function) so actually $\mcl{M}(X) = \mcl{M}'(X)$, completing the proof of observability in the infinite limit.

There are several additional technical ingredients that are necessary to go from observability in the infinite limit to quantitatively strong observability.  The main one is that we need to bound the coefficients of the polynomials $R_{j,l}(X)$.  It is not difficult to obtain sufficiently tight bounds when $\norm{\mu_j}, \norm{\Sigma_j}_2$ are all sufficiently small (smaller than some constant depending on $k$).  To reduce to this case, we need to do some additional work (see Section \ref{sec:reduce-to-very-close}).  Also, we will need quantitative bounds on inverting the adjusted characteristic function.  Such bounds are obtained in Section \ref{sec:char-func}.

\paragraph{Estimate the Hermite Moment Polynomials Optimally:} Here, our goal is to obtain estimates $h_m'(X)$ for the first $O_k(1)$ Hermite moment polynomials.  We sketch a proof of the following (informal) theorem.  The full version is in Theorem \ref{thm:estimate-hermite}.
\begin{theorem}[Informal]
Given an $\eps$-corrupted sample from a regular-form mixture of Gaussians $\mcl{M} = w_1G_1 + \dots + w_kG_k$, we can compute estimates $h_m'(X)$ such that 
\[
\norm{v(h_m(X) - h_m'(X))} \leq \wt{O}(\eps)
\]
where $h_m(X)$ are the true Hermite moment polynomials of the mixture and $v( \cdot )$ denotes converting a polynomial to a vector of coefficients (we then measure the $L^2$ norm of this vector).
\end{theorem}

Note that $h_m(X)$ is the mean of the distribution of $H_m(X,z)$ for $z \sim \mcl{M}$ so estimating $h_m(X)$ is a robust mean estimation task.  Previous papers \cite{liu2020settling, kane2020robust} estimate $h_m(X)$ up to accuracy $\wt{O}(\sqrt{\eps})$ by upper bounding the spectral norm of the covariance matrix and using standard results from robust mean estimation.  However achieving $\wt{O}(\eps)$ is significantly more difficult.

It can be shown, via hypercontractivity, that the distribution of $H_m(X,z)$ exhibits exponential tail decay (see Lemma \ref{lem:hermite-tail}).  However, this alone is not enough to robustly estimate the mean to within $\wt{O}(\eps)$ in a computationally efficient manner.  Existing results achieving optimal accuracy e.g. \cite{diakonikolas2020outlier} require known covariance or some additional moment structure (such as in the case of a single Gaussian).  Furthermore, there is evidence suggesting that achieving optimal accuracy for general sub-Gaussian distributions may be computationally hard \cite{hopkins2019hard}.  

To circumvent these barriers, we leverage the structure of the moments of the distribution of $H_m(X,z)$.  Roughly speaking, we write the covariances of the distributions of $H_0(X,z), \dots , H_m(X,z)$ in terms of the Hermite moment polynomials $h_0(X), \dots , h_m(X)$ (which are the means of the respective distributions).  Thus, we can estimate the means, compute estimates for the covariances and then use our covariance estimates to refine our estimates of the means and keep iterating.  This is similar to how algorithms for robustly learning a single Gaussian use the relation between its covariance and its fourth moment tensor.  Of course, the moment structure of the distribution of $H_m(X,z)$ is significantly more complex so the analysis will be more involved.  

Note that if the covariance of $H_m(X,z)$ were known, then we would be able to estimate the mean of the distribution to $\wt{O}(\eps)$ accuracy using standard techniques (e.g. \cite{diakonikolas2020outlier}).  The first important observation is that the covariance of the distribution of $H_m(X,z)$ can be written in terms of the first $2m$ Hermite moment polynomials $h_0(X), \dots , h_{2m}(X)$.  Next, if $m$ is sufficiently large in terms of $k$, then the polynomials $h_{m+1}(X), \dots , h_{2m}(X)$ can be computed in terms of $h_0(X), \dots , h_m(X)$ via the recurrence in (\ref{eq:intro-hermite-recurrence}).  Since we do not know the actual recurrence, we can solve for the coefficients in the recurrence using $h_0(X), \dots , h_m(X)$ and then use these coefficients to extend the recurrence and compute $h_{m+1}(X), \dots , h_{2m}(X)$.  With this insight, we have the following iterative algorithm.
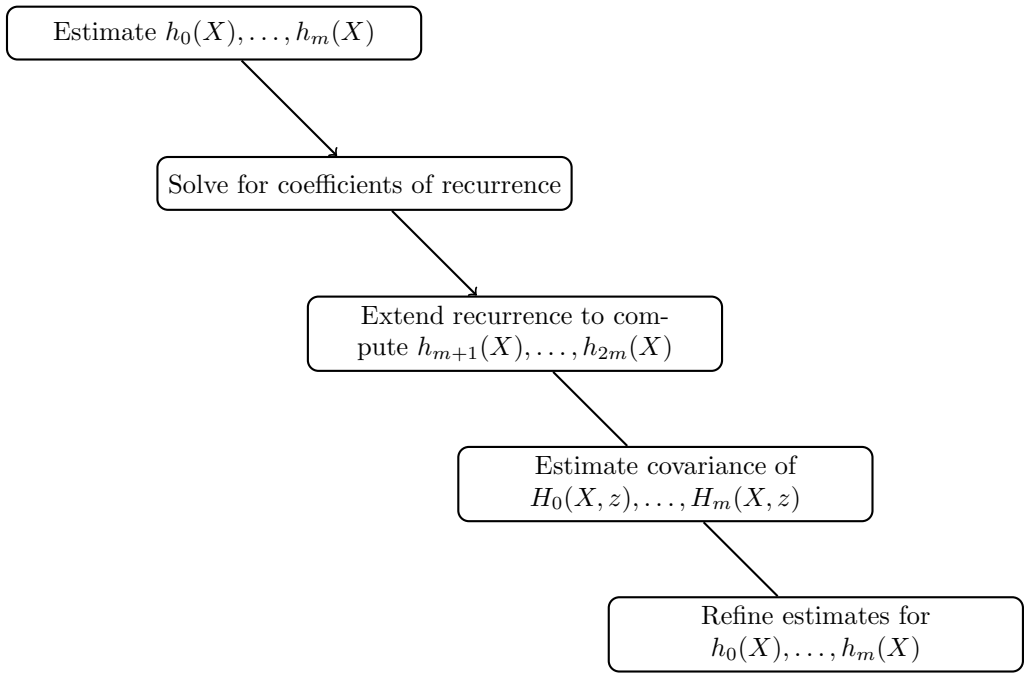
\begin{figure}[H]
\centering
\begin{tikzpicture}[node distance=2cm]
\tikzstyle{block} = [rectangle, draw=black, thick,  text width=15 em,align=center, rounded corners, minimum height=2em]
\node (start)[block] {Estimate $h_0(X), \dots , h_m(X)$};
\node (solve-coeff) [block, yshift = -2cm, xshift = 2cm]{Solve for coefficients of recurrence};
\node (extend-recurrence) [block, yshift = -4cm, xshift = 4cm]{Extend recurrence to compute $h_{m+1}(X), \dots , h_{2m}(X)$};
\node (est-cov) [block, yshift = -6cm, xshift = 6cm]{Estimate covariance of $H_0(X,z), \dots , H_m(X,z)$ };
\node (refine-est) [block, yshift = -8cm, xshift = 8cm]{Refine estimates for $h_0(X), \dots , h_m(X)$};
\draw [thick , ->] (start) -- (solve-coeff) node[midway, above left] {};
\draw [thick , ->] (solve-coeff) -- (extend-recurrence) node[midway, above right] {};
\draw [thick] (extend-recurrence) -- (est-cov);
\draw [thick] (est-cov) -- (refine-est);
\end{tikzpicture}

\end{figure}
Using an upper bound on the covariance (same as in \cite{liu2020settling, kane2020robust}), we can ensure that our initial estimates are $\wt{O}(\sqrt{\eps})$-accurate. Then, by repeatedly running the above, we can refine these estimates to $\wt{O}(\eps)$-accuracy.

\paragraph{Compute Rough Component Estimates:} Here, our goal is to prove the following (informal) theorem.  See Theorem \ref{thm:list-learning-poly-acc} for the full version.
\begin{theorem}[Informal]
Given an $\eps$-corrupted sample from a mixture of Gaussians $\mcl{M} - w_1G_1 + \dots + w_kG_k$, we can compute estimates $\ovl{G}_1, \dots , \ovl{G}_k$ for the components such that 
\[
d_{\TV}(G_j, \ovl{G}_j) \leq \eps^{\Omega_k(1)} \,.
\]
\end{theorem}
This theorem follows from a simple modification to the techniques in \cite{liu2020settling}.  Note that the only difference is that the main theorem in \cite{liu2020settling} assumes that the components of the mixture are not too close in TV distance.  However, this assumption can be removed by essentially merging components and treating them as one if they are $\eps^{\Omega(1)}$-close.

\paragraph{Estimate Density Function Optimally:} So far, we showed that the first $O_k(1)$ Hermite moment polynomials suffice to determine a mixture of Gaussians (and more generally a mixture of polynomial Gaussians).  We then showed how to estimate these Hermite moment polynomials to optimal accuracy.  The last step is to compute a mixture of polynomial Gaussians that matches these Hermite moment polynomials.  To do this, we will take the rough estimates of the components from the previous step and then multiply them by appropriate polynomials.  We sketch a proof of the following (informal) theorem.  The full version is in Theorem \ref{thm:estimate-regularform}. 
\begin{theorem}[Informal]
Let $\mcl{M} - w_1G_1 + \dots + w_kG_k$ be a mixture of Gaussians in regular form. Assume we are given estimates $\ovl{G}_1, \dots , \ovl{G}_k$ for the components such that for all $j$, $d_{\TV}(G_j , \ovl{G}_j) \leq \eps^{\Omega_k(1)}$.  Then we can compute a distribution
\[
f(x) = Q_1(x)\ovl{G}_1(x) + \dots + Q_k(x)\ovl{G}_k(x)
\]
where $Q_1, \dots , Q_k$ are polynomials of degree $C = O_k(1)$ such that the first $m = O_{k,C}(1)$ Hermite moment polynomials of $f$ match those of $\mcl{M}$ up to $\wt{O}(\eps)$ accuracy.
\end{theorem}

It is crucial to note that in order to match $m$ Hermite moment polynomials, the degree of the polynomials $Q_1, \dots , Q_k$ that we need \textit{does not  grow with $m$}.  In other words, we first fix the degree $C$ of the polynomials $Q_1, \dots , Q_k$. We then argue that for any $m$, we can match the first $m$ Hermite moment polynomials using polynomials $Q_1, \dots , Q_k$ of degree $C$.  The only place that $m$ shows up is in the accuracy i.e. the $\tilde{O}(\eps)$ hides a factor of the form $\eps (\log 1/\eps)^m$.  

The reason that we need to fix $C$ first and then choose $m$ sufficiently large in terms of $C$ is that the observability result, Theorem \ref{thm:identifiability}, only works when two distributions match on their first $m$ Hermite moment polynomials for $m$ much larger than $C,k$.

We now sketch how we actually compute the polynomials $Q_1, \dots , Q_k$.  For simplicity, we will first consider a single Gaussian $G = N(\mu, I + \Sigma)$ as this already will reveal the key intuitions.  Assume that we are given an estimate of $G$, say  $\ovl{G} = N(\wt{\mu}, I + \wt{\Sigma})$ with $d_{\TV}(G, \ovl{G}) \leq \eps^c$ for some constant $c$.  Recall Claim \ref{claim:hermite-of-mixture-intro}.  We have
\[
\sum_{m=0}^{\infty} \frac{1}{m!} \cdot h_{m,G}(X) y^m = e^{\mu(X)y + \frac{1}{2}\Sigma(X)y^2} \,.
\]
Now consider the generating function 
\[
e^{\mu(X)y + \frac{1}{2}\Sigma(X)y^2} = e^{(\mu(X) - \wt{\mu}(X))y + \frac{1}{2}(\Sigma(X) - \wt{\Sigma}(X))y^2} e^{\wt{\mu}(X)y + \frac{1}{2}\wt{\Sigma}(X)y^2} \,.
\]
Since  $d_{\TV}(G, \ovl{G}) \leq \eps^{c}$ , we can show that $\norm{\mu - \wt{\mu}},\norm{\Sigma - \wt{\Sigma}}_2 \leq \eps^{\Omega(c)}$.  Now consider the power series expansion of
\[
e^{(\mu(X) - \wt{\mu}(X))y + \frac{1}{2}(\Sigma(X) - \wt{\Sigma}(X))y^2}  = \sum_{m = 0}^{\infty} \frac{\left((\mu(X) - \wt{\mu}(X))y + \frac{1}{2}(\Sigma(X) - \wt{\Sigma}(X))y^2\right)^m}{m!} \,,
\]
We can expand 
\[
\left((\mu(X) - \wt{\mu}(X))y + \frac{1}{2}(\Sigma(X) - \wt{\Sigma}(X))y^2\right)^m
\]
using the binomial theorem.  The key observation is that since $\norm{\mu - \wt{\mu}},\norm{\Sigma - \wt{\Sigma}}_2 \leq \eps^{\Omega(c)}$, whenever we multiply more than $O(1/c)$ terms of the form $(\mu(X) - \wt{\mu}(X))$ or $(\Sigma(X) - \wt{\Sigma}(X))$ together, the result will have coefficient norm smaller than $\eps$.  Thus, we can essentially drop all but the first $O(1/c)$ terms in the power series expansion i.e.
\[
e^{(\mu(X) - \wt{\mu}(X))y + \frac{1}{2}(\Sigma(X) - \wt{\Sigma}(X))y^2} \sim P(X,y)
\]
where $P$ has degree at most $O(1/c)$ in $y$ and $X$.  Thus,
\[
e^{\mu(X)y + \frac{1}{2}\Sigma(X)y^2} \sim P(X,y)e^{\wt{\mu}(X)y + \frac{1}{2}\wt{\Sigma}(X)y^2} \,.
\]
It remains to plug in $y = i$, multiply by $e^{-\frac{1}{2}\norm{X}^2}$ and invert the characteristic function (recall Claim \ref{claim:inv-char-intro}).  We can then verify that the resulting function will be of the form
$Q(x)\ovl{G}(x)$ where $Q$ has degree at most $O(1/c)$.

The above intuition roughly says that $G$ can be approximated by $\ovl{G}$ times a polynomial of degree $O(1/c)$ as long as $d_{\TV}(G, \ovl{G}) \leq \eps^c$.  Thus, the mixture of Gaussians  $\mcl{M} - w_1G_1 + \dots + w_kG_k$ can be approximated by a function of the form
\[
f(x) =  Q_1(x)\ovl{G}_1(x) + \dots + Q_k(x)\ovl{G}_k(x)
\]
for some constant-degree polynomials $Q_1, \dots , Q_k$.  It remains to show how to compute the polynomials $Q_1, \dots , Q_k$.  To solve for $Q_1, \dots , Q_k$, it suffices to note that the Hermite moment polynomials of $f(x)$ are linear forms in the coefficients of $Q_1, \dots , Q_k$.  Thus, since we have estimates for the Hermite moment polynomials of the true mixture $\mcl{M}$, it suffices to solve a linear system for the coefficients of $Q_1, \dots , Q_k$.

In our full proof, we need to deal with one additional detail, which is that the $Q_1, \dots , Q_k$ computed above may take on negative values.  We need a few additional steps to modify the polynomials $Q_1, \dots , Q_k$ so that we always output a valid distribution.

\subsubsection{Reducing to Regular Form}\label{sec:reduce-to-regular}

For general mixtures, our algorithm has one additional step where we need to cluster the mixture into submixtures and then place each submixture in regular-form.  

To do this, we can use Theorem \ref{thm:list-learning-poly-acc} to obtain rough estimates for all of the components.  We then cluster the samples into subsamples by assigning each sample to the estimated component that assigns it the highest likelihood.  While this clustering will not classify all of the samples ``correctly" (e.g. if components overlap), we combine the subsamples into submixtures and argue that for some recombination the following two conditions hold:
\begin{itemize}
    \item The clustering into submixtures is accurate to within $\wt{O}(\eps)$ accuracy (see Lemma \ref{lem:find-good-clusters})
    \item For each submixture, we can apply a linear transformation to place it in regular form 
\end{itemize}
Thus, once we find this recombination (say by enumerating over all possible recombinations), we have reduced the problem to learning mixtures of Gaussians in regular form.

\section{ Preliminaries}\label{sec:prelim}

\subsection{Basic Definitions}
We now introduce some terminology that we will use throughout the paper.

\begin{definition}
We say a function $f(x): \R^d \rightarrow \R$ is a degree-$m$ polynomial Gaussian if it can be written in the form
\[
f(x) = Q(x)G(x)
\]
where $G(x)$ is the probability density function of a Gaussian and $Q$ is a polynomial in $d$ variables of degree at most $m$.
\end{definition}

\begin{definition}\label{def:poly-gaussian-ditribution}
We say a function $f(x): \R^d \rightarrow \R$ is a degree-$m$ mixture of polynomial Gaussians (MPG) if 
\[
f(x) = Q_1(x)G_1(x) + \dots + Q_k(x)G_k(x)
\]
where $G_1(x), \dots, G_k(x)$ are the probability density functions of Gaussians and $Q_1, \dots , Q_k$ are polynomials in $d$ variables of degree at most $m$.  If the polynomials $Q_1, \dots , Q_k$ are all nonnegative for any $x \in \R^d$ and $\int_{\R^d}f(x) dx= 1$, then we say $f$ is a degree-$m$  MPG distribution.
\end{definition}

 We will often need to work with mixtures of Gaussians whose components are in a specific form, which we call regular-form.  

\begin{definition}\label{def:regular-form}
We say a set of Gaussians $G_1 = N(\mu_1, I + \Sigma_1), \dots , G_k = N(\mu_k, I + \Sigma_k)$  is in $(\alpha, \beta)$-regular form if the following holds:
\begin{itemize}
    \item For all $j$, $\norm{\mu_j} \leq \alpha$ 
    \item For all $j$, $\norm{\Sigma_j}_2 \leq \alpha$
    \item For all $j$, $   \beta^{-1} I \leq I + \Sigma_j \leq \beta I$ 
\end{itemize}
We will sometimes need an additional conditions that there exists some $j$ such that 
\[
\norm{\mu_j} + \norm{\Sigma_j}_2 \leq \gamma \,.
\]
If this additional condition holds, we say that the set of Gaussians is in $(\alpha, \beta, \gamma)$-regular form.
\end{definition}
\begin{definition}
We say that a mixture of Gaussians $\mcl{M} = w_1 G_1 + \dots + w_kG_k$ is in $(\alpha, \beta)$ (respectively $(\alpha, \beta , \gamma)$) regular-form if the set of components $\{G_1, \dots , G_k \}$ is in $(\alpha, \beta)$ (respectively $(\alpha, \beta, \gamma)$) regular-form.
\end{definition}
\begin{remark}
Generally, we will be interested in the regime where $\alpha, \beta \leq \poly(\log 1/\eps)$ and $\gamma$ is a sufficiently small constant (in terms of $k$).
\end{remark}

\begin{definition}
Given a family of polynomials $\mcl{S} = \{P_1,P_2, \dots \}$ in variables $X = (X_1, \dots, X_d)$, we say a polynomial $Q(X)$ is $(A,B)$-simple with respect to $\mcl{S}$ for some parameters $A,B$ if $Q$ can be written as a linear combination of $A$ terms where
\begin{itemize}
    \item All coefficients in the linear combination have magnitude at most A
    \item Each term is a product of at most $B$ polynomials from $\mcl{S}$
\end{itemize}
\end{definition}

We will need the following standard fact (see e.g. \cite{arutyunyan2018deviation, kauers2014hypercontractive}) that allows us to bound the tail decay of the distribution of a polynomial $f(x)$ where $x$ is drawn from a Gaussian. 
\begin{claim}[Hypercontractivity]\label{claim:hypercontractivity}
Let $f$ be a polynomial of degree $m$.  Let $G = N(\mu, \Sigma)$ be a Gaussian in $\R^d$.  There is a universal constant $c$ such that for any even integer $q$,
\[
\left(\E_{x \sim G} |f(x)|^q \right) \leq  (cq)^{mq} \left( \E_{x \sim G} |f(x)|^2\right)^{q/2} \,.
\]
\end{claim}

\subsection{Tensors and Polynomials}

We will now introduce notation and tools to deal with tensors and formal polynomials.  We will need to translate between polynomials and their corresponding representations as tensors repeatedly in this paper.

\begin{definition}
Let $X$ denote the set of formal variables $(X_1, \dots , X_d)$.  Then for a positive integer $k$, $X^{\otimes k}$ denotes the $\underbrace{d \times \dots \times d}_{k}$ tensor 
\[
\underbrace{(X_1, \dots , X_d) \otimes \dots \otimes (X_1, \dots , X_d)}_{k}
\]
\end{definition}

Now we will define a canonical transformation between polynomials and tensors.

\begin{definition}
For a homogeneous polynomial $f(X)$ of degree $k$ in the $d$ variables $X_1, \dots , X_d$ with real coefficients, define $T(f)$ to be the unique symmetric tensor with dimensions $\underbrace{d \times \dots \times d}_k$ such that 
\[
\langle T(f), X^{\otimes k} \rangle = f(X) \,.
\]
We call $T(f)$ the coefficient tensor of $f$.
\end{definition}

\begin{definition}\label{def:poly-to-vec}
For a homogeneous polynomial $f(X)$ in the $d$ variables $X_1, \dots , X_d$ with real coefficients define $v(f)$ to be the vector obtained by flattening $T(f)$.  We call $v(f)$ the coefficient vector of $f$.  
\end{definition}

\begin{definition}
For a polynomial (not necessarily homogeneous) $f(X,y)$, viewed as a polynomial in $y$ whose coefficients are homogeneous polynomials in $X$ (of not necessarily the same degree) i.e.
\[
f(X,y) = f_0(X) +f_1(X)y + \dots + f_k(X)y^k
\]
we define $v_y(f)$ to be the vector obtained by concatenating $v(f_m(X))$ for all $m$. 
\end{definition}

We will frequently consider expressions of the form $\norm{v(f)}$ i.e. the $L^2$ norm of the coefficient vector.
\begin{definition}
For a polynomial $f(X)$, we call $\norm{v(f)}$ the coefficient norm of $f$.
\end{definition}

The first claim below gives us an upper bound on the coefficient norm of the product of two polynomials $f$ and $g$ in terms of the coefficient norms of $f$ and $g$.

\begin{claim}\label{claim:polytensor-prod-bound}
Let $f$ and $g$ be two homogeneous polynomials in the variables $X = (X_1, \dots , X_d)$ of degree $m_1,m_2$ respectively.  Then
\[
\norm{T(fg)}_2 \leq \norm{T(f)}_2 \norm{T(g)}_2 \,.
\]
Equivalently,
\[
\norm{v(fg)} \leq \norm{v(f)} \norm{v(g)}
\]
\end{claim}
\begin{proof}
Note that $T(fg)$ can be written as an average, over all partitions of $[m_1 + m_2]$ into two sets $S_1, S_2$ of size $m_1, m_2$, of $T(f)_{S_1} \otimes T(g)_{S_2} $ where $T(f)_{S_1} \otimes T(g)_{S_2} $ is a $d^{\otimes (m_1 + m_2)}$ tensor obtained by taking $T(f)$ in the dimensions indexed by $S_1$ and $T(g)$ in the dimensions indexed by $S_2$ and tensoring them together.  It is clear that
\[
\norm{T(f)\otimes T(g)}_2 = \norm{T(f)}_2\norm{T(g)}_2
\]
so using the triangle inequality, we get the desired conclusion.
\end{proof}

We also have a lower bound on $\norm{v(fg)}$ that follows immediately from the results  in \cite{liu2020settling}.
\begin{claim}[Claim 3.18 in \cite{liu2020settling}]\label{claim:prod-lower-bound}
Let $f,g$ be two homogeneous polynomials in the variables $X = (X_1, \dots , X_d)$ of degree at most $m$.  Then
\[
\norm{v(fg)} \geq \Omega_m(1) \norm{v(f)} \norm{v(g)} \,.
\]
\end{claim}

The next claim gives us an understanding of how linear transformations of the underlying variables $X = (X_1, \dots , X_d)$ affect the coefficient norm of a polynomial. 

\begin{claim}\label{claim:polytovec-linear-transform}
Let $f$ be a homogeneous polynomial in the variables $X = (X_1, \dots , X_d)$ of degree equal to $m$.  Let $\Sigma$ be a $d \times d$ matrix.  Then
\[
\norm{v(f(\Sigma X))} \leq \left(\norm{\Sigma}_{\textsf{op}}\right)^m \norm{v(f(X))} \,. 
\]
\end{claim}
\begin{proof}
Note that 
\[
v(f(\Sigma X)) = \left(\underbrace{\Sigma \otimes \dots \otimes \Sigma}_{m}\right)v(f(X)) \,. 
\]
where $\otimes $ in the above denotes the Kronecker product.  Also,
\[
\norm{\underbrace{\Sigma \otimes \dots \otimes \Sigma}_{m}}_{\textsf{op}} = \left(\norm{\Sigma}_{\textsf{op}}\right)^m
\]
and now we immediately get the desired inequality.

\end{proof}

\subsection{Tensors and Polynomials with Multiple Sets of Variables}
We will need a few additional definitions dealing with polynomials and tensors involving multiple sets of variables, say $X^{(1)} = (X_1^{(1)}, \dots ,X_d^{(1)}), \dots , X^{(k)} = (X_1^{(k)}, \dots ,X_d^{(k)}) $.  We first prove the following property.

\begin{claim}\label{claim:tensor-symmetrization}
Let $k$ be a positive integer and consider $k$ distinct sets of $d$ formal variables, say $X^{(1)} = (X_1^{(1)}, \dots ,X_d^{(1)}), \dots , X^{(k)} = (X_1^{(k)}, \dots ,X_d^{(k)}) $.  Let $A$ be the tensor of polynomials defined as follows:
\[
A = \fl\left(\left(X^{(1)}\right)^{\otimes m_1} \right) \otimes \dots \otimes  \fl\left(\left(X^{(k)}\right)^{\otimes m_k} \right) \,.
\]
Note that $A$ is an order-$k$ tensor with dimensions $d^{m_1}, \dots , d^{m_k}$.  For any polynomial $P\left( X^{(1)}, \dots , X^{(k)}\right)$ that is homogeneous with degree exactly $m_i$ in the set of variables $X^{(i)}$ for all $i$, there is a unique tensor $T$ such that
\begin{itemize}
    \item The entries of $T$ are real numbers
    \item $\langle T, A \rangle = P\left( X^{(1)}, \dots , X^{(k)}\right)$
    \item The tensorization of any $1$-dimensional slice of $T$ along the $i$\ts{th} axis into a $\underbrace{d \times \dots \times d}_{m_i}$ tensor is symmetric.
\end{itemize}
\end{claim}
\begin{proof}
Note that each entry of $T$ may be indexed by the corresponding monomial of $A$.  The symmetry property implies that the entries of $T$ that are indexed by the same monomial must be the same.  Thus, the unique tensor $T$ is constructed by taking each monomial of $P$ and dividing its coefficient evenly among all of the entries of $T$ that are indexed by that monomial.
\end{proof}

In light of Claim \ref{claim:tensor-symmetrization} we may make the following definition:
\begin{definition}\label{def:symmetrictensorization}
Let $k$ be a positive integer and consider $k$ distinct sets of $d$ formal variables, say $X^{(1)} = (X_1^{(1)}, \dots ,X_d^{(1)}), \dots , X^{(k)} = (X_1^{(k)}, \dots ,X_d^{(k)}) $.  For a polynomial $P\left( X^{(1)}, \dots , X^{(k)}\right)$ that is homogeneous with degree $m_1, \dots , m_k$ in the sets of variables $X^{(1)}, \dots , X^{(k)}$ respectively, let $T_{\sym}(P)$ be the (unique) tensor constructed in Claim \ref{claim:tensor-symmetrization}.  We call $T_{\sym}(P)$ the symmetric tensorization of $P$. 
\end{definition}

We will need a few basic properties relating polynomials and their symmetric tensorizations.  We are mostly interested in the case when there are two sets of variables (i.e. $k = 2$).  In this case, the symmetric tensorizations will simply be matrices.  The first property is immediate from the definition.

\begin{claim}\label{claim:polyproducttosymmetrictensor}
Let $k$ be a positive integer and consider $k$ distinct sets of $d$ formal variables, say $X^{(1)} = (X_1^{(1)}, \dots ,X_d^{(1)}), \dots , X^{(k)} = (X_1^{(k)}, \dots ,X_d^{(k)}) $.  Let $P_1, \dots , P_k$ be homogeneous polynomials in $d$ variables.  Then
\[
T_{\sym}\left( P_1(X^{(1)}) \cdots P_k(X^{(k)}) \right) = v(P_1(X)) \otimes \dots \otimes v(P_k(X)) \,.
\]
\end{claim}
\begin{proof}
Let the degrees of $P_1, \dots , P_k$ be $m_1, \dots, m_k$ respectively.  Let 
\[
A = \fl\left(\left(X^{(1)}\right)^{\otimes m_1} \right) \otimes \dots \otimes  \fl\left(\left(X^{(k)}\right)^{\otimes m_k} \right) \,.
\]
Note that 
\[
\langle A , v(P_1(X)) \otimes \dots \otimes v(P_k(X)) \rangle = P_1(X^{(1)}) \cdots P_k(X^{(k)}) \,.
\]
Also note that each of the one dimensional slices of $v(P_1(X)) \otimes \dots \otimes v(P_k(X))$ is symmetric when put into a $d \times \dots \times d$ tensor.  Thus, Claim \ref{claim:tensor-symmetrization} implies that $v(P_1(X)) \otimes \dots \otimes v(P_k(X))$ is exactly the symmetric tensorization of $P_1(X^{(1)}) \cdots P_k(X^{(k)})$.
\end{proof}

\begin{claim}\label{claim:polytensorinjnormbound}
Consider two sets of $d$ formal variables, say $X^{(1)} = (X_1^{(1)}, \dots ,X_d^{(1)})$ and $ X^{(2)} = (X_1^{(2)}, \dots ,X_d^{(2)}) $.  Let $P(X^{(1)},  X^{(2)}), Q(X^{(1)},  X^{(2)})$ be polynomials that are homogeneous in each of the sets of variables.  Then
\[
\norm{T_{\sym}\left( PQ\right)}_{\op} \leq \norm{T_{\sym}(P)}_{\op} \norm{T_{\sym}(Q)}_{\op} \,.
\]
\end{claim}
\begin{proof}
Assume that $P$ has degrees $m_1, m_2$ and $Q$ has degrees $n_1, n_2$ in $X^{(1)},X^{(2)}$ respectively.  Let
\begin{align*}
A = \fl\left(\left(X^{(1)}\right)^{\otimes m_1 + n_1} \right) \otimes  \fl\left(\left(X^{(2)}\right)^{\otimes m_2 + n_2} \right) \,.
\end{align*}
Note that 
\[
\langle A, T_{\sym}(P) \otimes  T_{\sym}(Q) \rangle = P( X^{(1)} ,  X^{(2)})Q( X^{(1)} , X^{(2)})
\]
where $T_{\sym}(P) \otimes  T_{\sym}(Q)$ is the Kronecker product of the two matrices.  Thus, $T_{\sym}\left( PQ\right)$ can be written as an average of tensors that are equivalent to $T_{\sym}(P) \otimes  T_{\sym}(Q)$ up to permutations of the rows and columns.  Thus, by the triangle inequality 
\[
\norm{T_{\sym}\left( PQ\right)} _{\op}  \leq \norm{T_{\sym}(P) \otimes  T_{\sym}(Q)}_{\op} =  \norm{T_{\sym}(P)}_{\op} \norm{T_{\sym}(Q)}_{\op}\,.
\]
\end{proof}

\section{Hermite Polynomials, Generating Functions and Differential Operators}
In this section, we introduce the Hermite moment polynomials and their associated generating functions.  We then introduce several tools for manipulating generating functions using differential operators that will be crucial later on.  While some of these tools were introduced in \cite{liu2020settling}, we introduce many additional tools in this paper as we will need more precise characterizations and bounds on various quantities.

\subsection{Hermite Polynomials and their Generating Functions}\label{sec:hermite-basics}

Here we develop some basic machinery for working with Hermite polynomials and their generating functions. The first set of definitions and results mirrors the work in \cite{liu2020settling}.  We begin with a standard definition.

\begin{definition}
Let $\mcl{H}_m(x)$ be the univariate Hermite polynomials $\mcl{H}_0 = 1, \mcl{H}_1 = x, \mcl{H}_2 = x^2 - 1 \cdots $ defined by the recurrence
\[
\mcl{H}_m(x) = x\mcl{H}_{m-1}(x) - (m-1)\mcl{H}_{m-2}(x)
\]
\end{definition}

Note that in $\mcl{H}_m(x)$, the degree of each nonzero monomial has the same parity as $m$.  In light of this, we can write the following:
\begin{definition}
Let $\mcl{H}_m(x, y^2)$ be the homogenized Hermite polynomials e.g. $ \mcl{H}_2(x,y^2) = x^2 - y^2, \mcl{H}_3(x,y^2) = x^3 - 3xy^2$.
\end{definition}

It will be important to note the following fact:
\begin{claim}\label{claim:hermite-identity}
We have
\[
e^{xz - \frac{1}{2}y^2z^2} = \sum_{m = 0}^{\infty} \frac{1}{m!}\mcl{H}_m(x, y^2)z^m
\]
where the RHS is viewed as a formal power series in $z$ whose coefficients are polynomials in $x,y$.
\end{claim}

Now we define a multivariate version of the Hermite polynomials. 
\begin{definition}[Multivariate Hermite Polynomials]\label{def:Hermite-two-var}
Let $H_m(X,z)$ be a formal polynomial in variables $X = (X_1, \dots , X_d)$ whose coefficients are polynomials in $d$ variables $z_1, \dots , z_d$ that is given by
\[
H_m(X,z) = \mcl{H}_m( z_1X_1 + \dots + z_dX_d , X_1^2 + \dots + X_d^2)
\]
We call $H_m(X,z)$ the multivariate Hermite polynomials.  Note that $H_m$ is homogeneous of degree $m$ as a polynomial in $X_1, \dots , X_d$
\end{definition}

\begin{definition}[Hermite Moment Polynomials]\label{def:Hermite-final}
For a distribution $D$ on $\R^d$, we let
\[
h_{m,D}(X) = \E_{(z_1, \dots , z_d) \sim D}[H_m(X,z)]
\]
where we take the expectation of $H_m$ over $(z_1, \dots , z_d)$ drawn from $D$.  Note that $h_{m,D}(X)$ is a polynomial in $d$ variables $(X_1, \dots , X_d)$. We will omit the $D$ in the subscript when it is clear from context. We refer to $h_{m,D}(X)$ as the Hermite moment polynomials of $D$.
\end{definition}
We can extend the above definition to any function $f: \R^d \rightarrow \R$ (that is not necessarily a distribution). 
\begin{definition}
 For any function $f: \R^d \rightarrow \R$, we define
\[
h_{m,f}(X) = \int_{\R^d} f(z)H_m(X,z) dz \,.
\]
\end{definition}
\begin{remark}
Note that there is no ambiguity because this definition agrees with the above when $f$ is a distribution.  The extended definition will mostly be used for working with MPG functions that may not be normalized and may take on negative values.
\end{remark}

The first important observation is that the Hermite moment polynomials for a Gaussian can be written in a simple closed form via generating functions.
\begin{claim}\label{claim:key-hermite-identity}
Let $D = N(\mu, I + \Sigma)$.  Then
\[
\sum_{m=0}^{\infty} \frac{1}{m!} \cdot h_{m,D}(X) y^m = e^{\mu(X)y + \frac{1}{2}\Sigma(X)y^2}
\]
where both sides are viewed as formal power series in $y$ whose coefficients are polynomials in $X$.  
\end{claim}
\begin{proof}
Using Claim \ref{claim:hermite-identity}, the LHS may be rewritten as 
\begin{align*}
&\E_{z \sim D}  \left[ \sum_{m=0}^{\infty} \frac{1}{m!} \cdot H_{m}(X,z) y^m\right] = \E_{z \sim D}  \left[ e^{(z_1X_1 + \dots + z_dX_d)y - \frac{1}{2}(X_1^2 + \dots + X_d^2)y^2 }\right] 
\\ &=  C \int \exp\left(-\frac{1}{2}(z - \mu)^T(I + \Sigma)^{-1}(z - \mu) + z^TXy- \frac{1}{2}X^TXy^2\right)  dz 
\\ &= C \int \exp\left(-\frac{1}{2}(z - \mu  - (I + \Sigma)Xy )^T(I + \Sigma)^{-1}(z - \mu  - (I + \Sigma)Xy) + \mu^T Xy+ \frac{1}{2}X^T\Sigma Xy^2\right)  dz
\\&  = \exp\left(\mu(X)y + \frac{1}{2}\Sigma(X)y^2\right)\,.
\end{align*}
where in the above, $C$ is the normalization constant for a normal distribution with covariance $I + \Sigma$.  Note that for the last step, we used the fact that 
\[
\int \exp\left(\frac{1}{2}(z - \mu)^T(I + \Sigma)^{-1}(z - \mu)\right)dz = \int \exp\left(\frac{1}{2}(z - \mu  - (I + \Sigma)Xy )^T(I + \Sigma)^{-1}(z - \mu - (I + \Sigma)Xy)\right)dz \,.
\]
\end{proof}

By slightly modifying the proof of Claim \ref{claim:key-hermite-identity}, we can prove a more general result when we have a function given by a polynomial Gaussian.

\begin{claim}
Let $f(x): \R^d \rightarrow \R$ be given  by $f(x) = Q(x)G(x)$ where $G = N(\mu, I + \Sigma)$ is a Gaussian and $Q$ is a polynomial of degree $c$.  Then
\[
\sum_{m=0}^{\infty} \frac{1}{m!} \cdot h_{m,f}(X) y^m = P(Xy)e^{\mu(X)y + \frac{1}{2}\Sigma(X)y^2} 
\]
where $P$ is a polynomial in $d$ variables of degree at most $c$ and $Xy$ denotes the $d$-tuple of formal variables $(X_1y, \dots , X_dy)$.
\end{claim}
\begin{proof}
Using Claim \ref{claim:hermite-identity}, the LHS may be rewritten as 
\begin{align*}
&\int_{\R^d} f(z)  \left[ \sum_{m=0}^{\infty} \frac{1}{m!} \cdot H_{m}(X,z) y^m\right] dz= \int_{\R^d} f(z)  \left[ e^{(z_1X_1 + \dots + z_dX_d)y - \frac{1}{2}(X_1^2 + \dots + X_d^2)y^2 }\right] dz
\\ &=  C \int Q(z) \exp\left(-\frac{1}{2}(z - \mu)^T(I + \Sigma)^{-1}(z - \mu) + z^TXy- \frac{1}{2}X^TXy^2\right)  dz 
\\ &= C \int Q(z) \exp\left(-\frac{1}{2}(z - \mu  - (I + \Sigma)Xy )^T(I + \Sigma)^{-1}(z - \mu  - (I + \Sigma)Xy) + \mu^T Xy+ \frac{1}{2}X^T\Sigma Xy^2\right)  dz 
\\ &= C e^{\mu(X)y + \frac{1}{2}\Sigma(X)y^2}   \int Q(z) \exp\left(-\frac{1}{2}(z - \mu  - (I + \Sigma)Xy )^T(I + \Sigma)^{-1}(z - \mu  - (I + \Sigma)Xy)\right) dz\,.
\end{align*}
where in the above, $C$ is a constant.  Now we can make the change of variables $z = (I + \Sigma)^{1/2}z' + \mu + (I + \Sigma)Xy$ and deduce that 
\begin{align*}
\int Q(z) \exp\left(-\frac{1}{2}(z - \mu  - (I + \Sigma)Xy )^T(I + \Sigma)^{-1}(z - \mu  - (I + \Sigma)Xy)\right) dz \\ = \det(I + \Sigma)^{1/2} \int   Q\left((I + \Sigma)^{1/2}z + \mu + (I + \Sigma)Xy \right) \exp \left( -\frac{1}{2}\norm{z}^2 \right)dz \\ = P(\mu + (I + \Sigma)Xy)
\end{align*}
for some polynomial $P$ of degree at most $c$.  Putting everything together gives us the desired result.
\end{proof}

We now have a few simple consequences of the above.
\begin{corollary}\label{coro:hermite-of-mixture}
Let $\mcl{M} = w_1 G_1 + \dots w_kG_k$ be a mixture of Gaussians where $G_j = N(\mu_j, I + \Sigma_j)$.  Then
\[
\sum_{m=0}^{\infty} \frac{1}{m!} \cdot h_{m,\mcl{M}}(X) y^m = w_1e^{\mu_1(X)y + \frac{1}{2}\Sigma_1(X)y^2} + \dots + w_ke^{\mu_k(X)y + \frac{1}{2}\Sigma_k(X)y^2}
\]
\end{corollary}

\begin{corollary}\label{coro:hermite-of-polymixture}
Let $f(x) = Q_1(x)G_1(x) + \dots + Q_k(x)G_k(x)$ be a degree-$c$ MPG function in $\R^d$ where $G_j = N(\mu_j, I + \Sigma_j)$.  Then there are polynomials $P_1(X), \dots, P_k(X)$ of degree at most $c$ in formal variables $X = (X_1, \dots , X_d)$  such that 
\[
\sum_{m=0}^{\infty} \frac{1}{m!} h_{m,f}(X)y^m = P_1(Xy)e^{\mu_1(X)y + \frac{1}{2}\Sigma_1(X)y^2} + \dots + P_k(Xy)e^{\mu_k(X)y + \frac{1}{2}\Sigma_k(X)y^2}
\]
where $Xy$ denotes the $d$-tuple of formal variables $(X_1y, \dots , X_dy)$.
\end{corollary}

\subsection{The Adjusted Characteristic Function and its Properties}\label{sec:char-func}
The characteristic function is a well-known concept in probability.  Here, we use a modified notion of an adjusted characteristic function.  One of the key components in our paper is Theorem \ref{thm:identifiability}, which relates the $L^1$ distance between MPG functions (note this is equivalent to TV distance for mixtures of Gaussians and MPG distributions) to the coefficient-norm distance between their Hermite moment polynomials.  The adjusted characteristic function will play a key role in proving Theorem \ref{thm:identifiability} because its inverse map gives us a way to map from a generating function for Hermite moment polynomials back to a distribution.

\begin{definition}\label{def:adj-char-func}
For a function $f$ on $\R^d$, we define its adjusted characteristic function $\wt{f}: \R^d \rightarrow \C$ as 
\[
\wt{f}(X) = \int_{\R^d} f(z) \left[ e^{i z \cdot X + \frac{1}{2} \norm{X}^2}\right] dz
\]
where $i = \sqrt{-1}$.
\end{definition}

Note that for distributions, the adjusted characteristic function is the characteristic function multiplied by $\frac{1}{2}\norm{X}^2$.  Now we define the inverse map.

\begin{definition}
For a function $g: \R^d \rightarrow \C$, we define $\chi  g$, a function from $\R^d$ to $\C$, as follows:
\[
\chi g(t) = \frac{1}{(2\pi)^d}\int g(X) e^{-\frac{1}{2}\norm{X}^2 - it \cdot X} dX
\]
where in the integral above, $X$ ranges over all of $\R^d$.
\end{definition}

It is straight-forward to verify that the transformation defined above indeed inverts the adjusted characteristic function.
\begin{fact}\label{fact:inversion}
For a function $f$ on $\R^d$, 
\[
\chi\wt{f} = f \,.
\]
\end{fact}

A key property of  the adjusted characteristic function is that its output is equivalent to plugging in $y = i$ into the generating function for its Hermite moment polynomials.
\begin{claim}\label{claim:invchar-formula}
Let $f: \R^d \rightarrow \R$ be a function.  Then
\[
\wt{f}(X) = \sum_{m = 0}^{\infty} \frac{i^m}{m!}h_{m,f}(X) \,.
\]
\end{claim}
\begin{proof}
Note that by Claim \ref{claim:hermite-identity},
\[
\sum_{m = 0}^{\infty} \frac{i^m}{m!}h_{m,f}(X) = \int_{\R^d} f(z) e^{ i z \cdot X + \frac{1}{2}\norm{X}^2} dz = \wt{f}(X) \,,
\]
as desired.
\end{proof}

As a consequence of the above and Corollary \ref{coro:hermite-of-polymixture}, we have 

\begin{corollary}\label{coro:adj-char-motivation}
If we have a degree-$c$ MPG function $f = Q_1(x)G_1 + \dots + Q_k(x)G_k$ then
\[
\wt{f}(X) = P_1(iX)e^{ i\mu_1(X) - \frac{1}{2}\Sigma_1(X)} + \dots +  P_k(iX)e^{ i\mu_k(X)  -\frac{1}{2}\Sigma_k(X)}  
\]
for some polynomials $P_1, \dots , P_k$ of degree at most $c$ with real coefficients.
\end{corollary}

In light of the above, we know that the adjusted characteristic function maps a function to a generating function for its Hermite moment polynomials.  Recall that our goal is to prove that small distance between Hermite moment polynomials implies small TV distance.  This means that we need to understand the $L^1$ norm of the inverse adjusted characteristic function.  In the remainder of this subsection, we prove some basic quantitative bounds on the inverse adjusted characteristic function that will be used later on.

\subsubsection{Computations in 1D}
The following two identities follow from direct computation.
\begin{claim}\label{claim:charfunc-formula}
For a real number $t \in \R$,
\[
\int_{-\infty}^{\infty} x^m e^{-\frac{1}{2}x^2} e^{-itx} dx = (-i)^m\sqrt{2\pi}e^{-t^2/2} \mathcal{H}_m(t)
\]
where recall $\mathcal{H}_m(x)$ is the univariate Hermite polynomial.
\end{claim}

\begin{fact}\label{fact:hermite-orthogonality}
\[
\int_{-\infty}^{\infty} \mathcal{H}_{m_1}(x)\mathcal{H}_{m_2}(x) \frac{e^{-x^2/2}}{\sqrt{2\pi}} dx = 1_{m_1 = m_2} (m_1)!
\]
\end{fact}

\subsubsection{Bounds on the Inverse Adjusted Characteristic Function}

The next result gives us several important properties for certain inverse adjusted characteristic functions corresponding to polynomial Gaussians.

\begin{claim}\label{claim:invchar-of-polynomial}
Let $p(X)$ be a polynomial with real coefficients in $d$ variables $X_1, \dots , X_d$ that is homogeneous of degree $m$.  Consider a Gaussian in $\R^d$, $G = N(\mu, I + \Sigma)$.  For $X \in \R^d$ let
\[
g(X) = i^m p(X)e^{i \mu(X) - \frac{1}{2}\Sigma(X)} \,.
\]
Then
\[
\chi g(t) = q(t) G(t)
\]
for some polynomial $q$ of degree at most $m$ with real coefficients.  Furthermore, 
\[
\E_{t \sim G}\left[ q(t)^2 \right] \leq  m! (\norm{(I + \Sigma)^{-1}}_{\textsf{op}})^m \norm{v(p)}^2
\]
where recall $v(p)$, defined in Definition \ref{def:poly-to-vec}, is the vectorization of the coefficients of $p$.
\end{claim}
\begin{proof}

We have
\begin{align*}
\chi g(t) = \frac{1}{(2\pi)^d}\int i^m p(X) e^{ - \frac{1}{2}X^T( I + \Sigma)X - i(t - \mu) \cdot X} dX 
\end{align*}
Substituting $X \rightarrow (I + \Sigma)^{-1/2}Y$ for $Y = (Y_1, \dots , Y_d)$ in the above we get
\[
\chi g(t) = \frac{1}{(2\pi)^d\det(I + \Sigma)^{1/2}} \int i^m p((I + \Sigma)^{-1/2}Y) e^{-\frac{1}{2}\norm{Y}^2 - i (I + \Sigma)^{-1/2}(t - \mu) \cdot Y}dY \,.
\]
To compute the above integral, note that $p((I + \Sigma)^{-1/2}Y)$ is a polynomial in $Y_1, \dots , Y_d$ that is homogeneous of degree $m$.  Let $h(Y) = p((I + \Sigma)^{-1/2}Y)$.  Let $s = (I + \Sigma)^{-1/2}(t - \mu)$ and let its coordinates be $s_1, \dots , s_d$.  We now separate $h$ into monomials and consider one monomial at a time.  Consider a monomial say $Y_1^{a_1}Y_2^{a_2} \cdots Y_d^{a_d}$ for some integers $a_1, \dots , a_d$.  Note that the term inside the exponential can be factored coordinate-wise.  Thus, we can apply Claim \ref{claim:charfunc-formula} to compute the integral as follows:
\begin{align*}
\int i^m Y_1^{a_1} \cdots Y_d^{a_d} e^{-\frac{1}{2}\norm{Y}^2 - is \cdot Y} dY = i^m \prod_{j=1}^d \int_{-\infty}^{\infty}x^{a_j} e^{-\frac{1}{2}x^2} e^{-is_jx}  dx \\
= (2\pi)^{d/2} e^{-\frac{1}{2}(s_1^2 + \dots + s_d^2)}\prod_{j=1}^d \mathcal{H}_{a_j}(s_j) \,.
\end{align*}

We denote the coefficients of $h$ by $c_{a_1, \dots , a_d}$.  Combining the above over all monomials, we get
\begin{align*}
 \chi g(t) &= \frac{1}{(2\pi)^{d/2}\det(I + \Sigma)^{1/2}} e^{-\frac{1}{2} \norm{s}^2}  \sum_{a_1, \dots , a_d} c_{a_1, \dots , a_d}\prod_{j=1}^d \mathcal{H}_{a_j}(s_j) \\ &=  \left( \sum_{a_1, \dots , a_d} c_{a_1, \dots , a_d}\prod_{j=1}^d \mathcal{H}_{a_j}(s_j) \right) N(\mu, I + \Sigma)(t)\,.
\end{align*}
From the above we deduce
\[
q(t) =  \sum_{a_1, \dots , a_d} c_{a_1, \dots , a_d}\prod_{j=1}^d \mathcal{H}_{a_j}(s_j)
\]
and it is clear that $q$ is a polynomial of degree at most $m$ in $t$ (since $s$ is a linear function of $t$).  It remains to bound $\E_{t \sim G}\left[ q(t)^2 \right] $.  Note that
\begin{align*}
\E_{t \sim G}\left[ q(t)^2 \right] &= \int_{\R^d} \left( \sum_{a_1, \dots , a_d} c_{a_1, \dots , a_d}\prod_{j=1}^d \mathcal{H}_{a_j}(s_j) \right)^2 N(\mu, I + \Sigma)(t) dt \\ &= \int_{\R^d} \left(\sum_{a_1, \dots , a_d} c_{a_1, \dots , a_d}\prod_{j=1}^d \mathcal{H}_{a_j}(s_j) \right)^2 N(0, I)(s) ds \\ &= \sum_{a_1, \dots , a_d}\sum_{a_1', \dots , a_d'}\int_{\R^d} c_{a_1, \dots, a_d}c_{a_1',\dots , a_d'}\prod_{j=1}^d \mathcal{H}_{a_j}(s_j){H}_{a_j'}(s_j)  \frac{e^{-\norm{s}^2/2}}{(2\pi)^{d/2}} ds  \,.
\end{align*}
However, since the integral factorizes over the different coordinates of $s$, by Fact \ref{fact:hermite-orthogonality}, the integral evaluates to $0$ unless $a_1, \dots , a_d = a_1', \dots , a_d'$ and overall, we get
\begin{align*}
\int_{\R^d} \left( \sum_{a_1, \dots , a_d}c_{a_1, \dots , a_d}\prod_{j=1}^d \mathcal{H}_{a_j}(s_j) \right)^2 \frac{e^{-\norm{s}^2/2}}{(2\pi)^{d/2}} ds  = \sum_{a_1, \dots , a_d} c_{a_1, \dots , a_d}^2 \prod_{j=1}^d \int_{-\infty}^{\infty} \mathcal{H}_{a_j}(x)^2 \frac{e^{-x^2/2}}{\sqrt{2\pi}} dx \\ = \sum_{a_1, \dots , a_d} c_{a_1, \dots , a_d}^2 \prod_{j=1}^d a_j! = m!\sum_{a_1, \dots , a_d} \binom{m}{a_1, \dots , a_d} \left(\frac{c_{a_1, \dots , a_d}}{\binom{m}{a_1, \dots , a_d}}\right)^2 = m!\norm{v(h)}^2 \,.
\end{align*}

Finally, by Claim \ref{claim:polytovec-linear-transform}, $\norm{v(h)}^2 \leq (\norm{(I + \Sigma)^{-1}}_{\textsf{op}})^m \norm{v(p)}^2$ so combining with the previous inequality, we deduce 
\[
\E_{t \sim G}\left[ q(t)^2 \right] \leq m! (\norm{(I + \Sigma)^{-1}}_{\textsf{op}})^m \norm{v(p)}^2 \,.
\]

\end{proof}

As a consequence of the previous claim, we have
\begin{corollary}\label{coro:one-to-one-map}
Let $G_1 = N(\mu_1, I + \Sigma_1), \dots , G_k = N(\mu_k, I + \Sigma_k)$ be Gaussians.  Let $c$ be a constant.  There is a one-to-one correspondence between polynomials $Q_1, \dots , Q_k$ of degree at most $c$ with real coefficients and polynomials $P_1, \dots , P_k$ of degree at most $c$ with real coefficients given by the following map:

The adjusted characteristic function of $f(x) = Q_1(x)G_1 + \dots + Q_k(x)G_k$ is 
\[
\wt{f}(X) = P_1(iX)e^{i \mu_1(X) - \frac{1}{2}\Sigma_1(X)} + \dots  + P_k(iX)e^{i \mu_k(X) - \frac{1}{2}\Sigma_k(X)} \,.
\]
\end{corollary}
\begin{proof}
This follows immediately from Corollary \ref{coro:adj-char-motivation} and Claim \ref{claim:invchar-of-polynomial}.
\end{proof}

Claim \ref{claim:invchar-of-polynomial} also implies a bound on the $L^1$ norm of the inverse adjusted characteristic function of a polynomial in terms of the size of its coefficients.

\begin{claim}\label{claim:key-charfunc-bound}
Let $p(X)$ be a polynomial in $d$ variables $X_1, \dots , X_d$ that is homogeneous of degree $m$ with real coefficients.  Then
\[
\norm{\chi(i^m p(X))}_1 \leq \sqrt{m!} \norm{v(p)}_2\,.
\]

\end{claim}
\begin{proof}
By Claim \ref{claim:invchar-of-polynomial} (with $\mu = 0, \Sigma = 0$), 
\[
\chi(i^m p)(t) = q(t)N(0,I)(t)
\]
where $q$ is a polynomial of degree at most $m$ such that $\E_{t \sim N(0,I)}[q(t)^2] \leq m!\norm{v(p)}^2$.  Now 
\[
\norm{\chi(i^m p)}_1 = \int_{\R^d} |q(t)| N(0,I)(t)dt \leq \sqrt{\int_{\R^d} |q(t)|^2 N(0,I)(t)dt}  \leq \sqrt{m!} \norm{v(p)}
\]
where we used Cauchy Schwarz in the first inequality.
\end{proof}

\subsection{Generating Function Terminology}\label{sec:gen-func-terms}
Here we introduce some general terminology for working with generating functions related to mixtures of Gaussians and mixtures of polynomial Gaussians.  In light of Corollaries \ref{coro:hermite-of-mixture} and \ref{coro:hermite-of-polymixture}, it will be useful to translate between generating functions consisting of sums of exponentials and their expansions as formal power series in $y$ whose coefficients are polynomials in $X$ e.g. 
\[
f(X,y) = \sum_{j=0}^{\infty} \frac{f_j(X)}{j!}y^j \,.
\]
For such an expression, we use the following terminology.

\begin{definition}
Given a formal power series in $y$, say 
\[
f(X,y) = \sum_{j=0}^{\infty} \frac{f_j(X)}{j!}y^j \,,
\]
where the coefficients $f_0(X), f_1(X), \dots $ are polynomials in formal variables $X = (X_1, \dots , X_d)$ and have real coefficients, we call the polynomials $f_0(X), f_1(X), \dots $ the primary terms of $f$.  
\end{definition}

We also introduce terminology for dealing with mixtures of polynomial Gaussians.  

\begin{definition}
Given Gaussians $G_1 = N(\mu_1, I + \Sigma_1), \dots , G_k = N(\mu_k, I + \Sigma_k)$ in $\R^d$, and polynomials $P_1(X), \dots , P_k(X)$ in formal variables $X = (X_1, \dots , X_d)$ with real coefficients, we say that the generating function of the polynomial combination is 
\[
f(X,y) = P_1(Xy)e^{ \mu_1(X)y + \frac{1}{2}\Sigma_1(X)y^2} + \dots +  P_k(Xy)e^{ \mu_k(X)y + \frac{1}{2}\Sigma_k(X)y^2} 
\]
where we view $f$ as a function of $y$ with indeterminates $X$.  Recall that $Xy$ denotes the $d$-tuple of variables $(X_1y, \dots , X_dy)$. 
\end{definition}
\begin{remark}
Note that by Corollary \ref{coro:hermite-of-mixture}, if we have a mixture of Gaussians $\mcl{M} = w_1G_1 + \dots + w_kG_k$, then the generating function of the polynomial combination $P_1(X) = w_1, \dots , P_k(X) = w_k$ is 
\[
f(X,y) =  w_1e^{ \mu_1(X)y + \frac{1}{2}\Sigma_1(X)y^2} + \dots +  w_ke^{ \mu_k(X)y + \frac{1}{2}\Sigma_k(X)y^2}  = \sum_{m=0}^{\infty} \frac{1}{m!} h_{m, \mcl{M}}(X)y^m \,.
\]
In other words, the primary terms in the formal power series expansion of $f(y)$ are exactly the Hermite moment polynomials of the mixture.  More generally, Corollary \ref{coro:hermite-of-polymixture} and Corollary \ref{coro:one-to-one-map} imply that for an MPG function $\mcl{M} = Q_1(x)G_1 + \dots + Q_k(x)G_k$, there are corresponding polynomials $P_1, \dots , P_k$, such that the generating function of the polynomial combination is
\[
f(X,y) =  P_1(Xy)e^{ \mu_1(X)y + \frac{1}{2}\Sigma_1(X)y^2} + \dots +  P_k(Xy)e^{ \mu_k(X)y + \frac{1}{2}\Sigma_k(X)y^2}  = \sum_{m=0}^{\infty} \frac{1}{m!} h_{m, \mcl{M}}(X)y^m \,.
\]
\end{remark}

In the exposition, we will use the following (informal) terminology.  Note that for a Gaussian $N(\mu, I + \Sigma)$, we often associate it with a generating function of the form $e^{\mu(X)y + \frac{1}{2}\Sigma(X)y^2}$.  
\begin{itemize}
    \item When we work with the actual pdfs of Gaussians and write e.g. expressions of the form $f = Q_1(x)G_1 + \dots + Q_k(x)G_k$, we say that we are working in \textbf{distribution space}
    \item When we work with expressions of the form $e^{\mu(X)y + \frac{1}{2}\Sigma(X)y^2}$ and e.g. expand them as power series containing Hermite moment polynomials, we say that we are working in \textbf{generating function space}
\end{itemize}

\subsection{Differential Operators and their Compositions}
Before moving on to the main proofs, we need to introduce one more piece of machinery: differential operators.  Later on, differential operators will play a crucial role in allowing us to manipulate generating functions and derive useful identities.  In this section, we present a few basic results that will be used throughout the paper.

We will frequently work with operators given by 
\[
\mcl{D} =  \partial - (a(X) + b(X)y)
\]
where the partial derivative is taken with respect to $y$, $a(X)$ is a (homogeneous) linear function and $b(X)$ is a (homogeneous) quadratic function.  Note that if applied to a formal power series
\[
f(y) = \sum_{j=0}^{\infty} \frac{Q_j(X)}{j!}y^j \,,
\]
the terms of the resulting power series are 
\[
\mcl{D}(f(y)) = \sum_{j=0}^{\infty} \frac{R_j(X)}{j!}y^j
\]
where $R_j(X) = Q_{j+1}(X) - a(X)Q_j(X) - jb(X)Q_{j-1}(X)$.  In particular, if for all $j$, $Q_j$ is homogeneous of degree $j$ in $X$, then the primary terms $R_j(X)$ are homogeneous and of degree $j+1$.

It will be important to understand compositions of differential operators.  We now prove several basic properties that will be used later on.  The first claim allows us to rewrite a composition of differential operators as a higher order differential operator with polynomials as coefficients.

\begin{claim}\label{claim:diff-operator-expansion}
Consider a composition of differential operators
\[
\mcl{D} = ( \partial - (a_k(X) + b_k(X)y)) \cdots ( \partial - (a_1(X) + b_1(X)y))
\]
where each $a_j$ is linear and each $b_j$ is quadratic.  Then $\mcl{D}$ can be rewritten in the form
\[
 \partial^k + R_{k-1}(X,y) \partial^{k-1} + \dots + R_0(X,y)
\]
where 
\begin{itemize}
    \item Each $R_j$ is a polynomial of degree at most $k - j$ in $y$ 
    \[
    R_j(X,y) =  R_{j, k-j}(X)y^{k-j} + \dots + R_{j, 0}(X)
    \]
    \item Each of the polynomials $R_{j,l}$ is homogeneous in $X$ with degree $k - j + l$, and is $(O_k(1), k)$-simple with respect to $\{a_1(X), b_1(X), \dots , a_k(X), b_k(X) \}$.
\end{itemize}
\end{claim}
\begin{proof}
We will use induction on $k$.  The base case is clear.  Now, assume that we have written the operator
\[
\mcl{D}_{k-1}=  ( \partial - (a_{k-1}(X) + b_{k-1}(X)y)) \cdots ( \partial - (a_1(X) + b_1(X)y))
\]
in the desired form
\[
\mcl{D}_{k-1} = \partial^{k-1} + R_{k-2}(X,y) \partial^{k-2} + \dots + R_0(X,y) \,.
\]
When we apply the last differential operator, we get 
\begin{align*}
 ( \partial - (a_{k}(X) + b_{k}(X)y)) \mcl{D}_{k-1} &=  \partial( \partial^{k-1} + R_{k-2}(X,y) \partial^{k-2} + \dots + R_0(X,y)) \\ & \quad - (a_k(X) + b_k(X)y)( \partial^{k-1} + R_{k-2}(X,y) \partial^{k-2} + \dots + R_0(X,y)) \,.
\end{align*}
It is clear that the second term can be written in the desired form.  To deal with the first term, we can simply use the product rule i.e.
\[
\partial ( R_j(X,y) \partial^j) = R_j(X,y) \partial^{j+1} + \partial (R_j(X,y)) \partial^j
\]
to write the entire differential operator in the desired form.

\end{proof}

The next claim implies that applying a differential operator of the form $\partial - (a(X) + b(X)y)$ cannot annihilate (or nearly annihilate) a polynomial unless $a(X),b(X)$ are both close to $0$.  

\begin{claim}\label{claim:prod-and-deriv-bound}
Consider a polynomial $P(X, y)$ of degree $k$ in $y$
\[
P(X,y) = P_0(X) + P_1(X)y + \dots + P_k(X)y^k
\]
where the coefficients $P_j(X)$ are polynomials in $X$ that are homogeneous of degree $k' + j$ for some constant $k'$.  Let 
\[
R(X,y) = -(a(X) + b(X)y)P(X,y) + \partial(P(X,y))
\]
where $a(X)$, $b(X)$ are a (homogeneous) linear and quadratic respectively  and the partial derivative is taken with respect to $y$.  Assume 
\[
\delta \leq \norm{v(a(X))} + \norm{v(b(X))} \leq \delta^{-1}  \,.
\]
Then
\[
\norm{v_y(R(X,y))} \geq (0.1\delta)^{O_{k,k'}(1)}  \norm{v_y(P(X,y))}  \,.
\]
\end{claim}
\begin{proof}
We will induct on $k$, the degree of $P$.  The case when $k = 0$ follows from Claim \ref{claim:prod-lower-bound}.  Now assume that there is a constant $c_{k-1,k'}$ for which the desired statement holds for all polynomials $P$ of degree at most $k-1$ in $y$.  Note that the coefficients of $y^{k+1}, y^k$ in $R$ are 
\begin{align*}
R_{k+1}(X) &= -b(X)P_k(X) \\
R_k(X) &= -b(X)P_{k-1}(X) -a(X)P_k(X) 
\end{align*}
Thus by Claim \ref{claim:polytensor-prod-bound} and Claim \ref{claim:prod-lower-bound},
\begin{align*}
\norm{v(R_{k+1}(X))} &\geq \Omega_{k,k'}(1) \norm{v(b(X))}\norm{v(P_k(X))} \\
\norm{v(R_k(X))} & \geq \Omega_{k,k'}(1)\norm{v(a(X))}\norm{v(P_k(X))} - O_{k,k'}(1) \norm{v(b(X))}\norm{v(P_{k-1}(X))} 
\end{align*}
If for some parameter $\eps$, 
\[
\norm{v(P_k(X))} \geq \eps \norm{v_y(P(X,y)}
\]
then for some sufficiently large constant $K$ depending only on $k,k'$,
\begin{align*}
\norm{v_y(R(X,y))} &\geq \frac{1}{2}\norm{v(R_{k+1}(X))} + \frac{\eps}{K}\norm{v(R_k(X))} \\ 
&\geq  \norm{v(b(X))} \left( \Omega_{k,k'}(1)\norm{v(P_k(X))} - \frac{\eps}{K} O_{k,k'}(1)\norm{v(P_{k-1}(X))} \right) + \Omega_{k,k'}(1) \frac{\eps}{K}\norm{v(a(X))}\norm{v(P_k(X))}\\
&\geq \Omega_{k,k'}(1)\frac{\eps^2}{K}( \norm{v(a(X))} +  \norm{v(b(X))} ) \norm{v_y(P(X,y))} \\
& \geq \Omega_{k,k'}(1)\eps^2\delta \norm{v_y(P(X,y))} \,.
\end{align*}
On the other hand, if $\eps$ is sufficiently small and
\[
\norm{v(P_k(X))} \leq \eps \norm{v_y(P(X,y)} \,,
\]
we may use the induction hypothesis on the polynomial
\[
P'(X,y) =  P_0(X) + P_1(X)y + \dots + P_{k-1}(X)y^{k-1} \,.
\]
Note
\[
R(X,y) = -(a(X) + b(X)y)P'(X,y) + \partial(P'(X,y)) - (a(X) + b(X)y)P_k(X)y^k + kP_k(X)y^{k-1}
\]
so by the induction hypothesis
\begin{align*}
\norm{v_y(R(X,y))} &\geq (0.1\delta)^{c_{k-1,k'}}\norm{v_y(P'(X,y))} - O_{k,k'}(1) (1 +  \norm{v(a(X))} + \norm{v(b(X))})\norm{v(P_k(X))} \\
&\geq \left((0.1\delta)^{c_{k-1,k'}}(1-\eps) - O_{k,k'}(1) \eps \delta^{-1}\right) \norm{v_y(P(X,y))} \,.
\end{align*}
Choosing $\eps = (0.1\delta)^{c_{k,k'}}$ for $c_{k,k'}
$ sufficiently large in terms of $k,k', c_{k-1,k'}$, in both cases we get
\[
\norm{v_y(R(X,y))} \geq (0.1\delta)^{O_{k,k'}(1)}  \norm{v_y(P(X,y))}  \,,
\]
completing the proof.
\end{proof}

By repeatedly applying the previous claim, we can lower bound the order-$0$ term when a composition of differential operators $(\partial - (a_k(X) + b_k(X)y)) \cdots (\partial - (a_1(X) + b_1(X)y))$ is expanded into an order-$k$ differential operator with polynomial coefficients (recall Claim \ref{claim:diff-operator-expansion}).

\begin{claim}\label{claim:prod-linear-forms}
Consider a composition of differential operators
\[
\mcl{D} = (\partial - (a_k(X) + b_k(X)y)) \cdots (\partial - (a_1(X) + b_1(X)y))
\]
where each $a_j$ is linear and each $b_j$ is quadratic.  Assume that $\mcl{D}$ is rewritten in the form
\[
\mcl{D} = \partial^k + R_{k-1}(X,y) \partial^{k-1} + \dots + R_0(X,y) \,.
\]
Also assume that for some constant $\delta$,  
\[
\delta < \norm{v(a_j(X))} +  \norm{v(b_j(X))} < \delta^{-1} 
\]
for all $j$.  There exists a (sufficiently large) constant $C_k$ depending only on $k$ such that
\[
\norm{v_y(R_0(X,y))} \geq  (0.1\delta)^{C_k} \,.
\]
\end{claim}
\begin{proof}
We will again use induction on $k$.  The base case is clear.  Now write
\begin{align*}
\mcl{D}_{k-1} = (\partial - (a_{k-1}(X) + b_{k-1}(X)y)) \cdots (\partial - (a_1(X) + b_1(X)y))\\  = \partial^{k-1} + T_{k-2}(X,y) \partial^{k-2} + \dots + T_0(X,y)
\end{align*}
for some polynomials $T_0,T_1, \dots , T_{k-2}$ satisfying the conditions in Claim \ref{claim:diff-operator-expansion}.  The induction hypothesis gives us
\[
\norm{v_y(T_0(X,y))} \geq  (0.1\delta)^{C_{k-1}} \,.
\]
Now 
\[
R_0(X,y) =  - (a_k(X) + b_k(X)y))T_0(X,y) + \partial(T_0(X,y)) \,,
\]
and applying Claim \ref{claim:prod-and-deriv-bound} completes the induction.

\end{proof}

\subsection{Applying Differential Operators to Generating Functions}

In this section, we analyze what happens when we apply certain differential operators to specific types of generating functions.  The results in this section are all from \cite{liu2020settling} and can be verified through direct computation.

\begin{claim}\label{claim:degree-reduction}
Let $\partial $ denote differentiation with respect to $y$.  If 
\[
f(y) = P(y,X)e^{a(X)y + \frac{1}{2}b(X)y^2}
\]
where $P$ is a polynomial in $y$ of degree $k$ (whose coefficients are polynomials in $X$) then
\[
(\partial - (a(X) + yb(X)))f(y) = Q(y,X)e^{a(X)y + \frac{1}{2}b(X)y^2}
\]
where $Q$ is a polynomial in $y$ with degree exactly $k-1$ whose leading coefficient is $k$ times the leading coefficient of $P$.
\end{claim}

\begin{corollary}\label{corollary:null-operator}
Let $\partial $ denote the differential operator with respect to $y$.  If 
\[
f(y) = P(y,X)e^{a(X)y + \frac{1}{2}b(X)y^2}
\]
where $P$ is a polynomial in $y$ of degree $k$ then
\[
(\partial - (a(X) + yb(X)))^{k+1}f(y) = 0.
\]
\end{corollary}


\section{Strong Observability: Hermite Moments to TV Distance}\label{sec:hermite-to-tv}

In this section, we prove our main observability theorem that if two degree-$m$ MPG functions with $k$ components are $\eps$-close on their first $O_{k,m}(1)$ Hermite moment polynomials, then the functions must be $\wt{O}(\eps)$-close in $L^1$ distance.  The theorem is stated formally below.

\begin{theorem}\label{thm:identifiability}
Let $G_1 = N(\mu_1, I + \Sigma_1), \dots , G_k = N(\mu_k, I + \Sigma_k)$ be a set of Gaussians in $(\alpha,\beta)$-regular form.  Let $Q_1(X), \dots , Q_k(X)$ be $d$-variate polynomials of degree at most $m$.  Define the function $g: \R^d \rightarrow \R$ as 
\[
g(x) = Q_1(x)G_1(x) + \dots + Q_k(x)G_k(x) \,.
\]
There exists a constant $C$ depending only on $k,m$ such that the following holds.  If for all $j$ with $0 \leq j \leq C$, we have
\[
\norm{v(h_{j,g}(X))} \leq \eps
\]
then we must have
\[
\norm{g}_1 \leq (2 + \alpha + \beta)^C \eps \,.
\]
\end{theorem}
\begin{remark}
Note that the above gives observability because we can simply set $g$ to be equal to the difference of two MPG-functions. 
\end{remark}

\noindent In this section, we will use the following conventions.
\begin{itemize}
    \item We have Gaussians $G_1 = N(\mu_1, I + \Sigma_1), \dots , G_k = N(\mu_k, I + \Sigma_k)$
    \item $\mcl{D}_j$ denotes the differential operator $(\partial - (\mu_j(X) + \Sigma_j(X)y))$ where the partial derivative is taken with respect to $y$.
\end{itemize}

\subsection{Recurrence for Hermite Moment Polynomials}

One of the key ingredients in the proof of Theorem \ref{thm:identifiability} is writing down a recurrence for the Hermite moment polynomials of an MPG function.  This is done in the following lemma.

\begin{lemma}\label{lem:hermite-moment-recurrence}
Let $G_1 = N(\mu_1, I + \Sigma_1), \dots , G_k = N(\mu_k, I + \Sigma_k)$ be Gaussians and $Q_1(X), \dots , Q_k(X)$ be $d$-variate polynomials of degree at most $m$.  Let 
\[
g(x) = Q_1(x)G_1(x) + \dots + Q_k(x)G_k(x) \,.
\]
Let $\kappa = (m+1)(2^k - 1)$.  Then there are $d$-variate polynomials $R_{j,l}(X)$ for $0 \leq j \leq \kappa$ and $0 \leq  l \leq \kappa - j$ such that
\begin{itemize}
    \item $R_{j,l}(X)$ is homogeneous of degree $\kappa - j + l$, $R_{\kappa,0}(X) = 1$
    \item $R_{j,l}(X)$ is $\left(O_{k,m}(1), O_{k,m}(1)\right)$-simple with respect to $\{\mu_1(X), \Sigma_1(X), \dots , \mu_k(X), \Sigma_k(X) \}$
    \item For all integers $a \geq \kappa$,
    \[
    \sum_{j=0}^{\kappa}\sum_{l = 0}^{\kappa - j} \frac{h_{a - \kappa + j - l,g}(X)R_{j,l}(X)}{(a-\kappa - l)!} = 0 \,,
    \]
    where undefined terms (i.e. negative factorials in the denominator) are treated as $0$.
\end{itemize}
\end{lemma}
\begin{proof}
By Corollary \ref{coro:hermite-of-polymixture}, we can write
\[
f(y) = \sum_{j=0}^{\infty} \frac{1}{j!} h_{j,g}(X)y^j = P_1(Xy)e^{\mu_1(X)y + \frac{1}{2}\Sigma_1(X)y^2} + \dots + P_k(Xy)e^{\mu_k(X)y + \frac{1}{2}\Sigma_k(X)y^2} \,,
\]
where $P_1, \dots , P_k$ are polynomials of degree at most $m$.  Now consider the differential operator
\[
\mcl{D} =\mcl{D}_k^{(m+1)2^{k-1}} \mcl{D}_{k-1}^{(m+1)2^{k-2}} \dots \mcl{D}_1^{m+1} \,.
\]

By Claim \ref{claim:degree-reduction}, we know $\mcl{D}(f) = 0$.  However, we can expand out the formula for $\mcl{D}$ using  Claim \ref{claim:diff-operator-expansion} and write it in the following form
\[
\mcl{D} = \partial^{\kappa} + R_{\kappa  - 1}(X, y)\partial^{\kappa - 1} + \dots + R_{1}(X, y)\partial + R_0(X,y) \,.
\]
Note that for each $0 \leq j \leq \kappa$, $R_j(X,y)$ is a polynomial of degree at most $\kappa - j$ in $y$.  Furthermore, by Claim \ref{claim:diff-operator-expansion}, it can be written in the form
\[
R_j(X,y) = \sum_{l=0}^{\kappa - j} R_{j,l}(X)y^l \,,
\]
where each of the polynomials $R_{j,l}$ is homogeneous of degree $\kappa - j + l$ in $X$ and is $(O_{k,m}(1), O_{k,m}(1))$- simple with respect to  $\{\mu_1(X),\Sigma_1(X) \dots , \mu_k(X), \Sigma_k(X) \}$.  Furthermore, it is obvious that $R_{\kappa,0} = 1$

Using the fact that $\mcl{D}(f) = 0$, we get that the polynomials $h_{j,g}$ in the generating function satisfy a recurrence relation of depth $O_{k,m}(1)$.  For any integer $a$, by looking at the coefficient of $y^{a - \kappa}$ in the power series expansion of $\mcl{D}(f)$, we deduce
\[
\sum_{j=0}^{\kappa}\sum_{l = 0}^{\kappa - j} \frac{f_{a - \kappa + j - l}(X)R_{j,l}(X)}{(a-\kappa - l)!} = 0 \,.
\]
This completes the proof.
\end{proof}

For the proof of Theorem \ref{thm:identifiability}, it will be useful to work only with generating functions and only translate back to distribution space at the end.  We have the following equivalent result to the previous lemma.  It can be proven in the exactly the same way.
\begin{lemma}\label{lem:hermite-recurrence-gen-funcs}
Let $N(\mu_1, I + \Sigma_1), \dots , N(\mu_k, I + \Sigma_k)$ be Gaussians.  Let $P_1(X), P_2(X), \dots , P_k(X)$ be polynomials in $X = (X_1, \dots , X_d)$ of degree at most $m$.  Consider the generating function of the polynomial combination i.e.
\[
f(X,y) = P_1(Xy)e^{ \mu_1(X)y+ \frac{1}{2}\Sigma_1(X) y^2} + \dots +  P_k(Xy)e^{ \mu_k(X)y + \frac{1}{2} \Sigma_k (X) y^2} 
\]
and let $f_0(X), f_1(X), \dots $ be the primary terms in the formal power series expansion of $f(y)$  i.e.
\[
f(X,y) = \sum_{j=0}^{\infty} \frac{f_j(X)}{j!}y^j \,.
\]
Let $\kappa = (m+1)(2^k - 1)$.  Then there are $d$-variate polynomials $R_{j,l}(X)$ for $0 \leq j \leq \kappa$ and $0 \leq  l \leq \kappa - j$ such that
\begin{itemize}
    \item $R_{j,l}(X)$ is homogeneous of degree $\kappa - j + l$, $R_{\kappa,0}(X) = 1$
    \item $R_{j,l}(X)$ is $\left(O_{k,m}(1), O_{k,m}(1)\right)$-simple with respect to $\{\mu_1(X), \Sigma_1(X), \dots , \mu_k(X), \Sigma_k(X) \}$
    \item For all integers $a \geq \kappa$,
    \[
    \sum_{j=0}^{\kappa}\sum_{l = 0}^{\kappa - j} \frac{f_{a - \kappa + j - l}(X)R_{j,l}(X)}{(a-\kappa - l)!} = 0 \,,
    \]
    where undefined terms (i.e. negative factorials in the denominator) are treated as $0$.
\end{itemize}
\end{lemma}

\subsection{Observability When Components are All Very Close}

Here, we deal with the special case when all pairs of components are close (within some small constant) in TV distance.  In the next subsection, we will show how to reduce to this case.  

We will first prove observability in an even simpler case where all of the components are within some small constant of isotropic.

\begin{lemma}\label{lem:hermite-to-tv}
Let $N(\mu_1, I + \Sigma_1), \dots , N(\mu_k, I + \Sigma_k)$ be Gaussians.  Let $P_1(X), P_2(X), \dots , P_k(X)$ be polynomials of degree at most $m$.  Let $\eps >0$ be some parameter.  Let $f(X,y)$ be the generating function of the polynomial combination and let $f_0(X), f_1(X), \dots $ be the primary terms in the formal power series expansion of $f(X,y)$.  Then there exist constants $c$, $C$ depending only on $k$ and $m$ with the following property.  If $\norm{\mu_j} \leq c$, $\norm{\Sigma_j}_2 \leq c$ for all $j \in [k]$ and $\norm{v(f_j(X))} \leq \eps$ for all $0 \leq j \leq C$, then
\[
\norm{\chi(f(X,i))}_1 \leq O_{m,k}(\eps)
\]
where $\chi$ denotes the inverse adjusted characteristic function and $i = \sqrt{-1}$.
\end{lemma}
\begin{proof}
By Lemma \ref{lem:hermite-recurrence-gen-funcs}, we can write
\[
\sum_{j=0}^{\kappa}\sum_{l = 0}^{\kappa - j} \frac{f_{a - \kappa + j - l}(X)R_{j,l}(X)}{(a-\kappa - l)!} = 0 \,.
\]
for all $a \geq 2\kappa$ where the polynomials $R_{j,l}$ satisfy the properties in the statement of the lemma.  This rearranges into
\[
f_{a}(X) = -(a - \kappa)!\sum_{j=0}^{\kappa - 1}\sum_{l = 0}^{\kappa - j} \frac{f_{a - \kappa + j - l}(X)R_{j,l}(X)}{(a-\kappa - l)!} \,.
\]
Using the triangle inequality and Claim \ref{claim:polytensor-prod-bound},
\[
\norm{v(f_{a}(X)} \leq \sum_{j=0}^{\kappa - 1}\sum_{l = 0}^{\kappa - j} \frac{(a- \kappa)!}{(a - \kappa - l)!} \norm{v(f_{a - \kappa + j - l}(X))} \norm{v(R_{j,l}(X)} \,.
\]
Let the sequence $G_a$ be defined by $G_a = \norm{v(f_{a}(X)}/\sqrt{a!}$.  The above inequality implies
\[
G_a \leq \sum_{j=0}^{\kappa - 1}\sum_{l = 0}^{\kappa - j} \frac{(a- \kappa)!\sqrt{(a - \kappa + j - l)!}}{(a - \kappa - l)!\sqrt{a!}} \norm{G_{a - \kappa + j - l}} \norm{v(R_{j,l}(X)} \,.
\]
Note that 
\[
\frac{(a- \kappa)!\sqrt{(a - \kappa + j - l)!}}{(a - \kappa - l)!\sqrt{a!}} \leq \frac{(a- \kappa)!\sqrt{(a - 2l)!}}{(a - \kappa - l)!\sqrt{a!}} = O_{m,k}(1) \,.
\]
Also, note that by choosing $c$ (the upper bound on $\norm{\mu_1}, \dots , \norm{\mu_k}, \norm{\Sigma_1}_2, \dots , \norm{\Sigma_k}_2$) sufficiently small, we can ensure that $\norm{v(R_{j,l}(X)}$ is sufficiently small as a function of $k,m$ for all $j < \kappa$.  This is because the polynomials $R_{j,l}$ are homogeneous with positive degree and  $(O_{m,k}(1), O_{m,k}(1))$-simple with respect to $\mu_1(X),\Sigma_1(X), \dots , \mu_k(X), \Sigma_k(X)$.  Thus, as long as $c$ is sufficiently small, we can ensure 
\[
G_a \leq 0.1\max(G_{a-1}, \dots , G_0)
\]
for all $a \geq 2\kappa$.  If we choose the constant $C > 2\kappa$, we may assume $G_0, G_1, \dots G_{\kappa}$ are all $O_{m,k}(\eps)$.  Now we may apply Claim \ref{claim:key-charfunc-bound} to get
\[
\norm{\chi (f(X,i))}_1 \leq \sum_{j = 0}^{\infty} \frac{\sqrt{j!} \norm{v(f_j)}_2}{j!} = \sum_{j = 0}^{\infty} G_j = O_{m,k}(\eps)
\]
which completes the proof.

\end{proof}

Now, by taking a suitable linear transformation, we can generalize the above result to when the components are in regular form and all pairs of components are sufficiently close in TV.  There is some additional work to do because it is not immediately clear how taking a linear transformation affects things in generating function space.  We prove the following result.

\begin{corollary}\label{corollary:hermite-to-tv-close}
Let $N(\mu_1, I + \Sigma_1), \dots , N(\mu_k, I + \Sigma_k)$ be a set of Gaussians in $(\alpha, \beta)$-regular form.  Let $P_1(X), P_2(X), \dots , P_k(X)$ be polynomials of degree at most $m$.  Let $\eps >0$ be some parameter. Let $f(X,y)$ be the generating function of the polynomial combination and let $f_0(X), f_1(X), \dots $ be the primary terms in the formal power series expansion of $f(X,y)$.

Then there exists a (sufficiently large) constant $K$ depending only on $k$ and $m$ with the following property.  If we have
\begin{align*}
    \norm{\mu_j - \mu_{j'}}  \leq \frac{\beta^{-1}}{K} \\
    \norm{\Sigma_j - \Sigma_j'} \leq \frac{\beta^{-1}}{K}
\end{align*}
for all $j,j' \in [k]$ and $\norm{v(f_j(X))} \leq \eps$ for all $0 \leq j \leq K$, then
\[
\norm{\chi (f(X,i))}_1 \leq  (2 + \alpha + \beta)^{K} \eps \,.
\]
\end{corollary}
\begin{proof}
We define $\mcl{F} = \chi \left(f(X,i) \right)$.  Let $\mu = \mu_1, \Sigma = \Sigma_1$.  Let $L: \R^d \rightarrow \R^d$ denote the linear transformation $L(X) =(I + \Sigma)^{1/2}X + \mu$ . Now write
\begin{align*}
\mcl{F}(t) = \frac{1}{(2\pi)^d}\int_{\R^d} f(X,i) e^{-\frac{1}{2}\norm{X}^2 - it \cdot X} dX \,.
\end{align*}
Let $t = L(t')$ for some $t' \in \R^d$.  Then
\begin{align*}
\mcl{F}(t) &= \frac{1}{(2\pi)^d}\int_{\R^d} f(X,i) e^{-\frac{1}{2}\norm{X}^2 - it' \cdot (I + \Sigma)^{1/2} X - i \mu \cdot X} dX \\ &= \frac{1}{(2\pi)^d}\int_{\R^d} \left(\sum_{j=1}^k P_j(iX)e^{i\mu_j(X) - \frac{1}{2}\Sigma_j(X)}\right) e^{-\frac{1}{2}\norm{X}^2 - it' \cdot (I + \Sigma)^{1/2} X - i \mu \cdot X} dX \\
&= \frac{1}{(2\pi)^d}\int_{\R^d} \left(\sum_{j=1}^k P_j(iX)e^{i(\mu_j - \mu) \cdot X -  \frac{1}{2}X^T(\Sigma_j - \Sigma)X }\right) e^{-\frac{1}{2}X^T(I + \Sigma)X - it' \cdot (I + \Sigma)^{1/2} X } dX
\end{align*}

Let
\[
g(X,y) = \left(\sum_{j=1}^k P_j(Xy)e^{(\mu_j - \mu) \cdot Xy +  \frac{1}{2}X^T(\Sigma_j - \Sigma)Xy^2 }\right) \,.
\]
Then we can substitute $X = (I + \Sigma)^{-1/2}Z$ and rewrite 
\begin{align*}
\mcl{F}(t) &= \frac{1}{(2\pi)^d}\int_{\R^d} g(X,i)e^{-\frac{1}{2}X^T(I + \Sigma)X - it' \cdot (I + \Sigma)^{1/2} X } dX \\
&= \frac{\det(I + \Sigma)^{-1/2}}{(2\pi)^d}\int_{\R^d} g((I + \Sigma)^{-1/2}Z, i) e^{-\frac{1}{2}\norm{Z}^2 - it' \cdot Z } dZ \\&= \det(I + \Sigma)^{-1/2}\chi\left( g((I + \Sigma)^{-1/2}Z, i)\right)(t') \,.
\end{align*}
Recall that $t = L(t') = (I + \Sigma)^{1/2}t' + \mu$.  Thus, we conclude 
\begin{equation}\label{eq:equivalence}
\norm{\mcl{F}}_1 = \norm{\chi\left( g((I + \Sigma)^{-1/2}Z, i)\right)}_1 \,.
\end{equation}
We will now bound the RHS by applying Lemma \ref{lem:hermite-to-tv}.
\\\\
Let $C$ be the parameter from Lemma \ref{lem:hermite-to-tv} (set in terms of $k,m$).  Choosing $K$ sufficiently large, we may assume that 
\begin{equation}\label{eq:normbound1}
\norm{v(f_j(X))} \leq \eps
\end{equation}
for all $j \leq C$. 
Note that 
\[
g(X,y) = f(X,y)e^{-\mu (X)y-\frac{1}{2}\Sigma(X)y^2} \,.
\]
Let the primary terms of $g$ be $g_0(X), g_1(X), \dots , $.  We can also write the power series expansion of
\[
e^{-\mu (X)y-\frac{1}{2}\Sigma(X)y^2} = \sum_{j=0}^{\infty} \frac{\left( -\mu (X)y-\frac{1}{2}\Sigma(X)y^2\right)^j}{j!} = \sum_{j=0}^{\infty}\frac{u_j(X)y^j}{j!}
\]
for some polynomials $u_j$.  Note that by Claim \ref{claim:polytensor-prod-bound} and the fact that $\norm{\mu},\norm{\Sigma}_2 \leq \alpha$, we have 
\begin{equation}\label{eq:normbound2}
\norm{v(u_j(X))} \leq (2 + \alpha)^{O_{m,k}(1)}
\end{equation}
for $j \leq C$ (since $C = O_{m,k}(1)$). Now we have
\[
\sum_{j=0}^{\infty}\frac{g_j(X)y^j}{j!} = \left(\sum_{j=0}^{\infty}\frac{f_j(X)y^j}{j!}\right)\left(\sum_{j=0}^{\infty}\frac{u_j(X)y^j}{j!}\right) \,.
\]
Thus, after expanding and truncating to the first $C+1$ terms, we can use Claim \ref{claim:polytensor-prod-bound} and  equations (\ref{eq:normbound1}) and (\ref{eq:normbound2}) to deduce that for all $j$ with $0 \leq j \leq C$,
\[
\norm{v(g_j(X))} \leq (2 + \alpha)^{O_{m,k}(1)}\eps \,.
\]
Now note that the primary terms of the function $g((I + \Sigma)^{-1/2}X, y)$ are exactly $g_j((I + \Sigma)^{-1/2}X)$.  By Claim \ref{claim:polytovec-linear-transform}  we get for all $j$ with $0 \leq j \leq C$,
\[
 \norm{v\left(g_j\left((I + \Sigma)^{-1/2}X\right)\right)} \leq (2 + \alpha + \beta)^{O_{m,k}(1)}\eps \,.
\]
Now we need to check the remaining conditions of Lemma \ref{lem:hermite-to-tv}.  Let 
\[
G_j' = N(\mu_j', I + \Sigma_j') = N\left((I + \Sigma)^{-1/2}(\mu_j - \mu), I + (I + \Sigma)^{-1/2}( \Sigma_j - \Sigma)(I + \Sigma)^{-1/2}\right)
\]
for all $j \in [k]$.  Note that we can write 
\[
g((I + \Sigma)^{-1/2}X, y) = \sum_{j \in [k]} P_j'(Xy)e^{\mu_j'(X)y + \frac{1}{2}\Sigma_j'(X)y^2}
\]
for polynomials $P_1', \dots , P_k'$ of degree at most $m$.  By choosing $K$ sufficiently large, we can ensure that the Gaussians $G_1', \dots , G_k'$ satisfy the conditions of Lemma \ref{lem:hermite-to-tv} (because $\norm{\mu_j - \mu}, \norm{\Sigma_j - \Sigma}_2$ will be sufficiently small). Thus, applying Lemma \ref{lem:hermite-to-tv} and using (\ref{eq:equivalence}) completes the proof.
\end{proof}

\subsection{Reducing to When Components are All Very Close}\label{sec:reduce-to-very-close}
We will now deal with the case when the components are not necessarily all very close to each other.  We will still assume that the components are in $(\alpha,\beta)$-regular form for $\alpha, \beta \leq \poly(\log 1/\eps)$.  

The way we will reduce to the case where the components are all very close is as follows.  We show that we can partition the components $G_1, \dots , G_k$ into submixtures say $S_1, \dots , S_a \subset [k]$ such that
\begin{itemize}
    \item Components in the same submixture are sufficiently close to apply Corollary \ref{corollary:hermite-to-tv-close}
    
    \item Components in different submixtures are not too close
\end{itemize}
We then use differential operators to isolate each of these submixtures (relying on the second condition above) and deduce that if the Hermite moment polynomials of the entire mixture are close to $0$, then the Hermite moment polynomials of each submixture are close to $0$.  We can then apply Corollary \ref{corollary:hermite-to-tv-close} on each submixture to complete the proof.

We will need a few additional definitions.

\begin{definition}
Given Gaussians $G_1 = N(\mu_1, I + \Sigma_1), \dots , G_k = N(\mu_k, I  + \Sigma_k)$, we say a partition $S_1, \dots , S_a$ of $[k]$ is $\delta$-separated if for any pair of components in different parts of the partition, say $i_1 \in S_{j_1}, i_2 \in S_{j_2}, j_1 \neq j_2$,  
\[
\norm{\mu_{i_1} - \mu_{i_2}}_2 + \norm{\Sigma_{i_1} - \Sigma_{i_2}}_ 2 \geq \delta \,.
\]
\end{definition}
\begin{definition}
Given Gaussians $G_1 = N(\mu_1, I + \Sigma_1), \dots , G_k = N(\mu_k, I  + \Sigma_k)$, we say a partition $S_1, \dots , S_a$ of $[k]$ is $(\delta_1,\delta_2)$-good for some parameters $\delta_1, \delta_2$ if for any pair of components in different parts of the partition, say $i_1 \in S_{j_1}, i_2 \in S_{j_2}, j_1 \neq j_2$,  
\[
\norm{\mu_{i_1} - \mu_{i_2}}_2 + \norm{\Sigma_{i_1} - \Sigma_{i_2}}_ 2 \geq \delta_1 \,.
\]
and for any pair of components in the same part of the partition, say $i_1, i_2 \in S_{j_1}$,
\[
\norm{\mu_{i_1} - \mu_{i_2}}_2 + \norm{\Sigma_{i_1} - \Sigma_{i_2}}_ 2 \leq \delta_2 \,.
\]
\end{definition}
\begin{claim}\label{claim:good-partition}
Given Gaussians $G_1 = N(\mu_1, I + \Sigma_1), \dots , G_k = N(\mu_k, I  + \Sigma_k)$ and any parameter $\theta > 0$, there exists a $(\theta, k\theta)$-good partition of $[k]$.
\end{claim}
\begin{proof}
Consider a graph on nodes $1,2, \dots , k$ where two nodes $i, j$ are connected if and only if
\[
\norm{\mu_{i} - \mu_{j}}_2 + \norm{\Sigma_{i} - \Sigma_{j}}_ 2 \leq \theta \,.
\]
Now let $S_1, \dots, S_a$ be the connected components in this graph.  We claim this forms a $(\theta, k\theta)$-good partition.  It is clear that any pair $i, j$ in different parts of the partition satisfies 
\[
\norm{\mu_{i} - \mu_{j}}_2 + \norm{\Sigma_{i} - \Sigma_{j}}_2 \geq \theta \,.
\]
On the other hand, for any pair in the same part of the partition, we can find a path between them in the graph, say $i, l_1, \dots , l_c, j$ and use the triangle inequality summed along the path to deduce
\[
\norm{\mu_{i} - \mu_{j}}_2 + \norm{\Sigma_{i} - \Sigma_{j}}_2 \leq k\theta 
\]
which completes the proof.
\end{proof}

Now we begin the main technical part.  We first explain the intuition for using differential operators to isolate parts of the mixture.  Let $\mcl{D} = \partial - \mu(X) - \Sigma(X)y$ for some $\mu, \Sigma$.  Consider a Gaussian $N(\mu', I + \Sigma')$ and consider
\[
\mcl{D}\left(e^{\mu'(X) y+ \frac{1}{2}\Sigma'(X)y^2}\right) \,.
\]
Note that if $\mu, \Sigma$ are close to $\mu', \Sigma'$, then the result will be essentially $0$.  On the other hand, we can verify that if $\mu,\Sigma$ are not close to $\mu', \Sigma'$, then the result will be bounded away from $0$.  Thus, we can apply differential operators to essentially remove all components that are far away from $N(\mu', I + \Sigma')$, leaving only a submixture of components that are sufficiently close to each other.  Over the next two claims, we will formalize this intuition by showing that given a submixture of close components, if we repeatedly apply far-away differential operators, we cannot accidentally zero-out the submixture.

\begin{claim}\label{claim:operator-does-not-zero}
Consider a set of Gaussians $G_1 = N(\mu_1, I + \Sigma_1), \dots , G_k = N(\mu_k, I + \Sigma_k)$.  Let $P_1(X), \dots , P_k(X)$ be polynomials of degree at most $m$.  Let $f(X,y)$ be the generating function of the polynomial combination and let $f_0, f_1, \dots $ be the primary terms in the formal power series expansion of $f$.
\\\\
Consider two parameters $\mu, \Sigma$ and assume that for all $j \in [k]$,
\[
\delta \leq \norm{\mu_j - \mu}_2 + \norm{\Sigma_j - \Sigma}_2 \leq \delta^{-1}  \,.
\]
Let $g(X,y) = (\partial -( \mu(X) + \Sigma(X)y))f(X,y)$ where the partial derivative is taken with respect to $y$ and let $g_0, g_1, \dots $ be the primary terms of $g$.  Let $K$ be a constant that is sufficiently large in terms of $k,m$.  For any parameter $\eps > 0$, if for all $j \leq K$
\[
\norm{v(g_j(X))}_2 \leq \eps
\]
then for all $j \leq K$,
\[
\norm{v(f_j(X))}_2 \leq (2 +\norm{\mu} + \norm{\Sigma}_2 + \delta^{-1})^{O_{k,m,K}(1)} \eps \,.
\]
\end{claim}
\begin{proof}
Consider 
\[
h(X,y) = g(X,y)e^{-\mu(X)y - \frac{1}{2}\Sigma(X)y^2} \,.
\]
Note that 
\[
h(X,y) = \partial_y\left( f(X,y)e^{-\mu(X)y - \frac{1}{2}\Sigma(X)y^2} \right) \,.
\]
Let $F(X,y) = f(X,y)e^{-\mu(X)y - \frac{1}{2}\Sigma(X)y^2}$ and let $F_0,F_1, \dots $ be its primary terms.  The primary terms of $h$ are $F_1, F_2, \dots $ (i.e. the same but shifted down by one).
\\\\
First note that since $h(X,y) = g(X,y)e^{-\mu(X)y - \frac{1}{2}\Sigma(X)y^2}$, we have for all $1 \leq j \leq K + 1$
\begin{equation}\label{eq:normbound3}
\norm{v(F_j(X))}_2 \leq ( 2 + \norm{\mu} + \norm{\Sigma}_2)^{O_K(1)} \eps \,.
\end{equation}
To see this, it suffices to consider the product of the power series for $g(X,y)$ and $e^{-\mu(X)y - \frac{1}{2}\Sigma(X)y^2}$ after truncating both series at $y^K$ (dropping all terms with higher powers of $y$).  By Claim \ref{claim:polytensor-prod-bound}, the first $K +1$ primary terms in the power series expansion of $e^{-\mu(X)y - \frac{1}{2}\Sigma(X)y^2}$ have coefficient norm at most $( 2 + \norm{\mu} + \norm{\Sigma}_2)^{O_K(1)}$.
Thus, when we expand out the product of $g(X,y)$ and $e^{-\mu(X)y - \frac{1}{2}\Sigma(X)y^2}$, we use Claim \ref{claim:polytensor-prod-bound} again to deduce that the resulting primary terms will all have coefficient norm at most $( 2 + \norm{\mu} + \norm{\Sigma}_2)^{O_K(1)} \eps$.
\\\\
Now we will argue about $F_0(X)$ (which is just a constant).  Define the following differential operators for $j \in [k]$,
\[
\Delta_j = \partial - ((\mu_j(X) - \mu(X)) + (\Sigma_j(X) - \Sigma(X))y) \,.
\]
Now consider the differential operator 
\[
\Delta =\Delta_k^{(m+1)2^{k-1}} \Delta_{k-1}^{(m+1)2^{k-2}} \dots \Delta_1^{m+1} \,.
\]
We know by Claim \ref{claim:degree-reduction} that $\Delta (F(X,y) ) = \Delta (f(X,y) e^{-\mu(X)y - \frac{1}{2}\Sigma(X)y^2} ) = 0$.  On the other hand, we may use Claim \ref{claim:diff-operator-expansion} to expand $\Delta$ in the form
\[
\Delta = \partial^{(m+1)(2^{k} - 1)} + R_{(m+1)(2^{k} - 1) - 1}(X, y)\partial^{(m+1)(2^{k} - 1) - 1} + \dots + R_{1}(X, y)\partial + R_0(X,y) \,.
\]
for some polynomials $R_0, R_1, \dots $.  Let $\kappa = (m+1)(2^{k} - 1)$.  We have
\[
R_0(X,y) F(X,y) = - \left( \sum_{j=1}^{\kappa}R_j(X,y)\partial^{\kappa- j}(F(X,y)) \right)
\]
which is equivalent to
\[
R_0(X,y) \sum_{l=0}^{\infty} \frac{F_l(X)y^l}{l!} = - \left( \sum_{j=1}^{\kappa}R_j(X,y)\left(\sum_{l=0}^{\infty} \frac{F_{l + j}(X)y^l}{l!}\right) \right) \,.
\]
The key observation is that $F_0(X)$ appears on the LHS but not the RHS i.e. we may write
\begin{equation}\label{eq:isolate-term}
R_0(X,y) F_0(X) = - \left( \sum_{j=1}^{\kappa}R_j(X,y)\left(\sum_{l=0}^{\infty} \frac{F_{l + j}(X)y^l}{l!}\right) \right) - R_0(X,y)\sum_{l=1}^{\infty} \frac{F_l(X)y^l}{l!} \,.
\end{equation}

Now by Claim \ref{claim:prod-linear-forms}, we have
\[
\norm{v_y(R_0(X,y))} \geq (0.1\delta)^{O_{k,m}(1)} \,.
\]
Note that Claim \ref{claim:diff-operator-expansion} and Claim \ref{claim:polytensor-prod-bound} give an upper bound on the coefficient norm of $R_j(X,y)$ for all $0 \leq j \leq \kappa$.  In particular, we get
\[
\norm{v_y(R_j(X,y))} \leq (2\delta^{-1})^{O_{k,m}(1)} \,.
\]
Now, by combining with (\ref{eq:normbound3}) we can upper bound the coefficient norm of the first $\kappa$ terms in the power series of the RHS of (\ref{eq:isolate-term}).  Since $R_0(X,y)$ has degree at most $\kappa$,  as long as $K > 10\kappa$, we get
\[
\norm{v(F_0(X))} \leq (2 + \norm{\mu} + \norm{\Sigma}_2 + \delta^{-1})^{O_{K,k,m}(1)}\eps \,.
\]

Now we have an upper bound on the coefficient norm of all of $F_0(X), \dots , F_K(X)$.  Note that 
\[
f(X,y) = F(X,y)e^{\mu(X)y + \frac{1}{2}\Sigma(X)y^2}
\]
and we can expand out the product of the power series of $F(X,y)$ and $e^{\mu(X)y + \frac{1}{2}\Sigma(X)y^2}$, using the same argument as before, to get that for all $j$ with $0 \leq j \leq K$,
\[
\norm{v(f_j(X))} \leq (2 + \norm{\mu} + \norm{\Sigma}_2 + \delta^{-1})^{O_{K,k,m}(1)}\eps \,.
\]
\end{proof}

By repeatedly applying the previous claim, we get the following.

\begin{claim}\label{claim:cluster-hermite-polys}
Consider a set of Gaussians $G_1 = N(\mu_1, I + \Sigma_1), \dots , G_k = N(\mu_k, I + \Sigma_k)$ in $(\alpha,\beta)$-regular form and let $P_1(X), \dots , P_k(X)$ be polynomials of degree at most $m$.  Let $f(X,y)$ be the generating function of the polynomial combination and let $f_0, f_1, \dots $ be the primary terms in the formal power series expansion of $f$.

Let $S_1, \dots , S_a$  be a $\delta$-separated partition of $[k]$ for some constant $\delta < 1$.  For each $l \in [a]$, let $f^{(l)}(X,y)$ be the generating function of the polynomial combination of only the Gaussians in $S_l$ and let $f_0^{(l)}, f_1^{(l)}, \dots $ be the primary terms in the formal power series expansion of $f^{(l)}$.

For any constant $K$ that is sufficiently large in terms of $k,m$, There exists a constant $C_{k,K,m}$ depending on $k,K,m$ such that the following holds.  If for all $j \leq K$ we have
\[
\norm{v(f_j(X))} \leq \eps
\]
then for all $l \in [a], j \leq K - (m+1)2^k$,
\[
\norm{v(f^{(l)}_j(X))} \leq \eps \left(2 + \alpha + \delta^{-1}\right)^{C_{k,K,m}} \,.
\]
\end{claim}
\begin{proof}
Without loss of generality $S_1 =\{1,2, \dots , t \}$ for some $t \leq k$.  Recall that for each $j \in [k]$, we use $\mcl{D}_j$ to denote the differential operator $\partial - (\mu_j(X) + \Sigma_j(X)y)$.  Now consider the differential operator
\[
\mcl{D}^{(1)} =  \mcl{D}_k^{2^{k- t - 1}(m+1)} \cdots  \mcl{D}_{t+2}^{2(m+1)}\mcl{D}_{t+1}^{m+1} \,.
\]
Note that by Claim \ref{claim:diff-operator-expansion} and Claim \ref{claim:polytensor-prod-bound}, we can rewrite
\[
\mcl{D}^{(1)} = \partial^{\kappa} + R_{\kappa - 1}(X,y)\partial^{\kappa - 1} + \dots  +R_0(X,y) 
\]
where $\kappa = (m+1)(2^{k-t} - 1)$ and for all $j$,
\[
\norm{v_y(R_j(X,y))} \leq \left(2 + \alpha \right)^{O_{k,m}(1)} \,.
\]
Now consider the power series expansion of $\mcl{D}^{(1)} f$ and let $d_0(X), d_1(X), \dots$ be its primary terms.  Since the differential operator $\mcl{D}^{(1)}$ has degree at most $(m+1)2^k$, the assumption in the statement of the claim implies that for all $j \leq K - (m+1)2^k$,
\[
\norm{v(d_j(X))} \leq \left(2 + \alpha \right)^{O_{k,m, K}(1)}\eps \,.
\]
On the other hand, note that 
\[
\mcl{D}^{(1)} f = \mcl{D}^{(1)} f^{(1)}
\]
since by Claim \ref{claim:degree-reduction}, the operator $\mcl{D}^{(1)}$ zeros out all of the other components.  We may now repeatedly apply Claim \ref{claim:operator-does-not-zero} (since $\mcl{D}^{(1)}$ factors as a composition of linear differential operators) and use the fact that the partition $S_1, \dots , S_a$ is $\delta$-separated to deduce that as long as $K$ is sufficiently large in terms of $k,m$, we have for all $j \leq K - (m+1)2^k$ ,
\[
\norm{v(f^{(1)}_j(X))} \leq \left(2 + \alpha + \delta^{-1}\right)^{O_{k,m, K}(1)} \eps \,.
\]
Since the initial choice of $l = 1$ was arbitrary, this completes the proof.
\end{proof}

We can now prove the main theorem of this section.  The statement below is stated in terms of generating functions.  We will translate it into distribution space and prove Theorem \ref{thm:identifiability} immediately afterwards.

\begin{theorem}\label{thm:hermite-to-tv-full}
Consider a set of Gaussians $G_1 = N(\mu_1, I + \Sigma_1), \dots , G_k = N(\mu_k, I + \Sigma_k)$ in $(\alpha,\beta)$-regular form and let $P_1(X), \dots , P_k(X)$ be polynomials of degree at most $m$.  Let $f(X,y)$ be the generating function of the polynomial combination and let $f_0, f_1, \dots $ be the primary terms in the formal power series expansion of $f$.  

There exists a constant $C$ depending only on $k,m$ such that the following holds.  If for all $j$ with $0 \leq j \leq C$, we have
\[
\norm{v(f_j(X))} \leq \eps
\]
then we must have
\[
\norm{\chi(f(X,i))}_1 \leq (2 + \alpha + \beta)^C \eps \,.
\]
\end{theorem}
\begin{proof}
Let $K$ be the constant in Corollary \ref{corollary:hermite-to-tv-close}.  Recall that $K$ is set as a function of $k,m$.  Let $\theta = \frac{\beta^{-1}}{2kK}$. By Claim \ref{claim:good-partition}, there is a  $(\theta, k\theta)$-good partition of $[k]$, say $S_1, \dots , S_a$.  For each $l \in [a]$, let 
\[
f^{(l)}(y) = \sum_{j \in S_l}P_j(Xy)e^{\mu_j(X)y + \frac{1}{2}\Sigma_j(X)y^2} \,.
\]
and let $f_0^{(l)},f_1^{(l)}, \dots  $ be the primary terms in the expansion of $f^{(l)}$ as a power series in $y$.

By Claim \ref{claim:cluster-hermite-polys}, as long as $C$ is sufficiently large in terms of $k,m$, we have for all $0 \leq j \leq K$
\[
\norm{v(f_j^{(l)}(X))} \leq \eps (2 + \alpha + \beta)^{O_{m,k}(1)} \,.
\]

Now by Corollary \ref{corollary:hermite-to-tv-close}, since the partition is such that all pairs of components in the same part are close, we get that 
\[
\norm{\chi\left( f^{(l)}(X,i)\right) }_1 \leq \eps (2 + \alpha + \beta)^{O_{m,k}(1)} \,.
\]
Finally, since 
\[
f = f^{(1)} + \dots + f^{(a)}
\]
we get 
\[
\norm{\chi\left(f(X,i)\right)}_1 \leq (2 + \alpha + \beta)^{O_{m,k}(1)} \eps
\]
which completes the proof.
\end{proof}

To prove Theorem \ref{thm:identifiability}, it suffices to translate the above theorem from generating functions back to distributions.
\begin{proof}[Proof of Theorem \ref{thm:identifiability}]
By Corollary \ref{coro:hermite-of-polymixture}, there are polynomials $P_1, \dots , P_k$ of degree at most $m$ such that 
\[
f(X,y) = \sum_{j=0}^{\infty} \frac{1}{j!} h_{j,g}(X)y^j = P_1(Xy)e^{\mu_1(X)y + \frac{1}{2}\Sigma_1(X)y^2} + \dots + P_k(Xy)e^{\mu_k(X)y + \frac{1}{2}\Sigma_k(X)y^2} \,.
\]
Now, applying Theorem \ref{thm:hermite-to-tv-full} to the expression on the RHS, we have
\[
\norm{\chi(f(X,i))}_1 \leq (2 + \alpha + \beta)^{O_{m,k}(1)} \eps \,.
\]
Finally, Fact \ref{fact:inversion} and Claim \ref{claim:invchar-formula} imply that $g = \chi(f(X,i)) $ so we are done.
\end{proof}

\section{Hermite Moment Polynomials of a Single Gaussian: Tail Bounds and other Properties}\label{sec:tailbounds}

Note that the Hermite moment polynomials of a distribution are given by $\E_z[H_m(X,z)]$ for $z$ drawn from that distribution.  The way we estimate these Hermite moment polynomials from samples will be by robustly estimating the mean of the distribution $H_m(X,z)$.  In this section, our goal is to understand properties of the distribution of $H_m(X,z)$ for $z$ drawn from a single Gaussian of the form $N(\mu, I + \Sigma)$.  Later, in Section \ref{sec:est-hermite}, we will use this to deduce that we can estimate $H_m(X,z)$ to within nearly optimal error even when $z$ is drawn from a regular-form mixture of Gaussians.


We will need a few definitions.
\begin{definition}
Given two sets of $d$ variables $X^{(1)} = (X^{(1)}_1, \dots , X^{(1)}_d) $ and $X^{(2)} = (X^{(2)}_1, \dots , X^{(2)}_d) $ and a polynomial $P(X^{(1)}, X^{(2)})$, for integers $m_1, m_2$, the degree-$(m_1,m_2)$-part of $P$ consists of the monomials of $P$ that have total degree $m_1$ in $X^{(1)}$ and total degree $m_2$ in $X^{(2)}$.
\end{definition}

\begin{definition}
Consider two sets of formal variables, say $X^{(1)} = (X_1^{(1)}, \dots ,X_d^{(1)}), X^{(2)} = (X_1^{(2)}, \dots ,X_d^{(2)}) $.  Let $G = N(\mu, I + \Sigma)$ be a Gaussian.  We say the fundamental polynomials (with respect to $G$) are
\[
\Big\{\mu\left(X^{(1)}\right),\mu \left( X^{(2)} \right), \Sigma \left(X^{(1)}\right), \Sigma \left(X^{(2)}\right), \left(X^{(1)}\right)^T (I + \Sigma)X^{(2)} \Big\}
\]

\end{definition}

\subsection{Covariance of Multivariate Hermite Polynomials}\label{sec:hermite-cov-basic}
First, we analyze the covariance of $H_m(X,z)$ for $z$ drawn from a Gaussian $N(\mu, I + \Sigma)$.

\begin{claim}\label{claim:hermite-covariance-basic}
Let $G = N(\mu, I + \Sigma)$ be a Gaussian.  Let $m$ be a positive integer.  Consider two sets of $d$ variables $X^{(1)} = (X^{(1)}_1, \dots , X^{(1)}_d) $ and $X^{(2)} = (X^{(2)}_1, \dots , X^{(2)}_d) $.  Then 
 the expression 
 \[
 \E_{z \sim G}[ H_m(X^{(1)},z) H_m(X^{(2)},z)]
 \]
 (which is a formal polynomial in $2d$ variables) can be computed as follows:
 \begin{enumerate}
     \item Consider the power series expansion of
     \[
     \exp \left(  y \mu(X^{(1)} + X^{(2)}) + \frac{1}{2}y^2\Sigma(X^{(1)} + X^{(2)}) + y^2   X^{(1)} \cdot X^{(2)} \right)  = \sum_{j=0}^{\infty} \frac{Q_j(X^{(1)}, X^{(2)})y^j}{j!}
     \]
     \item Take $\binom{2m}{m}^{-1}$ times the degree $(m,m)$ part of $Q_{2m}(X^{(1)}, X^{(2)})$ 
 \end{enumerate}
\end{claim}
\begin{proof}
We will evaluate $\E_{z \sim G}[ H_m(X^{(1)},z)  H_m(X^{(2)},z)]$ using generating functions.  Let 
\[
F = \left( \sum_{m = 0}^{\infty} \frac{1}{m!}H_m(X^{(1)}, z)y^m\right)\left( \sum_{m = 0}^{\infty} \frac{1}{m!}H_m(X^{(2)}, z)y^m\right) \,.
\]
Note that $H_m(X,z)$ is homogeneous and of degree $m$ in $X$.  Thus, to compute $H_m(X^{(1)},z)  H_m(X^{(2)},z)$, it suffices to extract the degree-$(m,m)$ part of the coefficient of  $y^{2m}$ in the power series expansion of $F$.

Now by Claim \ref{claim:hermite-identity} we may write
\begin{align*}
\E_{z \sim G}[F] &= \E_{z \sim G} \left[ \exp \left( yz \cdot (X^{(1)} + X^{(2)}) - \frac{1}{2}y^2 (X^{(1)} \cdot X^{(1)} + X^{(2)} \cdot X^{(2)}) \right)\right]\\ &=  C \int \exp \left(-\frac{1}{2}(z - \mu)^T(I + \Sigma)^{-1}(z - \mu) + yz \cdot (X^{(1)} + X^{(2)}) - \frac{1}{2}y^2 (X^{(1)} \cdot X^{(1)} + X^{(2)} \cdot X^{(2)}) \right)\\ &=  C \int \exp \Bigg(-\frac{1}{2}\left(z - \mu - y(I + \Sigma)(X^{(1)} + X^{(2)})\right)^T(I + \Sigma)^{-1}\left(z - \mu - y(I + \Sigma)(X^{(1)} + X^{(2)})\right) \\  & \quad + y\mu \cdot (X^{(1)} + X^{(2)}) +  \frac{1}{2}y^2 \left(\left(X^{(1)}\right)^T \Sigma X^{(1)} + \left(X^{(2)}\right)^T \Sigma X^{(2)}  + 2\left(X^{(1)}\right)^T (I + \Sigma)X^{(2)}\right) \Bigg) \\ &= \exp \left(  y \mu(X^{(1)} + X^{(2)}) + \frac{1}{2}y^2\Sigma(X^{(1)} + X^{(2)}) + y^2   X^{(1)} \cdot X^{(2)} \right)  \,. 
\end{align*}
In the above, $C$ denotes the normalization constant for a Gaussian with covariance $I  + \Sigma$.  Next, expanding the above as a power series, we know that the degree $(m,m)$ part of $Q_{2m}(X^{(1)}, X^{(2)})$ is equal to 
\[
\E \left[\frac{(2m)!}{m! \cdot m!}H_m(X^{(1)}, z)H_m(X^{(2)}, z) \right]
\]
from which we immediately get the desired conclusion.
\end{proof}

As a corollary to the above, we get the following upper bound on the covariance of $v(H_m(X,z))$ for a single Gaussian.  To do this, we rely on the symmetric tensorization (recall definition \ref{def:symmetrictensorization}) and its properties to relate the expression $\E_{z \sim G}[ H_m(X^{(1)},z) H_m(X^{(2)},z)]$ to the covariance of the vector $v(H_m(X,z))$.
\begin{claim}\label{claim:single-gaussian-hermite-covariance}
Let $G = N(\mu, I + \Sigma)$ be a Gaussian.  Let $m$ be a positive integer.  Let $\Sigma_{H_m}$ be the covariance of $ v(H_m(X,z))$ for $z$ drawn from $G$.  Then 
\[
 \Sigma_{H_m} \leq E_{z \sim G}\left[ v(H_m(X,z)) \otimes v(H_m(X,z)) \right] \leq (m(1 + \norm{\mu}_2 + \norm{\Sigma}_2))^{O(m)}I 
\]
where $I$ on the RHS denotes the identity matrix of the appropriate size.
\end{claim}
\begin{proof}
Consider two sets of $d$ variables $X^{(1)} = (X^{(1)}_1, \dots , X^{(1)}_d) $ and $X^{(2)} = (X^{(2)}_1, \dots , X^{(2)}_d) $.  Let 
\[
P(X^{(1)}, X^{(2)}) = \E_{z \sim G}[ H_m(X^{(1)},z)  H_m(X^{(2)},z)] \,.
\]
Let $T = T_{\sym}(P)$.  Note that by Claim \ref{claim:polyproducttosymmetrictensor},
\[
T = \E_{z \sim G}[v(H_m(X^{(1)},z)) \otimes v(H_m(X^{(2)},z)) ] \,.
\]
On the other hand, by Claim \ref{claim:hermite-covariance-basic}, $P(X^{(1)}, X^{(2)})$ can be computed by considering the power series expansion
\begin{align*}
\exp \left(  y \mu(X^{(1)} + X^{(2)}) + \frac{1}{2}y^2\Sigma(X^{(1)} + X^{(2)}) + y^2   X^{(1)} \cdot X^{(2)} \right) =  \sum_{j=0}^{\infty} \frac{Q_j(X^{(1)}, X^{(2)})y^j}{j!} \,. 
\end{align*}
and taking $\binom{2m}{m}^{-1}$ times the degree $(m,m)$ part of $Q_{2m}$.  We will use $Q_{m,m}$ to denote the degree $(m,m)$ part of $Q_{2m}$.  Write 
\begin{align*}
Q &=  y \mu(X^{(1)} + X^{(2)}) + \frac{1}{2}y^2\Sigma(X^{(1)} + X^{(2)}) + y^2   X^{(1)} \cdot X^{(2)} \\&= y\mu^T  (X^{(1)} + X^{(2)}) + y^2 (X^{(1)})^T(I + \Sigma)X^{(2)} + \frac{1}{2}y^2 \left((X^{(1)})^T\Sigma X^{(1)} + (X^{(2)})^T\Sigma X^{(2)} \right) \,.
\end{align*}
Now we may write
\begin{align*}
\exp(Q) = 1 + Q + \frac{Q^2}{2!} + \dots      \,.
\end{align*}
Let $\mathcal{S}$ be the set of fundamental polynomials of $G = N(\mu, I + \Sigma)$ .  Note that $Q$ is a sum of $O(1)$ polynomials from among $\mathcal{S}$.  Furthermore, each of these polynomials is homogeneous in each of the sets of variables $X^{(1)},X^{(2)}$.  Thus, we can expand each of the terms $Q, Q^2, \dots $ as a sum of products of elements of $\mathcal{S}$.  Now by Claim \ref{claim:hermite-covariance-basic}, we can obtain $P(X^{(1)}, X^{(2)})$ by discarding all of the products that do not have the proper degrees (degree exactly $m$ in each of the subsets of variables $X^{(1)}, X^{(2)}$).  Since $Q$ is a sum of $O(1)$ polynomials from $\mathcal{S}$, each with degree $1$ or $2$, and we only keep terms with total degree $2m$, we deduce that the number of terms (and all of the coefficients in front of the terms) that we keep are $m^{O(m)}$.  Thus, $P(X^{(1)}, X^{(2)})$ is $(m^{O(m)}, 2m)$-simple with respect to $\mcl{S}$.

It now suffices to bound the operator norms of the symmetric tensorizations of each of the individual products of fundamental polynomials.  Consider a product of fundamental polynomials $ P_1  P_2 \cdots P_a $ for some $a \leq 2m$.  By Claim \ref{claim:polytensorinjnormbound}, 
\begin{align*}
\norm{T_{\sym}(P_1 P_2 \cdots P_a )}_{\op} \leq \norm{T_{\sym}(P_1)} \cdots \norm{T_{\sym}(P_a)} \leq  \left((\norm{\mu}_2 + 1)(\norm{\Sigma}_2 + 1)(\norm{I + \Sigma}_{\textsf{op}} + 1)\right)^{O(m)} \\ \leq (1 + \norm{\mu}_2 + \norm{\Sigma}_2)^{O(m)} \,.
\end{align*}
Combining this inequality with the fact that $P(X^{(1)}, X^{(2)})$ is $\left(m^{O(m)}, 2m\right)$-simple with respect to $\mcl{S}$ gives that 
\[
 T \leq  (m(1 + \norm{\mu}_2 + \norm{\Sigma}_2))^{O(m)}I \,.
\]
Also, clearly $\Sigma_{H_m} \leq T$ so we are done.
\end{proof}

We will also need a lower bound on the covariance of the Hermite polynomials.  This lower bound will hold when $G = N(\mu, I + \Sigma)$ is within a sufficiently small constant of isotropic.

\begin{claim}\label{claim:nearlyisotropicgaussian}
Let $G = N(\mu, I + \Sigma)$ be a Gaussian.  Let $m$ be a positive integer.  There exists a constant $c_m$ depending only on $m$ such that if $\norm{\mu}, \norm{\Sigma}_2 \leq c_m$,  then $\Sigma_{H_m}$, the covariance of $ v(H_m(X,z))$ for $z$ drawn from $G$, satisfies
\[
\Sigma_{H_m} \geq \frac{1}{2}I \,.
\]
\end{claim}
\begin{proof}
Consider two sets of $d$ variables $X^{(1)} = (X^{(1)}_1, \dots , X^{(1)}_d) $ and $X^{(2)} = (X^{(2)}_1, \dots , X^{(2)}_d) $.  Let 
\[
P(X^{(1)}, X^{(2)}) = \E_{z \sim G}[ H_m(X^{(1)},z)  H_m(X^{(2)},z)] \,.
\]
Let $T = T_{\sym}(P)$.  Note that 
\[
T = \E_{z \sim G}[v(H_m(X^{(1)},z)) \otimes v(H_m(X^{(2)},z)) ] \,.
\]
On the other hand, recall that by Claim \ref{claim:hermite-covariance-basic}, $P(X^{(1)}, X^{(2)})$ can be computed by considering the power series expansion
\begin{align*}
\exp \left(  y \mu(X^{(1)} + X^{(2)}) + \frac{1}{2}y^2\Sigma(X^{(1)} + X^{(2)}) + y^2   X^{(1)} \cdot X^{(2)} \right) =  \sum_{j=0}^{\infty} \frac{Q_j(X^{(1)}, X^{(2)})y^j}{j!} \,. 
\end{align*}
and taking $\binom{2m}{m}^{-1}$ times the degree $(m,m)$ part of $Q_{2m}$.  We will use $Q_{m,m}$ to denote the degree $(m,m)$ part of $Q_{2m}$.

The key observation now is that if $\mu = 0, \Sigma = 0$, then the LHS of the above is just $ \exp(y^2   X^{(1)} \cdot X^{(2)} )$.  In this case, $Q_{m,m}$ would be $ (2m)!/m! \cdot (X^{(1)} \cdot X^{(2)})^{m}$.  

We now show that by choosing $c_m$ sufficiently small, we can ensure that $Q_{m,m}$ does not change by too much.  Let us write the power series expansion
\[
\exp \left(  y \mu(X^{(1)} + X^{(2)}) + \frac{1}{2}y^2\Sigma(X^{(1)} + X^{(2)}) \right) = \sum_{j=0}^{\infty} \frac{R_j(X^{(1)}, X^{(2)})y^j}{j!} \,.
\]
For each $j$, let $R_{j,j}(X^{(1)}, X^{(2)})$ denote the degree $(j,j)$ part of $R_{2j}$.  Next note that 
\[
\frac{Q_{m,m}(X^{(1)} , X^{(2)})}{(2m)!} = \sum_{j=0}^m  \frac{(X^{(1)} \cdot X^{(2)})^{m-j}R_{j, j}(X^{(1)}, X^{(2)})}{(m - j)!(2j)!} \,.
\]
Consider the expression $(X^{(1)} \cdot X^{(2)})^{m - j}R_{j, j}(X^{(1)}, X^{(2)})$ for a fixed $j$.  Note that by choosing $c_m$ sufficiently small, we can ensure, using Claim \ref{claim:polytensorinjnormbound}, that 
\[
\norm{T_{\sym}(R_{j, j}(X^{(1)}, X^{(2)}))}_2
\]
is bounded by a sufficiently small function of $m$ for $ 1 \leq j \leq m$.  Next, note that 
\[
T_{\sym}((X^{(1)} \cdot X^{(2)})^{m - j}) = I
\]
where $I$ is the identity matrix of the appropriate size.  Thus, by Claim \ref{claim:polytensorinjnormbound}. we can ensure that
\[
\norm{T_{\sym}\left((X^{(1)} \cdot X^{(2)})^{m - j}R_{j,j}(X^{(1)}, X^{(2)})\right)}_{\textsf{op}}
\]
is bounded by a sufficiently small function of $m$ for all $1 \leq j \leq m$.  Thus, if we let 
\[
A = T_{\sym} \left(Q_{m,m}(X^{(1)} , X^{(2)}) - \frac{(2m)!}{m!}(X^{(1)} \cdot X^{(2)})^{m} \right) = T_{\sym} \left((2m)!\sum_{j=1}^m  \frac{(X^{(1)} \cdot X^{(2)})^{m-j}R_{j, j}(X^{(1)}, X^{(2)})}{(m - j)!(2j)!} \right)
\]
then we can ensure that $\norm{A}_{\textsf{op}}$ is bounded by a sufficiently small function of $m$.

However, the symmetric tensorization of $(X^{(1)} \cdot X^{(2)})^{m}$ is exactly the identity matrix $I$.  Thus, we have 
\[
T \geq 0.9 I 
\]
where recall 
\[
T  = \E_{z \sim G}[v(H_m(X^{(1)},z)) \otimes v(H_m(X^{(2)},z)) ]  = T_{\sym}\left(\binom{2m}{m}^{-1} Q_{m,m}(X^{(1)} , X^{(2)}) \right) \,.
\]

To bound the covariance of $v(H_m(X,z))$, it remains to compute
\[
\E_{z \sim G}[v(H_m(X,z))] \otimes \E_{z \sim G}[v(H_m(X,z))] \,.
\]
However, by Claim \ref{claim:key-hermite-identity}, $\E_{z \sim G}[H_m(X,z)] = h_{m,G}(X)$ is exactly the coefficient of $y^m/m!$ in the power series expansion of 
\[
e^{\mu(X)y + \frac{1}{2}\Sigma(X)y^2} \,.
\]
As long as $c_m$ is sufficiently small, we can use Claim \ref{claim:polytensor-prod-bound} to get that 
\[
\norm{\E_{z \sim G}[v(H_m(X,z))] \otimes \E_{z \sim G}[v(H_m(X,z))]}_2 = \norm{v(h_{m,G}(X))}^2 \leq 0.1 \,.
\]
Putting everything together, we get that 
\[
\Sigma_{H_m} = T -  \E_{z \sim G}[v(H_m(X,z))] \otimes \E_{z \sim G}[v(H_m(X,z))] \geq 0.5 I \,.
\]
\end{proof}

\subsection{Tail Bounds and Stability Bounds }

Now we prove that the distribution of $H_m(X,z)$ for $z $ drawn from a Gaussian $G = N(\mu, I + \Sigma)$ has exponential tail decay.  This will let us obtain tight stability bounds for samples drawn from this distribution.  The stability bounds will then be plugged into existing algorithms for robust mean estimation (see e.g. \cite{diakonikolas2019recent} for an explanation of stability and its use in robust mean estimation).

The proofs in this section are mostly standard and many of them are deferred to Appendix \ref{appendix:tailbounds}.  First we need a tail bound on the distribution of $H_m(X,z)$.

\begin{lemma}\label{lem:hermite-tail}
Let $G =  N(\mu, I + \Sigma)$.  Consider the vector $v(H_m(X,z))$ for $z \sim G$.  Let $u$ be any unit vector with the same dimensionality.  There are positive constants $c_m, C_m$ depending only on $m$ such that for any real number $t > 1$
\[
\Pr\left[ \abs{v(H_m(X,z)) \cdot u} \geq t (2 + \norm{\mu}_2 + \norm{\Sigma}_2)^{C_m} \right] \leq e^{-t^{c_m}}
\]
\end{lemma}
\begin{proof}
We will use  Claim \ref{claim:single-gaussian-hermite-covariance}  and Claim \ref{claim:hypercontractivity} to bound 
\[
\E_{z \sim G}\left[ (v(H_m(X,z)) \cdot u)^k \right]
\]
for some appropriately chosen even integer $k$ and then use Markov's inequality.  First, Claim \ref{claim:single-gaussian-hermite-covariance} implies
\[
\E_{z \sim G}\left[ (v(H_m(X,z)) \cdot u)^2 \right] \leq (m(1 + \norm{\mu}_2 + \norm{\Sigma}_2))^{O(m)} \,.
\]
Now note that for a fixed $u$, $v(H_m(X,z)) \cdot u$ is a polynomial in $z$ of degree at most $m$.  Thus, by Claim \ref{claim:hypercontractivity}, we get that for even integers $k$, 
\[
\E_{z \sim G}\left[ (v(H_m(X,z)) \cdot u)^k \right] \leq  (mk(1 + \norm{\mu}_2 + \norm{\Sigma}_2))^{O(mk)} \,.
\]

Now by Markov's inequality,
\[
\Pr\left[ \abs{v(H_m(X,z)) \cdot u} \geq t(2 + \norm{\mu}_2 + \norm{\Sigma}_2)^{C_m} \right] \leq \frac{(mk(1 + \norm{\mu}_2 + \norm{\Sigma}_2))^{O(mk)}}{t^k(2 + \norm{\mu}_2 + \norm{\Sigma}_2)^{C_m k}} \,.
\]
Choosing $C_m$ to be sufficiently large in terms of $m$ and $k = t^{c_m} $ for some sufficiently small positive constant $c_m$ depending only on $m$ gives the desired inequality.
\end{proof}

Now that we have shown that the distribution of $H_m(X,z)$ exhibits exponential tail decay in all directions, we  can prove finite sample concentration inequalities.  First we prove a concentration inequality in 1D, stating that for a set of samples from a distribution with exponential tail decay, with high probability, the empirical mean of any $(1-\eps)$ fraction of the samples is within $\wt{O}(\eps)$ of the true mean.

\begin{claim}\label{claim:poly-gaussian-tail}
Let $\mcl{D}$ be a distribution on $\R$ and $0 < c < 1$ be a positive constant such that for all real numbers $t > 1$,
\[
\Pr_{x \sim \mcl{D}} [ |x| \geq   t   ] \leq e^{-t^{c}} \,. 
\]
Let $\eps < 1/2$ and  $d$ be parameters.  Given a set $S$ of $n \geq (d/\eps)^{10^5/c}$ independent samples from $\mcl{D}$, with probability at least $1 -  e^{-(8d/\eps)^2}$, any subset $S' \subseteq S$ of size at least $(1-\eps)n$ satisfies
\[
\left \lvert \mu_{\mcl{D}} -  \frac{1}{|S'|} \sum_{x \in S'} x   \right \rvert \leq \eps \log^{1/c} (1/\eps) \left(\frac{10^2}{c}\right)^{10/c} \,.
\]
\end{claim}
\begin{proof}
See Appendix \ref{appendix:tailbounds}

\end{proof}

Now, we are ready to introduce the definition of stability of a set of samples in $\R^d$.  The definition below is standard in robust statistics literature (see e.g. \cite{diakonikolas2019recent}).
\begin{definition}
For $\eps > 0$ and $\delta \geq \eps$, we say a finite set $S \subset \R^d$ is $(\eps, \delta)$-stable with respect to a distribution $\mcl{D}$ if for every unit vector $v \subset \R^d$ and every subset $S' \subseteq S$ of size at least $(1- \eps)|S|$ we have
\begin{align*}
&\left \lvert v \cdot \left(\mu_{\mcl{D}} - \frac{1}{S'}\sum_{x \in S'} x \right)  \right \rvert    \leq \delta \\
&\left \lvert \frac{1}{S'}\sum_{x \in S'}(v \cdot (x - \mu_{\mcl{D}}))^2 - 1  \right \rvert    \leq  \frac{\delta^2}{\eps}
\end{align*}
\end{definition}

Note that the previous definition makes sense for a distribution $\mcl{D}$ that has covariance $I$.  However, we will need to work with distributions with unknown covariance.  We make the following analogous definition of stability for distributions with unknown covariance.

\begin{definition}
For $\eps > 0$ and $\delta \geq \eps$, we say a finite set $S \subset \R^d$ is $(\eps, \delta)$-pseudo-stable with respect to a distribution $\mcl{D}$ if for every unit vector $v \subset \R^d$ and every subset $S' \subseteq S$ of size at least $(1- \eps)|S|$ we have
\begin{align*}
&\left \lvert v \cdot \left(\mu_{\mcl{D}} - \frac{1}{S'}\sum_{x \in S'} x \right)  \right \rvert    \leq \delta \\
& \left \lvert v^T \left(\Sigma_{\mcl{D}} - \frac{1}{S'}\sum_{x \in S'}(x - \mu_{\mcl{D}}) (x - \mu_{\mcl{D}})^T \right)v  \right \rvert    \leq  \frac{\delta^2}{\eps}    
\end{align*}
\end{definition}
\begin{remark}
Note that if the covariance matrix of $\mcl{D}$, $\Sigma_{\mcl{D}}$, is well-conditioned, then a set of samples that is $(\eps, \delta)$-pseudo-stable can be transformed into a set of samples that is $(O(\eps), O(\delta))$-stable by applying a suitable linear transformation.  However, if the covariance matrix is not well-conditioned, then the definitions of stability and pseudo-stability are incomparable.
\end{remark}

We now prove the main result of this section.  It states that for a set $S$ of samples drawn from the distribution of $H_m(X,z)$ for $z \sim \mcl{M}$ where $\mcl{M}$ is a regular-form mixture of Gaussians, $S$ is $(\eps, \wt{O}(\eps))$ pseudo-stable with high probability.  The proof involves combining Lemma \ref{lem:hermite-tail} and Claim \ref{claim:poly-gaussian-tail} and union bounding over a sufficiently fine discrete net over the set of all possible directions.  The details are deferred to Appendix \ref{appendix:tailbounds}.

\begin{claim}\label{claim:hermite-stability}
Consider a set of Gaussians $G_1 = N(\mu_1, I + \Sigma_1), \dots , G_k = N(\mu_k, I + \Sigma_k)$ in $\R^d$ and assume they are in $(\alpha,\beta)$-regular form.

Consider the mixture $\mcl{M} = w_1 G_1 + \dots + w_kG_k$ (where $w_1, \dots , w_k$ are nonnegative weights summing to $1$).  Let $m$ be a positive integer and let $n > \poly_{k,m}(d/\eps)$ for some sufficiently large polynomial.  Let $\mcl{D}$ be the distribution of $v(H_m(X,z))$ for $z$ drawn from $\mcl{M}$.  Consider a set $S$ of $n$ such samples (drawn independently) $x_1 = v(H_m(X,z_1)), \dots , x_n = v(H_m(X,z_n)) $.  This set of $n$ samples is $(\eps, \delta)$-pseudo-stable with
\[
\delta = \eps\left(2 + \alpha + \beta + \log(1/\eps) \right)^{O_{m,k}(1)}
\]
with probability $1 - e^{-10d/\eps}$.
\end{claim}
\begin{proof}
See Appendix \ref{appendix:tailbounds}.
\end{proof}

\section{Estimating the Hermite Moment Polynomials}\label{sec:est-hermite}

In this section, we show how to estimate the Hermite moment polynomials of a regular-form mixture of Gaussians $\mcl{M} = w_1 G_1 + \dots + w_kG_k$ where $G_j = N(\mu_j, I + \Sigma_j)$ to optimal accuracy.  The main theorem that we will prove is as follows.
\begin{theorem}\label{thm:estimate-hermite}
Consider a mixture of Gaussians $\mcl{M} = w_1 G_1 + \dots + w_kG_k$ in $\R^d$.  Let $m$ be a positive integer that is sufficiently large in terms of $k$.  Assume that $\mcl{M}$ is in $(\alpha, \beta, \gamma)$-regular form where 
\begin{itemize}
    \item $\alpha \leq \poly(\log 1/\eps)$
    \item $\beta \leq \poly(\log 1/\eps)$
    \item $\gamma$ is sufficiently small in terms of $k$ and $m$.
\end{itemize}  
Further assume that $w_{\min} \geq \theta$ for some  constant $\theta$.  Let $n > \poly_{k,m}(d/\eps)$ for some sufficiently large polynomial.  Assume that we are given an $\eps$-corrupted set of $n$ samples from $\mcl{M}$, say $z_1, \dots , z_n$.  There is an algorithm that runs in time $\poly_{k,m}(d/\eps)$ and with probability $1 - e^{-d/\eps}$ (over the random samples) outputs estimates $h_1', \dots , h_m'$ for the Hermite moment polynomials of the mixture such that 
\[
\norm{v(h_{j}(X) - h_j'(X))} \leq (2 + \alpha + \beta  + \theta^{-1} + \log 1/\eps)^{O_{k,m}(1)} \eps
\]
for all $j \leq m$ where $h_1, \dots , h_m$ are the true Hermite moment polynomials of $\mcl{M}$.
\end{theorem}

Recall the outline of the proof of Theorem \ref{thm:estimate-hermite} in Section \ref{sec:regular-form-overview}.  Assume that we have some initial estimates $\wt{h_1}, \dots , \wt{h_m}$ for the first $m$ Hermite moment polynomials.  We will first use the recurrence relations between Hermite moment polynomials to obtain estimates $\wt{h_{m+1}}, \dots , \wt{h_{2m}}$ for the first $2m$ Hermite moment polynomials.  This is done in Section \ref{sec:hermite-recurrence}.  Next, recall that the Hermite moment polynomials $h_0, \dots , h_m$ are the means of the distributions $H_m(X,z)$ for $z \sim \mcl{M}$.  We use our estimates of $\wt{h_{0}}, \dots , \wt{h_{2m}}$ to compute estimates for the covariances of these distributions, say $\wt{\Sigma_{H_1}}, \dots , \wt{\Sigma_{H_m}}$.  This is done in Section \ref{sec:compute-cov}.  Using these estimates for the covariances, we can refine our estimates for the means, obtaining a finer set of estimates, say  $\wh{h_1}, \dots , \wh{h_m}$.  We prove that by iterating the above, we can get down to $\wt{O}(\eps)$ accuracy.  This is done in Section \ref{sec:hermite-est-final}.
\\\\

We begin with a simple consequence of the results in Section \ref{sec:hermite-cov-basic}, that for a mixture $\mcl{M}$ in $(\alpha,\beta,\gamma)$-regular form, we have a lower and upper bound on the covariance of $H_m(X,z)$ for $z \sim \mcl{M}$.

\begin{claim}\label{claim:hermite-covariance-bound}
Let $\mcl{M} = w_1 G_1 + \dots + w_kG_k$ where $G_j = N(\mu_j, I + \Sigma_j)$ be a mixture that is in $(\alpha, \beta, \gamma)$-regular form.  Let $m$ be a positive integer.  Assume that $w_{\min} \geq \theta$.  Let $\Sigma_{H_m}$ be the covariance of $v(H_m(X,z))$ for $z$ drawn from $\mcl{M}$. Then
\begin{itemize}
    \item $\Sigma_{H_m} \leq (2 + \alpha + \beta)^{O_m(1)}I$
    \item As long as $\gamma$ is sufficiently small in terms of $m$ then
    \[
        \Sigma_{H_m} \geq \frac{1}{2}\theta I \,.
    \]
\end{itemize}   

\end{claim}
\begin{proof}
For the first part, note that 
\begin{align*}
\Sigma_{H_m}  \leq \E_{z \sim \mcl{M}}\left[ v(H_m(X,z)) \otimes v(H_m(X,z)) \right]  = \sum_{j=1}^k w_j\E_{z \sim G_j}\left[ v(H_m(X,z)) \otimes v(H_m(X,z)) \right] \,.
\end{align*}
Applying Claim \ref{claim:single-gaussian-hermite-covariance} completes the proof of the first part.

Now we prove the second part.  For a Gaussian $G_j$, let $\Sigma_{H_m,G_j}$ be the covariance of $v(H_m(X,z))$ for $z$ drawn from $G_j$.  Note that 
\[
\Sigma_{H_m} \geq w_j\Sigma_{H_m,G_j} \geq  \theta \Sigma_{H_m,G_j}
\]
for all $j$.  Now since the original mixture is in $(\alpha, \beta, \gamma)$-regular form, taking $j$ to be the component such that $\norm{\mu_j} + \norm{\Sigma_j}_2 \leq \gamma$ and applying Claim \ref{claim:nearlyisotropicgaussian} completes the proof.
\end{proof}

\subsection{More Recurrence Relations}\label{sec:hermite-recurrence}
The main goal in this subsection is to prove that given $\eps$-accurate estimates of $h_{0}, \dots , h_{m}$ where $m$ is sufficiently large in terms of $k$, then we can compute $\wt{O}(\eps)$-accurate estimates for $h_{m+1}, \dots , h_{2m}$.

The subroutine will rely heavily on recurrence relations between the Hermite moment polynomials.  Recall Lemma \ref{lem:hermite-moment-recurrence}.  Applying Lemma \ref{lem:hermite-moment-recurrence} to a mixture of Gaussians $\mcl{M}$ implies that the Hermite moment polynomials $h_{j, \mcl{M}}$ satisfy a recurrence of order $O_{k}(1)$ whose coefficients are ``simple".  We will now develop additional tools based on these recurrence relations that will allow us to estimate $h_{m+1}, \dots , h_{2m}$ using $h_{0}, \dots , h_{m}$.

 In this section, we use the following notation.  
\begin{itemize}
    \item We have a mixture of $k$ Gaussians $\mcl{M} = w_1 G_1 + \dots + w_kG_k$ where $G_j = N(\mu_j, I + \Sigma_j)$.
    \item Recall Corollary \ref{coro:hermite-of-mixture}.  It will be particularly important to consider the generating function
\[
f_{\mcl{M}}(X,y) = \sum_{m=0}^{\infty} \frac{1}{m!} \cdot h_{m,\mcl{M}}(X) y^m = w_1 e^{\mu_1(X)y + \frac{1}{2}\Sigma_1(X)y^2} + \dots +  w_k e^{\mu_k(X)y + \frac{1}{2}\Sigma_k(X)y^2} \,. 
\]
\item Similar to Section \ref{sec:hermite-to-tv}, we use $\mcl{D}_j$ to denote the differential operator $(\partial - (\mu_j(X) + \Sigma_j(X)y))$ where the partial derivative is taken with respect to $y$.
\end{itemize}

The next result builds on Lemma \ref{lem:hermite-moment-recurrence} and says that for \textit{any} recurrrence of order $2\kappa = 2(2^k - 1)$ that the first $O_k(1)$ Hermite moment polynomials almost satisfy, we can \textit{extend} the recurrence to estimate the next several Hermite moment polynomials.

\begin{claim}\label{claim:recurrence-for-difference}
Let $\kappa = 2^k - 1$.  Let $T_{j,l}(X)$ for $0 \leq j \leq \kappa - 1$ and $0 \leq l \leq \kappa - j $ be polynomials such that $T_{j,l}(X)$ is homogeneous in $X$ of degree $\kappa - j + l$.

Let  $\mcl{M} = w_1 G_1 + \dots + w_kG_k$ be a mixture of $k$ Gaussians in $(\alpha,\beta)$-regular form.  Let $m$ be a positive integer that is sufficiently large in terms of $k$.  For all $a \geq \kappa$, let
\[
D_a(X) = h_{a, \mcl{M}}(X) + (a - \kappa)! \sum_{j=0}^{\kappa - 1}\sum_{l = 0}^{\kappa - j} \frac{h_{a - \kappa + j - l, \mcl{M}}(X)T_{j,l}(X)}{(a-\kappa - l)!} \,,
\]
where undefined terms (i.e. negative factorials in the denominator) are treated as $0$.  Assume that for all $ a \leq m$, 
\[
\norm{v(D_a(X))} \leq \eps \,.
\]
Then for all $ m \leq a \leq 2m $,
\[
\norm{v(D_a(X))} \leq (2 + \alpha  )^{O_{m, k}(1)} \eps \,.
\]
\end{claim}
\begin{proof}

For each $j$ with $0 \leq j \leq \kappa - 1$, let 
\[
T_j(X,y) = \sum_{l= 0}^{\kappa - j} T_{j,l}(X)y^l \,.
\]
Now define the differential operator
\[
\mcl{T} = \partial^{\kappa} + T_{\kappa - 1}(X,y)\partial^{\kappa - 1} + \dots + T_0(X,y) \,.
\]
Consider the generating function $\mcl{T}(f_{\mcl{M}}(X,y))$.  Note that its power series expansion is precisely
\[
\mcl{T}(f_{\mcl{M}}(X,y)) = \sum_{j=0}^{\infty} \frac{D_{\kappa + j}(X)y^j}{j!} \,.
\]
Alternatively, applying the operator $\mcl{T}$ to the sum-of-exponentials form of $f_{\mcl{M}}$, we see that $\mcl{T}(f_{\mcl{M}}(X,y))$ can be written in the form 
\[
\mcl{T}(f_{\mcl{M}}(X,y)) = P_1(X,y) e^{\mu_1(X)y + \frac{1}{2}\Sigma_1(X)y^2} + \dots +  P_k(X,y) e^{\mu_k(X)y + \frac{1}{2}\Sigma_k(X)y^2}
\]
where each of the polynomials $P_1, \dots , P_k$ has degree at most $\kappa$ in $y$.  Thus if we let
\[
\mcl{D} = \mcl{D}_k^{2^{k-1}(\kappa + 1)} \mcl{D}_{k-1}^{2^{k-2}(\kappa + 1)}\dots \mcl{D}_1^{(\kappa + 1)}
\]
then by repeatedly applying Claim \ref{claim:degree-reduction}, we get
\[
\mcl{D}\left(\mcl{T}(f_{\mcl{M}}(X,y))\right)  = 0 \,.
\]  
Note that we can use Claim \ref{claim:diff-operator-expansion} to write the differential operator $\mcl{D}$ in the form
\[
\mcl{D} = \partial^{\kappa(\kappa + 1)} + R_{\kappa(\kappa + 1)-1}(X,y) \partial^{\kappa(\kappa + 1) -1} + \dots + R_0(X,y) \,.
\]
We can then write each $R_j$ in the form
\[
R_j(X,y) =  R_{j, \kappa(\kappa + 1) -j}(X)y^{\kappa(\kappa + 1) -j} + \dots + R_{j, 0}(X)
\]
where each of the polynomials $R_{j,l}$ is homogeneous in $X$ with degree $\kappa (\kappa + 1) - j + l$ and is $(O_k(1), O_k(1))$-simple with respect to $\{\mu_1(X), \Sigma_1(X), \dots , \mu_k(X), \Sigma_k(X) \}$.  Using that $\mcl{D}\left(\mcl{T}(f_{\mcl{M}}(X,y))\right)  = 0$, we can now write a recurrence relation that the polynomials $D_a(X) $ must satisfy.  In particular, if we let $\lambda = \kappa(\kappa + 1)$, we must have for all $a \geq \lambda$
\[
\frac{ D_{\kappa + a}(X)}{(a - \lambda)!}  + \sum_{j=0}^{\lambda - 1}\sum_{l = 0}^{\lambda - j} \frac{D_{\kappa + a - \lambda+ j - l}(X)R_{j,l}(X)}{(a- \lambda - l)!}  = 0 \,.
\]
This rearranges into
\begin{equation}\label{eq:recurrence1}
 D_{\kappa + a}(X) = -(a - \lambda)!\sum_{j=0}^{\lambda - 1}\sum_{l = 0}^{\lambda - j} \frac{D_{\kappa + a - \lambda+ j - l, }(X)R_{j,l}(X)}{(a- \lambda - l)!} \,.
\end{equation}
Now we can use the recurrence (\ref{eq:recurrence1}) to compute $D_{m+1}(X), \dots , D_{2m}(X)$ in terms of the earlier terms in the sequence.  Note that we are given
\[
\norm{v(D_a(X))} \leq \eps
\]
for all $a \leq m$.  Also, since $R_{j,l}$ is $(O_k(1), O_k(1))$-simple with respect to $\{\mu_1(X), \Sigma_1(X), \dots , \mu_k(X), \Sigma_k(X) \}$, we have by Claim \ref{claim:polytensor-prod-bound} that
\[
\norm{v(R_{j,l}(X))} \leq (2 + \alpha  )^{O_{k}(1)} \,.
\]
Thus, when applying the recurrence to compute $D_{m+1}(X), \dots , D_{2m}(X)$, we have for all $a \leq 2m$, 
\[
\norm{v(D_a(X))} \leq (2 + \alpha )^{O_{m, k}(1)} \eps \,.
\]

\end{proof}

As a consequence of the above, given estimates for the Hermite moment polynomials $h_0, \dots , h_m$, we can estimate the next Hermite moment polynomials $h_{m+1}, \dots , h_{2m}$ by first solving for a recurrence relation that the first $m$ Hermite moment polynomials satisfy, and then extending this recurrence to compute $h_{m+1}, \dots , h_{2m}$.
\begin{claim}\label{claim:hermite-estimate-using-recurrence}
Let  $\mcl{M} = w_1 G_1 + \dots + w_kG_k$ be a mixture of $k$ Gaussians in $(\alpha,\beta)$-regular form.  Let $m$ be a positive integer that is sufficiently large in terms of $k$.  Assume that we are given estimates $h_0'(X), \dots , h_m'(X)$ such that 
\[
\norm{v(h_{a}(X) - h_a'(X) )} \leq \eps
\]
for all $a \leq m$ (where $h_a(X)$ are the true Hermite moment polynomials of $\mcl{M}$).  Then there is an algorithm that runs in $\poly_{m,k}(d)$ time that computes estimates $h_{m+1}'(X), \dots h_{2m}'(X)$  such that for all $a \leq 2m$
\[
\norm{v(h_a(X) - h_a'(X) )} \leq  (2 + \alpha)^{O_{m,k}(1)} \eps  \,.
\]
\end{claim}
\begin{proof}
Let $\kappa = 2^k - 1$.  We first solve for polynomials $R'_{j,l}(X)$ for all $0 \leq j \leq \kappa $ and $0 \leq l \leq \kappa - j$ such that 
\begin{itemize}
    \item $R'_{j,l}$ is homogeneous of degree $\kappa - j + l$ and $R'_{\kappa,0} = 1$
    \item $\norm{v(R'_{j,l}(X))} \leq (2 + \alpha)^{O_{m,k}(1)}$
    \item For all $a \leq m$
    \[
        \norm{v\left(\sum_{j=0}^{\kappa }\sum_{l = 0}^{\kappa - j} \frac{h_{a - \kappa + j - l}'(X)R_{j,l}'(X)}{(a-\kappa - l)!}\right)} \leq  (2 + \alpha)^{O_{m,k}(1)} \eps
    \]
\end{itemize}  

Note that the expressions inside the norms on the LHS are linear in the coefficients of the $R'_{j,l}$.  Thus, we can solve for the coefficients via a convex program.  To see that a solution exists, let $R'_{j,l}(X) = R_{j,l}(X)$ where the $R_{j,l}$ are the polynomials given by Lemma \ref{lem:hermite-moment-recurrence}.  It is clear that the first two conditions are satisfied because the $R_{j,l}$ are $\left(O_k(1), O_k(1)\right)$-simple with respect to $\{\mu_j(X)\}_{j \in [k]} , \{ \Sigma_j(X) \}_{j \in [k]}$.  Next,
\begin{align*}
\norm{v\left(\sum_{j=0}^{\kappa }\sum_{l = 0}^{\kappa - j} \frac{h_{a - \kappa + j - l}'(X)R_{j,l}(X)}{(a-\kappa - l)!}\right)}  &= \norm{v\left(\sum_{j=0}^{\kappa }\sum_{l = 0}^{\kappa - j} \frac{(h_{a - \kappa + j - l}'(X) - h_{a - \kappa + j - l}(X)) R_{j,l}(X)}{(a-\kappa - l)!}\right)} \\ &\leq (2 + \alpha)^{O_{m,k}(1)} \eps
\end{align*}
for all $a \leq m$, where the last step holds because $\norm{v(h_a(X) - h_a'(X) )} \leq \eps$ for all $a \leq m$ and $\norm{v(R_{j,l}(X))} \leq (2 + \alpha)^{O_{m,k}(1)}$ by Claim \ref{claim:polytensor-prod-bound}.

Now we consider what happens when we apply the recurrence given by the $R'_{j,l}$, namely
\begin{equation}\label{eq:recurrence2}
h'_{a}(X) = - (a - \kappa)! \sum_{j=0}^{\kappa - 1}\sum_{l = 0}^{\kappa - j} \frac{h'_{a - \kappa + j - l}(X)R'_{j,l}(X)}{(a-\kappa - l)!} 
\end{equation}
and use the above to compute estimates $h'_a(X)$ for $m+1 \leq a \leq 2m$.  Note that for $a \leq m$,
\begin{align*}
\norm{v\left(\sum_{j=0}^{\kappa }\sum_{l = 0}^{\kappa - j} \frac{h_{a - \kappa + j - l}(X)R'_{j,l}(X)}{(a-\kappa - l)!}\right)}  &\leq  \norm{v\left(\sum_{j=0}^{\kappa }\sum_{l = 0}^{\kappa - j}  \frac{(h_{a - \kappa + j - l}(X) - h'_{a - \kappa + j - l}(X)) R'_{j,l}(X)}{(a-\kappa - l)!}\right)} \\ &\quad + \norm{v\left(\sum_{j=0}^{\kappa }\sum_{l = 0}^{\kappa - j} \frac{h'_{a - \kappa + j - l}(X)R'_{j,l}(X)}{(a-\kappa - l)!}\right)} 
\\ &\leq (2 + \alpha)^{O_{m,k}(1)} \eps \,.
\end{align*}
By Claim \ref{claim:recurrence-for-difference}, we deduce that 
\[
\norm{v\left(\sum_{j=0}^{\kappa }\sum_{l = 0}^{\kappa - j} \frac{h_{a - \kappa + j - l}(X)R'_{j,l}(X)}{(a-\kappa - l)!}\right)} \leq (2 + \alpha)^{O_{m,k}(1)} \eps
\]
for all $a \leq 2m$.  Finally, comparing to the recurrence in (\ref{eq:recurrence2}) and subtracting, we get that for all $m+1 \leq a \leq 2m$,
\begin{align*}
\norm{v(h'_{a}(X)  - h_a(X))} &\leq (a - \kappa)! \norm{ v\left(\sum_{j=0}^{\kappa - 1}\sum_{l = 0}^{\kappa - j} \frac{(h'_{a - \kappa + j - l}(X) - h_{a - \kappa + j - l}(X))R'_{j,l}(X)}{(a-\kappa - l)!}\right)} \\ &\quad +  (a - \kappa)! \norm{v\left(\sum_{j=0}^{\kappa }\sum_{l = 0}^{\kappa - j} \frac{h_{a - \kappa + j - l}(X)R'_{j,l}(X)}{(a-\kappa - l)!}\right)} \\ & \leq (2 + \alpha)^{O_{m,k}(1)}( \eps + \max_{b < a}\norm{v(h_{b}'(X) - h_{b}(X))} )\,.
\end{align*}
Since we need to apply the recurrence at most $m$ times to get the terms $h'_a(X)$ for $a \leq 2m$, the above implies that for all $a \leq 2m$
\[
\norm{v(h_a(X) - h_a'(X) )} \leq  (2 +\alpha)^{O_{m,k}(1)} \eps 
\]
as desired.  It is clear that all of the steps, solving for the $R'_{j,l}$ and computing subsequent terms using the recurrence, can be done in $\poly_{k,m}(d)$ time.
\end{proof}

\subsection{Computing the Covariance of the Hermite Moment Polynomials}\label{sec:compute-cov}

In the previous section, we showed how to compute $h_{m+1}, \dots , h_{2m}$ from $h_0, \dots , h_m$.  Now we show how to express the covariances of the distributions of $H_0(X,z), \dots , H_{m}(X,z)$ (for $z \sim \mcl{M}$) in terms of $h_0, \dots , h_{2m}$.   

\begin{claim}\label{claim:est-hermite-cov}
Let  $\mcl{M} = w_1 G_1 + \dots + w_kG_k$ be a mixture of $k$ Gaussians in $(\alpha,\beta)$-regular form.  Let $\Sigma_{H_m}$ be the covariance of $v(H_m(X,z))$ for $z$ drawn from $\mcl{M}$.  Assume that we are given estimates, say $h_{0}'(X), \dots , h_{2m}'(X)$,  such that
\[
\norm{v(h_{j}(X) - h_{j}'(X) )} \leq \eps
\]
for all $j \leq 2m$ where $h_j(X)$ are the Hermite moment polynomials of the mixture $\mcl{M}$.  Then in $\poly_{k,m}(d)$ time, we can compute a symmetric matrix $M'$ such that 
\[
\norm{ M' - \Sigma_{H_m}}_{\textsf{op}} \leq (2 + \alpha)^{O_m(1)} \eps \,.
\]
\end{claim}
\begin{proof}
Consider two sets of $d$ variables $X^{(1)} = (X^{(1)}_1, \dots , X^{(1)}_d) $ and $X^{(2)} = (X^{(2)}_1, \dots , X^{(2)}_d) $.  Let
\[
P(X^{(1)}, X^{(2)}) = \E_{z \sim \mcl{M}} \left[H_m(X^{(1)},z)  H_m(X^{(2)},z) \right]  \,.
\]
We will first show how to estimate the polynomial $P$.  Define
\[
\mcl{A}(y) = \sum_{j=1}^k w_je^{y\mu_j(X^{(1)} + X^{(2)}) + \frac{1}{2}y^2\Sigma_j(X^{(1)} + X^{(2)})}  \,.
\]
By Corollary \ref{coro:hermite-of-mixture},
\[
\mcl{A}(y)  = \sum_{j=0}^{\infty} \frac{h_{j}(X^{(1)} + X^{(2)}) y^j}{j!} \,.
\]
On the other hand, let $Q_0,Q_1, \dots, Q_j  $ be the terms of the power series
\begin{equation}\label{eq:prodpowerseries}
e^{y^2(X^{(1)} \cdot X^{(2)})} \mcl{A}(y) = \sum_{j=1}^k w_je^{y\mu_j(X^{(1)} + X^{(2)}) + \frac{1}{2}y^2\Sigma_j(X^{(1)} + X^{(2)}) + y^2(X^{(1)} \cdot X^{(2)})} =  \sum_{j=0}^{\infty} \frac{Q_j(X^{(1)}, X^{(2)}) y^j}{j!} \,.
\end{equation}
Let $Q_{m,m}$ be the degree $(m,m)$ part of $Q_{2m}$.  By Claim \ref{claim:hermite-covariance-basic}, we know that
\begin{equation}\label{eq:rearrange}
P(X^{(1)}, X^{(2)}) = \binom{2m}{m}^{-1}Q_{m,m}(X^{(1)}, X^{(2)} ) \,.
\end{equation}

In particular, we can use our estimates $h_0', \dots h_{2m}'$ to estimate the first $2m + 1$ terms of $\mcl{A}$ and then multiply by 
\[
e^{y^2(X^{(1)} \cdot X^{(2)})}  = \sum_{j=0}^{\infty} \frac{(X^{(1)} \cdot X^{(2)})^jy^{2j}}{j!}
\]
which is an explicit power series that we can compute.  Formally, let $h_{j,j}(X^{(1)} ,X^{(2)})$ denote the degree $j,j$ part of $h_{2j}(X^{(1)} + X^{(2)})$.  Similarly, we use $h_{j,j}'$ to denote the degree $j,j$ part of $h'_{2j}(X^{(1)} + X^{(2)})$.  Using (\ref{eq:prodpowerseries})
\[
\sum_{j=0}^{\infty} \frac{Q_j(X^{(1)}, X^{(2)}) y^j}{j!} = \left(  \sum_{j=0}^{\infty} \frac{(X^{(1)} \cdot X^{(2)})^jy^{2j}}{j!} \right) \left(\sum_{j=0}^{\infty} \frac{h_{j}(X^{(1)} + X^{(2)}) y^j}{j!} \right) \,.
\]
Thus, by (\ref{eq:rearrange}),
\[
P(X^{(1)}, X^{(2)})  = (m!)^2 \sum_{j=0}^m \frac{h_{j,j}(X^{(1)} + X^{(2)})(X^{(1)} \cdot X^{(2)})^{m-j}}{(2j)!(m-j)!} \,. 
\]
On the other hand, using our estimates $h'$, we may compute 
\[
P'(X^{(1)}, X^{(2)})  = (m!)^2 \sum_{j=0}^m \frac{h'_{j,j}(X^{(1)} + X^{(2)})(X^{(1)} \cdot X^{(2)})^{m-j}}{(2j)!(m-j)!} \,.
\]
Note that since
\[
\norm{v(h_{j}(X) - h_{j}'(X) )} \leq \eps 
\]
for all $j \leq 2m$, we have that
\[
\norm{v(h_{j}(X^{(1)} + X^{(2)}) - h_{j}'(X^{(1)} + X^{(2)}) )} \leq O_m(1) \eps 
\]
for all $j \leq 2m$ where for the vectorization we view the polynomial as a homogeneous polynomial in $2d$ variables.  The above implies
\[
\norm{v(h_{j,j}(X^{(1)} + X^{(2)}) - h_{j,j}'(X^{(1)} + X^{(2)}) )} \leq O_m(1) \eps 
\]
for all $j \leq m$ so we conclude 
\[
\norm{T_{\sym}(h'_{j,j}(X^{(1)} + X^{(2)}) - h_{j,j}(X^{(1)} + X^{(2)}))}_2 \leq O_{m}(1) \eps 
\]
for all $j \leq m$.  Now since 
\[
T_{\sym}\left( (X^{(1)} \cdot X^{(2)})^{m-j}\right) = I
\]
we can use Claim \ref{claim:polytensorinjnormbound} to deduce
\begin{equation}\label{eq:operatorbound1}
\norm{T_{\sym}\left( P(X^{(1)}, X^{(2)}) - P'(X^{(1)}, X^{(2)})\right)}_{\textsf{op}} \leq O_m(1)\eps \,.
\end{equation}
However, note that by Claim \ref{claim:polyproducttosymmetrictensor}
\begin{equation}\label{eq:ddirectexpansion}
T_{\sym}\left( P(X^{(1)}, X^{(2)})\right) = \E_{z \sim \mcl{M}}\left[ v(H_m(X,z)) \otimes  v(H_m(X,z)) \right]\,.
\end{equation}
To compute $\Sigma_{H_m}$, it remains to estimate
\[
E_{z \sim \mcl{M}}[ v(H_m(X,z)) ] \otimes E_{z \sim \mcl{M}}[ v(H_m(X,z)) ] = v(h_{m}(X)) \otimes v(h_{m}(X)) \,.
\]
We can simply use our estimates $h_m'(X)$ and compute $v(h'_{m}(X)) \otimes v(h'_{m}(X))$.  Note
\begin{align*}
\norm{v(h_{m}(X)) \otimes v(h_{m}(X)) - v(h'_{m}(X)) \otimes v(h'_{m}(X))}_{\op} &\leq  \norm{v(h'_m(X)) \otimes v(h'_{m}(X) - h_m(X))}_{\op} \\ &\quad + \norm{v(h'_{m}(X) - h_m(X)) \otimes v(h_m(X))}_{\op} \\ & \leq \norm{v(h'_{m}(X) - h_m(X))}_2 \left(\norm{v(h'_{m}(X)}_2 + \norm{v(h_{m}(X)}_2 \right)  \,.
\end{align*}
However using Claim \ref{claim:single-gaussian-hermite-covariance}, we know that $\norm{v(h_{m}(X)} \leq (2+ \alpha)^{O_m(1)}$.  Thus, the RHS of the above is at most $(2 + \alpha)^{O_m(1)}\eps$.
\\\\
Overall, we can compute 
\[
M' = T_{\sym}( P'(X^{(1)}, X^{(2)})) - v(h'_{m}(X)) \otimes v(h'_{m}(X))
\]
and combining everything we have shown so far (\ref{eq:operatorbound1}, \ref{eq:ddirectexpansion}) gives \begin{align*}
\norm{M' - \Sigma_{H_m}}_{\op} &\leq  \norm{T_{\sym}( P'(X^{(1)}, X^{(2)})) - \E_{z \sim \mcl{M}}\left[ v(H_m(X,z)) \otimes  v(H_m(X,z)) \right] }_{\op} \\ & \quad + \norm{v(h'_{m}(X)) \otimes v(h'_{m}(X)) - E_{z \sim \mcl{M}}[ v(H_m(X,z)) ] \otimes E_{z \sim \mcl{M}}[ v(H_m(X,z)) ]}_{\op} \\ &= \norm{T_{\sym}\left( P'(X^{(1)}, X^{(2)}) - P(X^{(1)}, X^{(2)})\right) }_{\op}  \\ & \quad + \norm{v(h_{m}(X)) \otimes v(h_{m}(X)) - v(h'_{m}(X)) \otimes v(h'_{m}(X))}_{\op}\\&\leq (2 + \alpha)^{O_m(1)}\eps    
\end{align*}  
It is clear that computing $M'$ can be done in $\poly_{k,m}(d)$ time so we are done.
\end{proof}

\subsection{Estimating the Hermite Moment Polynomials Optimally}\label{sec:hermite-est-final}
We can now put together all of the parts in the previous two subsections to get an algorithm for estimating the Hermite moment polynomials to within $\tilde{O}(\eps)$ accuracy, proving Theorem \ref{thm:estimate-hermite}.

We will rely on the following generic theorems about robustly estimating the mean of a distribution from \cite{diakonikolas2019recent}.
\begin{theorem}[Theorem 2.7 in \cite{diakonikolas2019recent}]\label{thm:robust-mean-est-full}
Let $S$ be a $(3\eps , \delta)$-stable set with respect to a distribution $X$ and let $T$ be an $\eps$-corrupted version of $S$.  There is a polynomial time algorithm which given $T$ returns $\wh{\mu}$ such that 
\[
\norm{\wh{\mu} - \mu_X} = O(\delta) \,.
\]
\end{theorem}
\begin{corollary}[Corollary 2.9 in \cite{diakonikolas2019recent}]\label{corollary:robust-mean-est-cov}
Let $T$ be an $\eps$-corrupted set of samples of size at least $\poly(d/\eps)$ for some sufficiently large polynomial from a distribution $X$ on $\R^d$ with unknown bounded covariance $\Sigma_X \leq \sigma^2 I$.  There exists a polynomial time algorithm which given $T$ returns $\wh{\mu}$ such that with $1 - e^{-d/\eps}$ probability 
\[
\norm{\wh{\mu} - \mu_X} = O(\sigma\sqrt{\eps}) \,.
\]
\end{corollary}

Theorem \ref{thm:estimate-hermite} will follow immediately from the next lemma.  The lemma states that given estimates for the first $m$ Hermite  moment polynomials that are accurate to within some parameter $\eta$, we can refine them to be accurate to within roughly $\sqrt{\eps\eta}$.  The proof of the lemma involves combining the results in Sections \ref{sec:hermite-recurrence} and \ref{sec:compute-cov} with Theorem \ref{thm:robust-mean-est-full}.  To see how Theorem \ref{thm:estimate-hermite} follows from the lemma, we can use Corollary \ref{corollary:robust-mean-est-cov} to obtain initial estimates and then iterate Lemma \ref{lem:hermite-poly-est-iterate}.


\begin{lemma}\label{lem:hermite-poly-est-iterate}
Consider a mixture of Gaussians $\mcl{M} = w_1 G_1 + \dots + w_kG_k$ in $\R^d$.  Let $m$ be a positive integer that is sufficiently large in terms of $k$.  Assume that $\mcl{M}$ is in $(\alpha, \beta, \gamma)$-regular form where 
\begin{itemize}
    \item $\alpha \leq \poly(\log 1/\eps)$
    \item $\beta \leq \poly(\log 1/\eps)$
    \item $\gamma$ is sufficiently small in terms of $k$ and $m$.
\end{itemize}  
Further assume that $w_{\min} \geq \theta$ for some  constant $\theta$.  Let $n > \poly_{k,m}(d/\eps)$ for some sufficiently large polynomial.  Assume that we are given an $\eps$-corrupted set of $n$ samples from $\mcl{M}$, say $z_1, \dots , z_n$.  Also assume that we are given estimates $\widetilde{h_1}, \dots , \widetilde{h_m}$ of the Hermite moment polynomials such that
\[
\norm{v(h_j(X) - \widetilde{h_j}(X))} \leq \eta
\]
for all $j \leq m$ where $\eta$ is a parameter such that $\eta \leq 1/\poly_{k,m}(\alpha,\beta, \theta^{-1}, \log 1/\eps)$. Then there is an algorithm that runs in $\poly_{k,m}(d/\eps)$ time and, with probability at least $1 - e^{-10d/\eps}$ over the random samples, outputs estimates $\wh{h_1}, \dots , \wh{h_m}$ such that 
\[
\norm{v(h_j(X) - \wh{h_j}(X))} \leq (2 + \alpha + \beta + \theta^{-1} + \log 1/\eps)^{O_{m,k}(1)} \sqrt{\eps \eta} \,.
\]
\end{lemma}
\begin{proof}
By Claim \ref{claim:hermite-estimate-using-recurrence} and Claim \ref{claim:est-hermite-cov}, we can compute estimates $\Sigma_{H_1}^0, \dots , \Sigma_{H_m}^0$ such that for all $j \leq m$
\begin{equation}\label{eq:cov-diff}
\norm{\Sigma_{H_j}^0 - \Sigma_{H_j}}_{\textsf{op}} \leq (2 + \alpha + \beta)^{O_{k,m}(1)} \eta
\end{equation}
where $ \Sigma_{H_j}$ is the covariance of $v(H_j(X,z))$ for $z$ drawn from $\mcl{M}$.  Now by Claim \ref{claim:hermite-covariance-bound}, each of these estimates must be positive semi-definite and have  $ \Sigma_{H_j}^0 \geq \Omega(\theta) I $ so we can take their positive semidefinite square roots.  We have
\begin{equation}\label{eq:identity-cov}
\norm{I - (\Sigma_{H_j}^0)^{-1/2}\Sigma_{H_j}(\Sigma_{H_j}^0)^{-1/2}}_{\textsf{op}} \leq (2 + \alpha + \beta + \theta^{-1})^{O_{k,m}(1)} \eta \,.
\end{equation}
Now fix a $j \leq m$.  Let $\mcl{D}_j$ be the distribution of $v(H_j(X,z))$ for $z \sim \mcl{M}$.  The above implies that the covariance of the distribution obtained after applying $(\Sigma_{H_j}^0)^{-1/2}$ to $\mcl{D}_j$ is close to identity.

  By Claim \ref{claim:hermite-stability}, with $1 - e^{-10d/\eps}$ probability, a set of $n$ uncorrupted samples is $(3\eps, \delta)$-pseudo-stable with respect to $\mcl{D}_j$ with 
\[
\delta = (2 + \alpha + \beta + \log 1/\eps)^{O_{m,k}(1)} \eps \,.
\]
If this holds, then combining the pseudo-stability with (\ref{eq:identity-cov}) and the fact that $ \Sigma_{H_j}^0 \geq \Omega(\theta) I $ implies that the set of uncorrupted samples obtained after applying the transformation $(\Sigma_{H_j}^0)^{-1/2}$ is stable with respect to the distribution $(\Sigma_{H_j}^0)^{-1/2}\mcl{D}_j$ with parameters
\[
\left( 3\eps, (2 + \alpha + \beta + \theta^{-1} +  \log 1/\eps)^{O_{m,k}(1)} \sqrt{\eps \eta} \right)\,.
\]
Now by Theorem \ref{thm:robust-mean-est-full}, we can estimate the mean of $(\Sigma_{H_j}^0)^{-1/2}\mcl{D}_j$ up to accuracy 
\[
(2 + \alpha + \beta+ \theta^{-1} + \log 1/\eps)^{O_{m,k}(1)} \sqrt{\eps \eta} \,.
\]
Claim \ref{claim:hermite-covariance-bound} and (\ref{eq:cov-diff}) imply that the operator norm of $(\Sigma_{H_j}^0)^{1/2}$ is bounded by $(2 + \alpha + \beta)^{O_{m,k}(1)}$ so now we can simply invert the linear transformation and estimate the mean of $\mcl{D}_j$ to within 
\[
(2 + \alpha + \beta + \theta^{-1} + \log 1/\eps)^{O_{m,k}(1)} \sqrt{\eps \eta} \,.
\]
However, the mean of $\mcl{D}_j$ is exactly $v(h_j(X))$ so repeating this for all $j$ completes the proof.
\end{proof}

We now prove Theorem \ref{thm:estimate-hermite} by iterating Lemma \ref{lem:hermite-poly-est-iterate}.

\begin{proof}[Proof of Theorem \ref{thm:estimate-hermite}]

Note that $v(h_j(X))$ is the mean of $v(H_m(X,z))$ for $z$ drawn from $\mcl{M}$.  Now by Claim \ref{claim:hermite-covariance-bound} and Corollary \ref{corollary:robust-mean-est-cov}, we can obtain estimates 
\[
\widetilde{h_{1}}(X), \dots , \widetilde{h_m}(X)
\]
of the Hermite moment polynomials such that for all $j \leq m$
\[
\norm{v(h_j(X) - \widetilde{h_j}(X))} \leq (2 + \alpha + \beta)^{O_{k,m}(1)} \sqrt{\eps} \,.
\]
Now we can iterate Lemma \ref{lem:hermite-poly-est-iterate} on these estimates until we obtain final estimates $h_1', \dots , h_m'$ such that 
\[
\norm{v(h_j(X) - h_j'(X))} \leq (2 + \alpha + \beta + \theta^{-1} + \log 1/\eps)^{O_{k,m}(1)} \eps
\]
for all $j \leq m$ (to see this, note that each time the above inequality is not true, when we apply Lemma \ref{lem:hermite-poly-est-iterate} to refine our estimates, we reduce our estimation error by a factor of $1/2$).
\end{proof}

\section{Learning Regular Form Mixtures}

In this section, we combine everything that we have shown so far to give an algorithm for learning regular-form mixtures of Gaussians.  Recall the outline in Section \ref{sec:regular-form-overview}.  We have some unknown mixture $\mcl{M} = w_1G_1 + \dots + w_kG_k$ in regular form.  In this section, we will assume that we are given estimates $\ovl{G_1}, \dots , \ovl{G_k}$ for the components such that 
\[
d_{\TV}(G_j, \ovl{G_j}) \leq \eps^c
\]
for some constant $c > 0$.  These estimates can be obtained by directly applying results from \cite{liu2020settling} (see Section \ref{sec:poly-eps-acc} for more details).  

We will then bootstrap the rough component estimates by multiplying appropriate polynomials in front of them.  In particular, we show how to compute a degree-$O_{k,c}(1)$ MPG distribution
\[
f = Q_1(x)\ovl{G_1}(x) + \dots + Q_k(x)\ovl{G_k}(x)
\]
such that 
\[
\norm{\mcl{M} - f}_1 \leq \wt{O}(\eps) \,.
\]
The main theorem that we will prove in this section is stated formally below.
\begin{theorem}\label{thm:estimate-regularform}
Let $\mcl{M} = w_1G_1 + \dots + w_kG_k$  where $G_j = N(\mu_j, I + \Sigma_j)$ be a mixture of Gaussians such that all mixing weights are at least $\theta$ where $\theta^{-1} \leq \poly(\log 1/\eps)$.  Let $c > 0$ be a (small) constant.  Assume that the mixture is in $(\alpha, \beta, \gamma)$-regular form where 
\begin{itemize}
    \item $\alpha = \poly(\log 1/\eps)$
    \item $\beta =   \poly(\log 1/\eps)$
    \item $\gamma$ is sufficiently small in terms of $k$ and $c$.
\end{itemize}  
Let $n > \poly_{k,c}(d/\eps)$ for some sufficiently large polynomial.  Assume that we are given an $\eps$-corrupted set of $n$ samples $X_1, \dots , X_n$ from $\mcl{M}$. Assume that an adversary also gives us estimates $\ovl{G_1} = N(\wt{\mu}_1, I + \wt{\Sigma}_1)  , \dots , \ovl{G_k} =  N(\wt{\mu}_k, I + \wt{\Sigma}_k)$ for the components with the promise that
\[
d_{\TV}(G_j, \ovl{G_j}) \leq \eps^c \,.
\]
There is an algorithm that runs in time $\poly_{k,c}(d/\eps)$ and with high probability (over the random samples) outputs a degree-$O_{k,c}(1)$ MPG distribution $f: \R^d \rightarrow \R$ of the form
\[
f(x) = Q_1(x)\ovl{G}_1(x) + \dots + Q_k(x)\ovl{G}_k (x)
\]
such that $f$ satisfies
\[
\norm{\mcl{M}(x) - f(x)}_1 \leq (2 + \alpha + \beta + \theta^{-1} +  \log 1/\eps)^{O_{k,c}(1)}\eps \,.
\]
\end{theorem}

To prove Theorem \ref{thm:estimate-regularform}, we rely on the two main results that we have shown so far, Theorem \ref{thm:estimate-hermite} and Theorem \ref{thm:identifiability}, which are (informally) that  
\begin{itemize}
    \item We can estimate low-degree Hermite moment polynomials of $\mcl{M}$ to nearly optimal accuracy
    \item If two low-degree MPG functions are close on their low-degree Hermite moment polynomials then they are close in $L^1$ norm
\end{itemize}
To complete the learning algorithm, it remains to actually compute a low-degree MPG distribution that matches the Hermite moment polynomial estimates obtained from Theorem \ref{thm:estimate-hermite}.  Then, Theorem \ref{thm:identifiability} will imply that this MPG distribution is close to the density function of $\mcl{M}$ in $L^1$ norm.  

\subsection{Preliminary Computations}

We will first need a few preliminary definitions and computations. 
\begin{definition}\label{def:balanced}
We say a Gaussian $G$ is $\chi$-balanced if its covariance matrix $\Sigma$ satisfies
\[
\frac{1}{\chi}I \leq \Sigma \leq \chi I \,.
\]
If we have a mixture $\mcl{M} = w_1G_1 + \dots + w_kG_k$ such that all components are $\chi$-balanced, then we say that the mixture is $\chi$-balanced.
\end{definition}
Note that the parameter $\beta$ governs the balancedness of a mixture in $(\alpha, \beta, \gamma)$-regular form.  We will need the following few basic facts.

\begin{claim}\label{claim:parameter-dist}
For two Gaussians $N(\mu_1, \Sigma_1), N(\mu_2, \Sigma_2)$
\[
d_{\TV}(N(\mu_1, \Sigma_1), N(\mu_2, \Sigma_2)) = O\left( \left( (\mu_1 - \mu_2)^T\Sigma_1^{-1}(\mu_1 - \mu_2)\right)^{1/2} + \norm{\Sigma_1^{-1/2}\Sigma_2\Sigma_1^{-1/2} - I}_F\right)
\]
\end{claim}
\begin{proof}
See e.g. Fact 2.1 in \cite{kane2020robust}.
\end{proof}

The following two claims relate TV distance and parameter distance for balanced Gaussians.  The first deals with the case when the two Gaussians have very small overlap while the second deals with the case when the two Gaussians have very large overlap.

\begin{claim}\label{claim:balanced-dist1}
Let $G_1= N(\mu_1, \Sigma_1), G_2 = N(\mu_2, \Sigma_2)$ be $\chi$-balanced Gaussians. Let $\eps < 0.1$ be a parameter and assume that $d_{\TV}(G_1, G_2) \leq 1 - \eps$. Then the following two conditions hold:
\begin{enumerate}
    \item $\norm{\mu_{1} -  \mu_{2}} \leq O(\sqrt{\chi \log 1/\eps})$
    \item $\norm{\Sigma_1 - \Sigma_2}_2 \leq O(\chi \log^2 1/\eps)$
\end{enumerate}
\end{claim}
\begin{proof}
For the first claim, note that if we project onto the line connecting $\mu_1$ and $\mu_2$ then both distributions are Gaussian with standard deviation at most $\sqrt{\chi}$.  Thus, their means must be separated by at most $O(\sqrt{\chi \log 1/\eps})$.  For the second claim, we can use Lemma 3.2 in \cite{kane2020robust}.  Note that the covariance of the mixture $(G_1 + G_2)/2$ is
\[
\Sigma = \frac{\Sigma_1 + \Sigma_2}{2} + \frac{(\mu_1 - \mu_2)}{2}\frac{(\mu_1 - \mu_2)^T}{2}
\]
and the first part implies that 
\[
\Sigma \leq  O(\chi \log 1/\eps) I \,.
\]
Lemma 3.2 from \cite{kane2020robust} now immediately gives the desired bound.
\end{proof}

\begin{claim}\label{claim:balanced-dist2}
Let $G_1= N(\mu_1, \Sigma_1), G_2 = N(\mu_2, \Sigma_2)$ be $\chi$-balanced Gaussians. Let $\eps < 0.1$ be a parameter and assume that $d_{\TV}(G_1, G_2) \leq \eps$.  Then the following two conditions hold:
\begin{enumerate}
\item $\norm{\mu_{1} -  \mu_{2}} \leq O(\sqrt{\chi }\eps)$
\item $\norm{\Sigma_1 - \Sigma_2}_2 \leq \poly(\chi) \eps$
\end{enumerate}
\end{claim}
\begin{proof}
The first part follows from Theorem 1.2 in \cite{devroye2018total}.  To prove the second part, let $G_1' = N(\mu_2, \Sigma_1)$.  Note that by Claim \ref{claim:parameter-dist} and the first part,
\[
d_{\TV}(G_1', G_2) \leq d_{\TV}(G_1, G_2) + d_{\TV}(G_1, G_1') = O(\chi \eps)  \,.
\] 
Now, applying Theorem 1.1 from \cite{devroye2018total}, we deduce that  
\[
\norm{\Sigma_1 - \Sigma_2}_2 \leq \poly(\chi) \eps \,,
\]
completing the proof.
\end{proof}

\subsection{Main Proof of Theorem \ref{thm:estimate-regularform}}
Now we prove the first key lemma of this section.  This lemma implies that there exist polynomials that we can multiply in front of our rough component estimates $\ovl{G_1}, \dots , \ovl{G_k}$ in order to match the Hermite moment polynomials of the true mixture $\mcl{M}$.  Afterwards, we show how to solve for these polynomials.
\begin{lemma}\label{lem:low-degree-approx}
Let $\mcl{M} = w_1G_1 + \dots + w_kG_k$  be a mixture of Gaussians.  Let $c > 0$ be a (small) constant.  Assume that the mixture is in $(\alpha, \beta, \gamma)$-regular form where 
\begin{itemize}
    \item $\alpha = \poly(\log 1/\eps)$
    \item $\beta =   \poly(\log 1/\eps)$
    \item $\gamma$ is sufficiently small in terms of $k$ and $c$.
\end{itemize}

Let $G_j = N(\mu_j, I + \Sigma_j)$ for all $j \in [k]$.  Let $\ovl{G_1} = N(\wt{\mu}_1, I + \wt{\Sigma}_1)  , \dots , \ovl{G_k} =  N(\wt{\mu}_k, I + \wt{\Sigma}_k)$ be Gaussians such that 
\[
d_{\TV}(G_j, \ovl{G_j}) \leq \eps^{c} \,.
\]
Let $m$ be a parameter.  Then there exist polynomials $P_1, \dots , P_k$ in $d$ variables of degree at most $10/c$ such that the following holds: 
\begin{enumerate}
\item If we write the power series expansion of the function
\[
\sum_{j = 1}^k w_j e^{ \mu_j(X)y + \frac{1}{2}\Sigma_j(X)y^2} - \sum_{j=1}^k (w_j + P_j(Xy))e^{ \wt{\mu}_j(X)y + \frac{1}{2}\wt{\Sigma}_j(X)y^2} = \sum_{l=0}^{\infty} \frac{f_l(X)y^l}{l!}
\]
then 
\[
\norm{v(f_l(X))} \leq (2 + \alpha + \beta)^{O_{k,m,c}(1)} \eps
\]
for all $0 \leq l \leq m$.
\item For all $j$, $P_j(0) = 0$
\item For all $j$, $\norm{v_y(P_j(Xy))} \leq  \beta^{O_{k,m,c}(1)} \eps^{c}$
\end{enumerate}
\end{lemma}
\begin{proof}

For simplicity, assume that $10/c$ is an integer.  The modification to when $10/c$ is not an integer is straight-forward.  For $j \in [k]$, let
\[
g^{(j)}(X,y) = e^{ (\mu_j(X) - \wt{\mu}_j(X))y + \frac{1}{2}(\Sigma_j(X) - \wt{\Sigma}_j(X))y^2} = \sum_{l=0}^{\infty}\frac{g^{(j)}_l(X)y^l}{l!}
\]
where for the last expression, we expand the generating function as a power series in $y$.  Let 
\[
h^{(j)}(X,y) = \sum_{l=0}^{10/c }\frac{g^{(j)}_l(X)y^l}{l!}
\]
i.e. $h^{(j)}$ is obtained by truncating the power series expansion of $g^{(j)}$ to the first $10/c$ terms.  Note that in the RHS above, $g^{(j)}_l(X)$ is homogeneous in $X$ of degree $l$.  Thus, we can write $h^{(j)}(X,y)$ as a polynomial in the $d$-tuple of formal variables $Xy$.  We claim that setting $P_1, \dots , P_k$ such that 
\[
 P_j(Xy) = w_j(h^{(j)}(X,y) - 1)  = w_j\sum_{l=1}^{10/c }\frac{g^{(j)}_l(X)y^l}{l!}
\]
for all $j \in [k]$ suffices.  
\\\\
Note that the setting  trivially satisfies condition $2$.  Next, we check condition $3$.  First note that because $d_{\TV}(G_j, \ovl{G_j}) \leq \eps^{c}$, Claim \ref{claim:balanced-dist2} implies that 
\[
\norm{v(\mu_j(X) - \wt{\mu}_j(X))}, \norm{v(\Sigma_j(X) - \wt{\Sigma}_j(X))} \leq  \beta^{O(1)}\eps^c \,.
\]  
Thus, since
\[
e^{ (\mu_j(X) - \wt{\mu}_j(X))y + \frac{1}{2}(\Sigma_j(X) - \wt{\Sigma}_j(X))y^2} = 1 + \sum_{l = 1}^{\infty} \frac{\left(  (\mu_j(X) - \wt{\mu}_j(X))y + \frac{1}{2}(\Sigma_j(X) - \wt{\Sigma}_j(X))y^2\right)^l}{l!}
\]
we have
\begin{equation}\label{eq:exp-expansion}
\sum_{l=1}^{\infty}\frac{g^{(j)}_l(X)y^l}{l!} = \sum_{l = 1}^{\infty} \frac{\left(  (\mu_j(X) - \wt{\mu}_j(X))y + \frac{1}{2}(\Sigma_j(X) - \wt{\Sigma}_j(X))y^2\right)^l}{l!}
\end{equation}
and we can use Claim \ref{claim:polytensor-prod-bound} and the triangle inequality to verify condition $3$.
\\\\
Now we check condition $1$.  By combining (\ref{eq:exp-expansion}) with Claim \ref{claim:polytensor-prod-bound} and the triangle inequality, we get that for all $l$ with $10/c \leq  l \leq m$, 
\[
\norm{v(g^{(j)}_l(X))} \leq (2 + \beta)^{O_{m}(1)}\eps \,.
\]
Next write the power series expansion 
\[
e^{ \wt{\mu}_j(X)y + \frac{1}{2}\wt{\Sigma}_j(X)y^2} = \sum_{l=0}^{\infty} \frac{t^{(j)}_l(X)y^l}{l!} \,.
\]
Since the true mixture is in $(\alpha, \beta, \gamma)$-regular form and the components $\ovl{G_j}$ are $\eps^c$-close to the respective true components, we get that for all $l$ with $0 \leq l \leq m$,
\[
\norm{v(t^{(j)}_l(X))} \leq (2 + \alpha + \beta)^{O_{m}(1)} \,.
\]
Finally note that 
\[
w_j + P_j(Xy) = w_jh^{(j)}(X,y) \,.
\]

Now we may write
\begin{align*}
\sum_{j = 1}^k w_j e^{ \mu_j(X)y + \frac{1}{2}\Sigma_j(X)y^2} - \sum_{j=1}^k (w_j + P_j(Xy))e^{ \wt{\mu}_j(X)y + \frac{1}{2}\wt{\Sigma}_j(X)y^2} \\ = \sum_{j=1}^k w_j \left(\sum_{l= 10/c + 1}^{\infty} \frac{g^{(j)}_l(X)y^l}{l!} \right) e^{ \wt{\mu}_j(X)y + \frac{1}{2}\wt{\Sigma}_j(X)y^2} \,.
\end{align*}
However, since we only care about the first $m$ terms of the power series expansion, it suffices to consider the expression
\[
\sum_{j=1}^k w_j\left(\sum_{l= 10/c + 1}^{m} \frac{g^{(j)}_l(X)y^l}{l!} \right) \left( \sum_{l=0}^{m} \frac{t^{(j)}_l(X)y^l}{l!}\right) \,.
\]
Finally using our bounds on $\norm{v(g^{(j)}_l(X))}$ and $\norm{v(t^{(j)}_l(X))}$ from above, we immediately get the desired inequality.  Note that in the above we only dealt with the case when $10/c < m$.  If $10/c \geq m$, then the above argument implies that we can actually choose $P_1, \dots , P_k$ so that $f_0, \dots , f_m $ are all identically $0$.

\end{proof}

Using Lemma \ref{lem:low-degree-approx} we can now prove Theorem \ref{thm:estimate-regularform}.

\begin{proof}[Proof of Theorem \ref{thm:estimate-regularform}]
Let $m$ be a constant that will be set later as a sufficiently large function depending only on $k,c$.  Using Theorem \ref{thm:estimate-hermite}, we can obtain estimates $h_1', \dots , h_m'$ for the Hermite moment polynomials of the mixture $\mcl{M}$ such that 
\begin{equation}\label{eq:hermite-estimates}
\norm{v(h_j(X) - h_j'(X)} \leq (2+ \alpha + \beta + \theta^{-1} + \log 1/\eps)^{O_{k, m}(1)}\eps
\end{equation}
for all $j \leq m$ where $h_1, \dots , h_m$ denote the true Hermite moment polynomials of the mixture.  We will solve for constants $c_1, \dots , c_k$ and polynomials $P_1(X), \dots , P_k(X)$ in variables $X = (X_1, \dots , X_d)$ of degree at most $10/c$ such that the following properties hold:
\begin{enumerate}
\item If we write the following generating function as a formal power series
\[
(c_1 + P_1(Xy))e^{ \wt{\mu}_1(X) y + \frac{1}{2}\wt{\Sigma}_1(X)y^2} + \dots + (c_k + P_k(Xy))e^{ \wt{\mu}_k(X)y + \frac{1}{2}\wt{\Sigma}_k(X)y^2} = \sum_{j=0}^{\infty} \frac{f_j(X)y^j}{j!}
\]
then 
\begin{equation}\label{eq:hermite-constraints}
\norm{v(h_j'(X) - f_j(X))} \leq (2+ \alpha + \beta + \theta^{-1} +  \log 1/\eps)^{O_{k,m,c}(1)}\eps
\end{equation}
for all $j \leq m$. 

\item $c_j \geq \theta$ for all $j$

\item $P_j(0) = 0$ and $\norm{v(P_j(X))} \leq \beta^{O_{k,m,c}(1)}\eps^{c} $ for all $j \in [k]$
\end{enumerate}

Note that the expressions $h_j'(X) - f_j(X)$ are linear in $c_1, \dots , c_k$ and the coefficients of $P_1, \dots , P_k$ so if the system is feasible, then we can find a solution efficiently because all constraints are convex.  Note that we view it as a system where the indeterminates that we are solving for are exactly $c_1, \dots , c_k$ and the coefficients of $P_1, \dots , P_k$.

To see that the system is feasible, it suffices to combine (\ref{eq:hermite-estimates}) with Lemma \ref{lem:low-degree-approx} and note that by Corollary \ref{coro:hermite-of-mixture} 
\[
\sum_{j=1}^k w_je^{ \mu_j(X)y + \frac{1}{2}\Sigma_j(X)y^2} = \sum_{j=0}^{\infty} \frac{h_j(X)y^j}{j!} \,.
\]
Thus, by solving the system given by (\ref{eq:hermite-constraints}), we can assume that we found a valid solution $c_1, \dots , c_k,  P_1(X), \dots , P_k(X)$.  To complete the proof, we now show that from any solution to (\ref{eq:hermite-constraints}), we can construct an MPG distribution $f$ that is close to the density function of the mixture.
\\\\
By choosing $m$ sufficiently large in terms of $k,c$, we can now apply Theorem \ref{thm:hermite-to-tv-full} on the following generating function
\[
\sum_{j=1}^k (c_j + P_j(Xy))e^{ \wt{\mu}_j(X) y + \frac{1}{2}\wt{\Sigma}_j(X)y^2} - \sum_{j=1}^k w_je^{ \mu_j(X)y + \frac{1}{2}\Sigma_j(X)y^2} \,.
\]
To verify the conditions of the theorem, we can combine (\ref{eq:hermite-estimates}, \ref{eq:hermite-constraints}) and note that the polynomials $P_1(X), \dots , P_k(X)$ have degree at most $10/c$.  We deduce that if we let
\[
T(X) = \sum_{j=1}^k (c_j + P_j(iX))e^{i \wt{\mu}_j(X)  - \frac{1}{2}\wt{\Sigma}_j(X)} - \sum_{j=1}^k w_je^{i \mu_j(X)  - \frac{1}{2}\Sigma_j(X)}
\]
then 
\[
\norm{\chi T}_1 \leq (2+ \alpha + \beta + \theta^{-1} +  \log 1/\eps)^{O_{k,c}(1)}\eps \,.
\]
However note that by Claim \ref{claim:invchar-formula},
\[
\chi \left( \sum_{j=1}^k w_je^{i \mu_j(X)  - \frac{1}{2}\Sigma_j(X)}\right)
\]
is exactly the density function of $\mcl{M}$.  Thus, setting 
\[
f_0 = \chi \left(\sum_{j=1}^k (c_j + P_j(iX))e^{i \wt{\mu}_j(X)  - \frac{1}{2}\wt{\Sigma}_j(X)} \right)
\]
achieves that 
\[
\norm{\mcl{M}(x) - f_0(x)}_1 \leq (2 + \alpha + \beta + \theta^{-1} + \log 1/\eps)^{O_{k,c}(1)}\eps \,.
\]
Note that by applying Claim \ref{claim:invchar-of-polynomial}, we get that $f_0$ is a valid low-degree MPG function.  However, it is not necessarily a distribution.  We now show how to make minor modifications to $f_0$ to transform it into a distribution. Note that by Claim \ref{claim:invchar-of-polynomial}  and condition $3$ in the system that we solved for the $P_j$, the function $f_0$ defined above can be written in the form 
\[
f_0(x) = (c_1 + R_1(x))\tilde{G}_1(x) + \dots + (c_k + R_k(x))\tilde{G}_k(x)
\]
where for all $j \in [k]$, $R_j$ is a polynomial with real coefficients and degree at most $10/c$ and
\[
\E_{x \sim G_j}\left[ (R_j(x))^2 \right] \leq (2 + \beta)^{O_{k,c}(1)}\eps^{2c} \,.
\]
Note that the $R_j$ have real coefficients.  Let $B$ be an integer such that $B > 10/c$.  Note that for all $x \in \R^d$, by the AM-GM inequality,
\[
1 + \frac{R_j(x)}{c_j} +  \left(\frac{R_j(x)}{c_j} \right)^{2B} \geq 0 \,.
\]
Now let
\[
f_1(x) = \sum_{j=1}^k \left(c_j + R_j(x) +  \frac{R_j(x)^{2B}}{c_j^{2B - 1}}\right) \tilde{G}_j(x) \,.
\]
We verify that $f_1$ is close to $\mcl{M}$ in $L^1$ norm.  Note that using hypercontractivity  (Claim \ref{claim:hypercontractivity}), we have
\[
\norm{f_1 - f_0}_1 \leq \theta^{-(2B - 1)}\sum_{j = 1}^k \E_{x \sim G_j}\left[ R_j(x)^{2B}\right] \leq (\theta^{-1})^{O_{k,c}(1)} \sum_{j=1}^k \left(\E_{x \sim G_j}\left[ (R_j(x))^2 \right]\right)^{B} \leq (2 + \beta + \theta^{-1})^{O_{k,c}(1)}\eps \,.
\]
Thus,
\[
\norm{\mcl{M}(x) - f_1(x)}_1 \leq (2 + \alpha + \beta + \theta^{-1} + \log 1/\eps)^{O_{k,c}(1)}\eps \,.
\]
Finally, note that the above implies that 
\[
1 - (2 + \alpha + \beta + \theta^{-1} + \log 1/\eps)^{O_{k,c}(1)}\eps \leq  \int_{\R^d}f_1(x)dx \leq 1 + (2 + \alpha + \beta + \theta^{-1} + \log 1/\eps)^{O_{k,c}(1)}\eps
\]
and furthermore, we can explicitly compute the integral $ \int_{\R^d}f_1(x) dx$ in polynomial time so letting 
\[
f = \frac{f_1(x)}{\int_{\R^d}f_1(x) dx}
\]
yields a degree-$O_{k,c}(1)$ MPG distribution such that 
\[
\norm{\mcl{M}(x) - f(x)}_1 \leq (2 + \alpha + \beta + \theta^{-1} + \log 1/\eps)^{O_{k,c}(1)}\eps \,,
\]
which completes the proof.
\end{proof}

\section{Full Algorithm}

Now we are ready to complete our full learning algorithm.  A high-level description of our full algorithm is given below.  Our algorithm consists of the following steps.  First, we show that we can estimate the components of the mixture to $\eps^{\Omega(1)}$ accuracy by modifying the techniques in \cite{liu2020settling}.  Then using these estimates, we cluster into submixtures such that the clustering is $\wt{O}(\eps)$-accurate.  Finally, for each submixture, we argue that we can compute a linear transformation that places it in regular form and then apply Theorem \ref{thm:estimate-regularform} to estimate its density function.   Since we will enumerate over many possible candidate clusterings, we will end up with many possible candidate density functions so the last step involves a hypothesis test to select a candidate that is indeed close to the true distribution in TV distance.

\begin{algorithm}[H]
\caption{{\sc Full Algorithm} }
\begin{algorithmic} 
\State \textbf{Input:} $\eps$-corrupted sample $X_1, \dots , X_n$ from mixture of Gaussians $\mcl{M} = w_1G_1 + \dots + w_kG_k$
\State {\sc Learn Parameters to $\eps^{\Omega(1)}$ accuracy}
\For {each set of candidate components $\ovl{G_1}, \dots \ovl{G_k}$}
\State Assign samples to components according to maximum likelihood to form sets of samples $\{\ovl{S_1}, \dots , \ovl{S_k} \}$
\For {all partitions of $[k]$ into sets $R_1, \dots , R_l$  }
\State {\sc Learn Mixture to $\tilde{O}(\eps)$ accuracy} on samples $\mcl{R}_j = \cup_{i \in R_j}\ovl{S_i}$ with initial estimates $\{\ovl{G_i} \}_{i \in R_j}$
\EndFor
\State Compute density estimate by combining over all of $R_1, \dots , R_l$ and guessing weights of each submixture
\EndFor
\State Hypothesis test over all candidate density estimates to output an estimate $f$ that is $\tilde{O}(\eps)$-close to $\mcl{M}$ 
\end{algorithmic}
\label{alg:full}
\end{algorithm}

The algorithm { \sc Learn Mixture to $\tilde{O}(\eps)$ accuracy} requires that the components of the mixture are not too far from each other in TV distance (so that there exists a transformation that puts the mixture in regular form) and also requires initial estimates $\ovl{G_1}, \dots , \ovl{G_k}$ for the component Gaussians that are $\eps^{\Omega(1)}$ close to the true components in TV distance so that we can apply Theorem \ref{thm:estimate-regularform}.  We will show that the clustering step ensures the first property. The initial estimates are simply obtained from the output of the first step of {\sc Learn Parameters to $\eps^{\Omega(1)}$ accuracy}.

\begin{algorithm}[H]
\caption{{\sc Learn Mixture to $\tilde{O}(\eps)$ accuracy} }
\begin{algorithmic} 
\State \textbf{Input:} $\eps$-corrupted sample $X_1, \dots , X_n$ from mixture of Gaussians $\mcl{M} = w_1G_1 + \dots + w_kG_k$ such that $d_{\TV}(G_i, G_j) \leq 1 - \eps^{O(1)}$ for all $i \neq j$
\State \textbf{Input:} Initial estimates $\ovl{G_1}, \dots , \ovl{G_k}$ such that for all $i \in [k]$,
\[
d_{\TV}(\ovl{G_i}, G_i) \leq \eps^{\Omega(1)} \,.
\]
\State Let $\wt{\mu}_1, \wt{\Sigma}_1 $ be the mean and covariance of $\ovl{G_1}$.  
\State Apply the transformation $X_i \rightarrow \wt{\Sigma}_1^{-1/2}(X_i - \wt{\mu}_1)$ to the datapoints
\State Use Theorem \ref{thm:estimate-regularform} on the transformed data to compute a density estimate $f$
\State Output density estimate $f(\wt{\Sigma}_1^{-1/2}(x - \wt{\mu}_1)) \cdot \det(\Sigma_1)^{-1/2}$
\end{algorithmic}
\end{algorithm}

The main theorem of this paper, which we prove in this section, is stated below. 
\begin{theorem}\label{thm:main}
Let $\mcl{M} = w_1G_1 + \dots + w_kG_k$ be a $\chi$-balanced mixture of Gaussians (recall Definition \ref{def:balanced}).  Furthermore, assume that all of the mixing weights are at least $A^{-1}$ for some constant $A$.  Assume that $\eps$ is sufficiently small compared to $k,A$ and $\chi \leq \poly(\log \eps)$.  Let $n = \poly_{k,A}(d/\eps)$ for some sufficiently large polynomial and let $X_1, \dots , X_n$ be an $\eps$-corrupted sample from $\mcl{M}$.  Then {\sc Full Algorithm} runs in $\poly_{k,A}(d/\eps)$ time and with $0.9$ probability, outputs a degree $O_{k,A}(1)$ MPG distribution $f$ such that 
\[
\norm{f(x) - \mcl{M}(x)}_1 \leq (2 + \log 1/ \eps + \chi)^{O_{k,A}(1)} \eps \,.
\]
\end{theorem}

\subsection{Estimate Components to $\eps^{\Omega(1)}$ Accuracy}\label{sec:poly-eps-acc}

First, we will estimate the components of the mixture to $\eps^{\Omega(1)}$ accuracy.  This can be done with a few simple modifications to the techniques in \cite{liu2020settling}.  

\begin{theorem}\label{thm:list-learning-poly-acc}
Let $k, A > 0$ be constants.  There is a sufficiently large function $G$ and a sufficiently small function $g$ depending only on $k,A$ such that given an $\eps$-corrupted sample $X_1, \dots , X_n$ from a mixture of Gaussians $\mcl{M} = w_1G_1 + \dots + w_kG_k \in \R^d$ where $\eps < g(k,A)$, the $w_i$ are all at least $A^{-1}$, and $n \geq (d/\eps)^{G(k,A)}$, there is an algorithm that runs in time $\poly(n)$ and with $0.999$ probability, outputs a set of $(1/\eps)^{O_{k,A}(1)}$ candidate mixtures such that for at least one of these candidates,  $\{\widetilde{w_1}, \ovl{G_1}, \dots , \widetilde{w_k}, \ovl{G_k} \}$, we have
\[
|w_i - \widetilde{w_i}| + d_{\TV}(G_i, \ovl{G_i}) \leq \eps^{g(k,A)}
\]
for all $i \in [k]$.
\end{theorem}

\subsubsection{Achieving Constant Accuracy}
First, we will show that we can estimate the components to constant accuracy.  Recall the following result from \cite{liu2020settling}.

\begin{lemma}\label{lem:previous}[Lemma 7.5 in \cite{liu2020settling}]
Let $k, A, b > 0$ be constants and $\theta$ be a desired accuracy.  There is a sufficiently large function $G$ and a sufficiently small function $g$ depending only on $k, A , b, \theta$ such that given an $\eps$-corrupted sample $X_1, \dots , X_n$ from a mixture of Gaussians $\mcl{M} = w_1G_1 + \dots + w_kG_k \in \R^d$ where
\begin{itemize}
    \item The $w_i$ are all at least $A^{-1}$ for some constant $A$
    \item $d_{\TV}(G_i, G_j) \geq b$
\end{itemize}
and 
\begin{itemize}
    \item $ \eps < g(k,A,b, \theta)$
    \item $n \geq (d/\eps)^{G(k,A,b, \theta)}$
\end{itemize}
then there is an algorithm that runs in time $\poly(n)$ and with $0.999$ probability outputs a set of $(1/\theta)^{G(k,A,b, \theta)}$ candidate mixtures at least one of which satisfies
\begin{align*}
\max \left( d_{\TV} (\ovl{G_1}, G_1) , \dots , d_{\TV} (\ovl{G_k}, G_k) \right) \leq \theta \\
\max \left( |\widetilde{w_1} - w_1|, \dots , |\widetilde{w_k} - w_k| \right) \leq \theta
\end{align*}
\end{lemma}
\begin{remark}
Lemma 7.5 in \cite{liu2020settling} is stated with an assumption that the $w_i$ have bounded fractionality with denominator at most $A$.  The modification given in Section 6.3 in \cite{liu2020settling} (namely Theorem 6.12) immediately allows us to remove the bounded fractionality part of the assumption.
\end{remark}

Fix $k,A$.  Note that we would like to eliminate the assumption that $d_{\TV}(G_i, G_j) \geq b$ in the above result.  In fact, for achieving constant accuracy, this is not difficult to do because we can simply find some scale for which we can lump together components whose TV distance is too small and otherwise the components will be sufficiently separated.
\begin{corollary}\label{corollary:const-acc-learning}
Let $k,A$ be constants.  Let $\theta$ be some desired accuracy.  There is a sufficiently large function $F$ and a sufficiently small function $f$ depending only on $k,A,\theta$ (with $F(k,A, \theta),f(k,A, \theta) > 0$) such that the following holds.  Given an $\eps$-corrupted sample $X_1, \dots , X_n$ from a mixture of Gaussians $\mcl{M} = w_1G_1 + \dots + w_kG_k \in \R^d$ where 
\begin{itemize}
    \item The $w_i$ are all at least $A^{-1}$ for some constant $A$
    \item $\eps < f(k,A, \theta)$
    \item $n \geq (d/\eps)^{F(k,A, \theta)}$
\end{itemize} 
then there is an algorithm that runs in time $\poly(n)$ and with $0.999$ probability outputs a set of $(1/\theta)^{F(k,A,\theta)}$ candidate mixtures at least one of which satisfies 
\begin{align*}
\max \left( d_{\TV} (\ovl{G_1}, G_1) , \dots , d_{\TV} (\ovl{G_k}, G_k) \right) \leq \theta \\
\max \left( |\widetilde{w_1} - w_1|, \dots , |\widetilde{w_k} - w_k| \right) \leq \theta
\end{align*}
\end{corollary}
\begin{proof}
Let $g$ and $G$ be the functions in Lemma \ref{lem:previous}.  Fix $k,A, \theta$.  We define $h(b) = 0.1 \theta  g(k,A,b, 0.5\theta)/k$.  Consider the sequence 
\[
\theta \rightarrow h(\theta) \rightarrow \dots \rightarrow h^{(k^2)}(\theta) \,.
\]
There must be some $j < k^2$ such that no pair of true components has TV distance between $h^{(j)}(\theta)$ and $h^{j+1}(\theta)$.  Now consider the graph $\mcl{G}$ on $[k]$ where two indices $i_1,i_2$ are connected if and only if 
\[
d_{\TV}(G_{i_1}, G_{i_2}) \leq h^{j+1}(\theta) \,.
\]
Now consider a modified mixture $\mcl{M'}$ where for each connected component of vertices in $\mcl{G}$, say $S \subset [k]$, we replace all of the Gaussians $G_i$ with $i \in S$ with copies of one fixed representative from this component.  We then combine all of these copies by adding the mixing weights.  The resulting mixture $\mcl{M'}$ satisfies the following properties
\begin{itemize}
    \item There are $k' \leq k$ components
    \item All mixing weights are at least $A^{-1}$
    \item All pairs of components are separated by at least $h^{(j)}(\theta)$ in TV distance
    \item
    \[
    d_{\TV}(\mcl{M}, \mcl{M'}) \leq 0.1 \theta g(k,A, h^{(j)}(\theta), 0.5\theta) 
    \]
\end{itemize}
In particular, we can treat our samples as an $\eps'$-corrupted sample from $\mcl{M'}$ with 
\[
\eps' = \eps + 0.1 \theta g(k,A, h^{(j)}(\theta),0.5 \theta) \,.
\]
Now as long as $f,F$ are chosen appropriately, we can apply Lemma \ref{lem:previous} to learn the components of the mixture $\mcl{M'}$ to accuracy $0.5\theta$.  Finally, we can simply guess the mixing weights and duplications in our list of candidates to ensure that one of our candidate mixtures is component-wise within $\theta$ of the true mixture.
\end{proof}

\subsubsection{Achieving $\eps^{\Omega(1)}$-accuracy}

Similar to \cite{liu2020settling}, once we obtain constant accuracy estimates for the components, we then try to refine these estimates.  To do this, we first cluster the datapoints into submixtures by assigning each datapoint to the submixture that assigns it the highest probability.  We would like the following two properties
\begin{itemize}
    \item The clustering is $1 - \eps^c$-accurate
    \item Components within a submixture have TV distance at most $1 - \eps^{c'}$
\end{itemize}
where $c'$ is sufficiently small compared to $c$.  Once we have clustered the datapoints into such submixtures, we can learn the components of each submixture to $\eps^{\Omega(1)}$ accuracy, again following the same outline as the algorithm in \cite{liu2020settling}.

First we show that the clustering can be done accurately.  We need the following basic results from \cite{liu2020settling}.
\begin{lemma}[Lemma 7.2 from \cite{liu2020settling}]\label{lem:gaussian-triangle-ineq}
 Let $A,B,C$ be Gaussian distributions.  Assume that $d_{\TV}(A,B) \leq 0.9$. There is a universal constant $c > 0$ such that if $d_{\TV}(A,C) \geq 1 - \eps$ and $\eps < c$ then 
\[
d_{\TV}(B,C) \geq 1 - \eps^c \,.
\]
\end{lemma}
\begin{lemma}[Lemma 7.4 from \cite{liu2020settling}]\label{lem:gaussian-ratios}
Let $A$ and $B$ be two Gaussians with $d_{\TV}(A,B) \leq 0.9$.  There is a universal constant $c > 0$ such that if $D \in \{A, B \}$ and $\eps < c$ then
\[
P_{x \sim D}\left[ \eps \leq \frac{A(x)}{B(x)} \leq \frac{1}{\eps} \right] \geq 1 - \eps^c \,.
\]
\end{lemma}

Similar to \cite{liu2020settling}, we will also need VC-dimension bounds on the hypothesis class formed by comparing two MPG functions.  For the clustering step, we only need to deal with actual mixtures of Gaussians, but we will need the VC-dimension bound for MPG functions later on when we do hypothesis testing so we state the full result here.  First we need a definition.  

\begin{definition}\label{def:yatracos}
Let $\mcl{F}$ be a family of functions on some domain $\mcl{X}$.  Let $\mcl{H}_{\mcl{F},a}$ be the set of functions of the form $f_{\mcl{M}_1, \mcl{M}_2,\dots , \mcl{M}_a}$ where $\mcl{M}_1, \mcl{M}_2, \dots , \mcl{M}_a \in \mcl{F}$ and 
\begin{align*}
f_{\mcl{M}_1, \mcl{M}_2,\dots , \mcl{M}_a}(x) = 
\begin{cases}
1 \text{ if } \mcl{M}_1(x) \geq \mcl{M}_2(x), \dots , \mcl{M}_a(x) \\
0 \text{ otherwise}
\end{cases}   
\end{align*}
\end{definition}

 The VC dimension bound below is a direct consequence of the work in \cite{anthony2009neural}.
\begin{lemma}[Theorem 8.14 in \cite{anthony2009neural}]\label{lem:VCdim}
Let $\mcl{F}_{k,m}$ be the family of functions in $\R^d$ that are a degree $m$ MPG function with at most $k$ components.  Then the VC dimension of $\mcl{H}_{\mcl{F}_{k,m},a}$ is $\poly(d,a,m, k)$.
\end{lemma}

It is a standard result in learning theory that for a hypothesis class with bounded VC dimension, taking a polynomial number of samples suffices to get a good approximation for all hypotheses in the class.
\begin{lemma}[\cite{vapnik2015uniform}]\label{lem:uniformconvergence}
Let $\mcl{H}$ be a hypothesis class of functions from some domain $\mcl{X}$ to $\{0,1 \}$ with VC dimension $V$.  Let $\mcl{D}$ be a distribution on $\mcl{X}$.  Let $\eps, \delta >0$ be parameters.  Let $S$ be a set of $n = \poly(V, 1/\eps,  \log 1/\delta)$ i.i.d samples from $\mcl{D}$.  Then with $1 - \delta$ probability, for all $f \in \mcl{H}$
\[
\left \lvert \E_{x \sim S}[f(x)] - \E_{x \sim \mcl{D}}[f(x)]\right \rvert \leq \eps \,.
\]
\end{lemma}

Now we prove that given constant-accuracy component estimates, we can find an accurate clustering.
\begin{lemma}\label{lem:find-good-clusters}
Let $\mcl{M} = w_1G_1 + \dots + w_kG_k \in \R^d$ be a mixture of Gaussians where the $w_i$ are all at least $A^{-1}$ for some constant $A$.  There exists a sufficiently small function $g(k,A)> 0$ depending only on $k,A$ such that the following holds.   Let $X_1, \dots , X_n$ be an $\eps$-corrupted sample from the mixture $\mcl{M}$ where $\eps < g(k,A)$ and $n = \poly_{k,A}(d/\eps)$ for some sufficiently large polynomial.  Let $S_1, \dots , S_k \subset \{X_1, \dots, X_n \}$ denote the sets of samples from each of the components $G_1, \dots , G_k$ respectively.  Let $R_1, \dots , R_l$ be a partition of $[k]$ such that for $i_1 \in R_{j_1}, i_2 \in R_{j_2}$ with $j_1 \neq j_2$,
\[
d_{\TV}(G_{i_1}, G_{i_2}) \geq 1 - \eps'
\]
where $  \eps' \leq g(k,A)$.  Let $\ovl{G_1}, \dots , \ovl{G_k}$ be any Gaussians such that $d_{\TV}(G_i, \ovl{G_i}) \leq g(k,A)$ for all $i$.    Let $\ovl{S_1}, \dots , \ovl{S_k} \subset \{X_1, \dots , X_n \}$ be the subsets of samples obtained by assigning each sample to the component $\ovl{G_i}$ that gives it the maximum likelihood.  Then there is a universal constant $\eta > 0$ such that with probability  at least $0.999$,
\[
\left| \left(\cup_{i \in R_j}S_i\right)  \cap \left(\cup_{i \in R_j} \ovl{S_i}\right) \right| \geq (1- O_{k,A}(1)\eps - \eps'^{\eta}) \max\left( \left|\left(\cup_{i \in R_j}S_i\right) \right|, \left|\left(\cup_{i \in R_j}\ovl{S_i}\right) \right| \right)
\]
for all $j \in [l]$.
\end{lemma}
\begin{proof}
We will upper bound the expected number of uncorrupted points that are mis-classified for each $j \in [l]$ and then use the VC dimension bound in Lemma \ref{lem:VCdim} and the uniform convergence guarantee in Lemma \ref{lem:uniformconvergence} to argue that for any initial estimates $\ovl{G_1}, \dots , \ovl{G_k}$, the fraction of points that we mis-classify is small.  The expected number of uncorrupted points can be upper bounded by
\[
\sum_{j_1 \neq j_2} \sum_{\substack{i_1 \in R_{j_1} \\ i_2 \in R_{j_2}}} \int 1_{\ovl{G_{i_1}}(x) > \ovl{G_{i_2}}(x)} d G_{i_2}(x) \,.
\]

Clearly we can ensure $d_{\TV}(G_i, \ovl{G_i}) \leq 1/2$.  Thus, by Lemma \ref{lem:gaussian-triangle-ineq} and the assumption about $R_1, \dots , R_l$, $d_{\TV}(\ovl{G_{i_1}}, G_{i_2}) \geq 1 - \eps'^{\Omega(1)}$ for all $G_{i_2}$ where $i_2$ is not in the same piece of the partition as $i_1$.  Let $c$ be such that 
\[
d_{\TV}(\ovl{G_{i_1}}, G_{i_2}) \geq 1 - \eps'^c \,.
\]
By Lemma \ref{lem:gaussian-ratios}, 
\[
\Pr_{x \in G_{i_2}}\left[  \eps'^{c/2} \leq \frac{\ovl{G_{i_2}}(x)}{G_{i_2}(x)} \leq \eps'^{c/2} \right] \geq 1 - \eps'^{\Omega(c)}
\]
and combining the above two inequalities, we deduce
\[
\int 1_{\ovl{G_{i_1}}(x) > \ovl{G_{i_2}}(x)} d G_{i_2}(x) \leq \eps'^{\Omega(1)} \,.
\]
Since we are only summing over $O_k(1)$ pairs of components, as long as $\eps'$ is sufficiently small compared to $k, A$, the expected fraction of misclassified uncorrupted points is $\eps'^{\Omega(1)}$.   
\\\\
It remains to note that the clustering depends only on the comparisons between the values of the pdfs of the Gaussians $\ovl{G_1}, \dots , \ovl{G_k}$ at each of the samples $X_1, \dots , X_n$.  Since $n = \poly_{k,A}(d/\eps)$ for some sufficiently large polynomial, applying Lemma \ref{lem:VCdim} and Lemma \ref{lem:uniformconvergence} completes the proof (note that the fraction of corrupted points is at most $\eps$ overall and the mixing weights are lower bounded so the fraction of corrupted points in each cluster is  at most $O_{k,A}(1)\eps$).
\end{proof}

Finally, it remains to note that once we have clustered the points, we can learn the parameters.  For this, we rely on the following definition and result from \cite{liu2020settling}.

\begin{definition}[Definition 5.1 in \cite{liu2020settling}]\label{def:well-conditioned}
We say a mixture of Gaussians $w_1G_1 + \dots  + w_kG_k$ is $\delta$-tight if
\begin{enumerate}
    \item Let $\mcl{G}$ be the graph on $[k]$ obtained by connecting two nodes $i,j$ if $d_{\TV}(G_i, G_j) \leq 1 - \delta$.  Then $\mcl{G}$ is connected
    \item $d_{\TV}(G_i, G_j) \geq \delta $ for all $i \neq j$
    \item $w_{\min} \geq \delta$
\end{enumerate}
\end{definition}

\begin{theorem}[Theorem 5.2 in \cite{liu2020settling}]\label{thm:full-close-case}
There is a function $f(k) > 0$ depending only on $k$ such that given an $\eps$-corrupted sample from a $\delta$-tight mixture of Gaussians 
\[
\mcl{M} = w_1N(\mu_1, \Sigma_1) + \dots + w_kN(\mu_k, \Sigma_k)
\]
where $ \delta \geq \eps^{f(k)}$, there is a polynomial time algorithm that outputs a set of $(1/\eps)^{O_k(1)}$ candidate mixtures $\{  \widetilde{w_1}N(\widetilde{\mu_1}, \widetilde{\Sigma_1}) + \dots + \widetilde{w_k}N(\widetilde{\mu_k}, \widetilde{\Sigma_k} \}$ and with high probability, at least one of them satisfies that for all $i$:
\[
\abs{w_i - \widetilde{w_i}} + d_{\TV}(N(\mu_i, \Sigma_i), N(\widetilde{\mu_i}, \widetilde{\Sigma_i})) \leq \eps^{\Omega_k(1)} \,.
\]
\end{theorem}

With one additional argument, we can remove the assumption that the components are $\delta$-separated in TV distance from the above.

\begin{corollary}\label{coro:full-close-case}
There is a function $g(k) > 0$ depending only on $k$ such that given an $\eps$-corrupted sample from a mixture of Gaussians 
\[
\mcl{M} = w_1N(\mu_1, \Sigma_1) + \dots + w_kN(\mu_k, \Sigma_k)
\]
where $\mcl{M}$ satisfies conditions 1 and 3 of Definition \ref{def:well-conditioned} for some $ \delta \geq \eps^{g(k)}$, there is a polynomial time algorithm that outputs a set of $(1/\eps)^{O_k(1)}$ candidate mixtures $\{  \widetilde{w_1}N(\widetilde{\mu_1}, \widetilde{\Sigma_1}) + \dots + \widetilde{w_k}N(\widetilde{\mu_k}, \widetilde{\Sigma_k} \}$ and with high probability, at least one of them satisfies that for all $i$:
\[
\abs{w_i - \widetilde{w_i}} + d_{\TV}(N(\mu_i, \Sigma_i), N(\widetilde{\mu_i}, \widetilde{\Sigma_i})) \leq \eps^{\Omega_k(1)} \,.
\]
\end{corollary}
\begin{proof}
Let $f(k)$ be the function in Theorem \ref{thm:full-close-case}.  Consider the sequence
\[
\psi_0 = \eps, \psi_1 = \eps^{(0.1f(k))}, \dots , \psi_{k^2} = \eps^{(0.1f(k))^{k^2}} \,.
\]
There must be an index $j < k^2$ such that there is no pair of components $G_{i_1}$ and $G_{i_2}$ with TV distance between $\psi_j$ and $\psi_{j+1}$.  Let $\mcl{G}$ be the graph on $[k]$ where two indices $i_1$ and $i_2$ are connected if and only if 
\[
d_{\TV}(G_{i_1}, G_{i_2}) \leq \psi_{j} \,.
\]
Now consider a modified mixture $\mcl{M'}$ where for each connected component of vertices in $\mcl{G}$, say $S \subset [k]$, we replace all of the Gaussians $G_i$ with $i \in S$ with copies of one fixed representative from this component.  We then combine the copies by adding the mixing weights.  The resulting mixture $\mcl{M'}$ satisfies the following properties
\begin{itemize}
    \item $\mcl{M'}$ has $k' \leq k$ components $G_1', \dots ,G_{k'}'$
    \item The graph on $[k']$ obtained by connecting two nodes $i,j$ if $d_{\TV}(G'_i, G'_j) \leq 1 - \delta + k\psi_{j}$ is connected
    \item All pairs of components are separated by at least $\psi_{j+1}$ in TV distance
    \item All mixing weights are at least $\delta$
    \item $ d_{\TV}(\mcl{M}, \mcl{M'}) \leq k\psi_{j}$
\end{itemize}
Thus, we can treat our sample as an $\eps'$-corrupted sample from $\mcl{M'}$ with
\[
\eps' = \eps + k\psi_{j} \,.
\]
Note that $\psi_{j+1} \geq \eps'^{f(k)} $ and thus by choosing $g(k)$ appropriately, we can ensure that the mixture $\mcl{M'}$ is $\eps'^{f(k)}$-tight.  Thus, we can apply Theorem \ref{thm:full-close-case} to learn the components of the mixture $\mcl{M'}$ to within $\eps'^{\Omega_k(1)} = \eps^{\Omega_k(1)}$.  We can then guess the mixing weights and duplications of components to output a list of candidate mixtures at least one of which is component-wise within $\eps^{\Omega_k(1)}$ of the true mixture.
\end{proof}

To put everything together, we will need the following simple claim which also appears in \cite{liu2020settling}.
\begin{claim}[Claim 7.6 in \cite{liu2020settling}]\label{claim:exists-good-partition}
Let $\mcl{M} = w_1G_1 + \dots + w_kG_k $ be a mixture of Gaussians.  For any constant $c > 0$ and parameter $\eps$, there exists a function  $f(c,k)$ such that there exists a partition (possibly trivial) of $[k]$ into sets $R_1, \dots , R_l$ such that 
\begin{itemize}
    \item If we draw edges between all $i,j$ such that $d_{\TV}(G_i, G_j) \leq 1 - \eps^{c \kappa}$
    then each piece of the partition is connected
    \item For any $i,j$ in different pieces of the partition $  d_{\TV}(G_i, G_j) \geq 1 - \eps^{\kappa}$
\end{itemize}
and $f(c,k) < \kappa < 1$.
\end{claim}

Now we can prove Theorem \ref{thm:list-learning-poly-acc}.

\begin{proof}[Proof of Theorem \ref{thm:list-learning-poly-acc}]
This follows from combining Corollary \ref{corollary:const-acc-learning}, Claim \ref{claim:exists-good-partition}, Lemma \ref{lem:find-good-clusters} and finally applying Corollary \ref{coro:full-close-case}.  Note we can choose the constant $c$ in Claim \ref{claim:exists-good-partition} sufficiently small so that when combined with Lemma \ref{lem:find-good-clusters}, the resulting accuracy that we get on each submixture is high enough that we can then apply Corollary \ref{coro:full-close-case} (we can treat the subsample corresponding to each submixture as a $O(\eps'^{\eta})$-corrupted sample from that submixture).  We apply Lemma \ref{lem:find-good-clusters} with $\eps' = \eps^{\kappa}$ where the $\kappa$ is obtained from Claim \ref{claim:exists-good-partition}.
\end{proof}

\subsection{Estimate Mixture to $\tilde{O}(\eps)$ accuracy}\label{sec:opt-acc}

Note that the results in the previous section do not require the Gaussians to be reasonably well-conditioned.  However for the next step, of going from $\eps^{\Omega(1)}$ accuracy to $\tilde{O}(\eps)$ accuracy, we will require the assumption that the Gaussians are reasonably well-conditioned.  We will restrict to the case where the Gaussians are $\chi$-balanced where $\chi$ should be thought of as a constant (or $\poly(\log 1/\eps)$).

The main theorem that we will prove in this section is as follows.

\begin{theorem}\label{thm:list-of-functions}
Let $\mcl{M} = w_1G_1 + \dots + w_kG_k$ be a $\chi$-balanced mixture of Gaussians.  Furthermore, assume that all of the mixing weights are at least $A^{-1}$ for some constant $A$.  Assume that $\eps$ is sufficiently small compared to $k,A$ and $\chi \leq \poly(\log 1/\eps)$.  Let $n = \poly_{k,A}(d/\eps)$ for some sufficiently large polynomial and let $X_1, \dots , X_n$ be an $\eps$-corrupted sample from $\mcl{M}$.  Then with $0.99$ probability, the list of candidate density estimates computed by {\sc Full Algorithm} is a list of size $(1/\eps)^{O_{k,A}(1)}$ and contains a degree $O_{k,A}(1)$ MPG distribution $f$ such that 
\[
\norm{f(x) - \mcl{M}(x)}_1 \leq (2 + \log 1/ \eps + \chi)^{O_{k,A}(1)} \eps \,.
\]
\end{theorem}
\begin{proof}
Let $g(k,A)$ be the function in Theorem \ref{thm:list-learning-poly-acc}.  Using Theorem \ref{thm:list-learning-poly-acc}, with $0.999$ probability, among the list of candidate components computed in {\sc Learn Parameters to $\eps^{\Omega(1)}$ Accuracy}, there is some set $\{\ovl{G_1}, \dots , \ovl{G_k} \}$ such that for all $i \in [k]$
\[
d_{\TV}(\ovl{G_i}, G_i) \leq \eps^{g(k,A)} \,.
\]

Now consider the graph $\mcl{G}$ on $[k]$ where two nodes $i$ and $j$ are connected if and only if $d_{\TV}(G_i, G_j) \leq 1 - \eps^{1/\eta}$ where $\eta$ is the universal constant in Lemma  \ref{lem:find-good-clusters}.  Let $R_1, \dots , R_l \subset [k]$ be the connected components in this graph.  For each $j \in [l]$ define the submixture
\[
\mcl{M}_j = \frac{\sum_{i \in R_j}w_iG_i}{\sum_{i \in R_j} w_i} \,.
\]

Now by Lemma  \ref{lem:find-good-clusters}, with $0.999$ probability, if we partition $[k]$ according to $R_1, \dots , R_l$ and then assign samples by maximum likelihood using the estimates $\{ \ovl{G_1}, \dots , \ovl{G_k} \}$, the resulting the subsets of samples $\mcl{R}_1, \dots , \mcl{R}_l$ are equivalent to $O_{k,A}(1)\eps$-corrupted samples from each of the submixtures  $\mcl{M}_1, \dots , \mcl{M}_l$ (note that we are setting $\eps' = \eps^{1/\eta}$ in Lemma \ref{lem:find-good-clusters}).  It now suffices to estimate the density function of each submixture to $\wt{O}(\eps)$ accuracy.  We will argue that given this clustering, {\sc Learn Mixture to $\wt{O}(\eps)$ accuracy} indeed learns each of the submixtures to $\wt{O}(\eps)$ accuracy.

Let $\mu_1, \Sigma_1, \dots \mu_k, \Sigma_k$ be the means and covariances of the true components $G_1, \dots , G_k$ and let $\wt{\mu}_1, \wt{\Sigma}_1, \dots , \wt{\mu}_k, \wt{\Sigma}_k$ be the means and covariances of our initial estimates.  Without loss of generality assume $\mcl{R}_1 = \{1,2, \dots , t \}$.    Let $L$ denote the linear transformation $x \rightarrow \wt{\Sigma}_{1}^{-1/2}(x - \wt{\mu}_1)$.  Since $d_{\TV}(\ovl{G}_i, G_i) \leq \eps^{g(k,A)}$, the Gaussians $\ovl{G}_i$ must all be $2\chi$-balanced.  Now by Claim \ref{claim:balanced-dist2}, we conclude 
\[
\norm{L(\mu_1)} + \norm{L(\Sigma_1) - I}_2 \leq \eps^{\Omega_{k,A}(1)} \,.
\]
Also since $\mcl{R}_1 = \{1,2, \dots , t \}$ is a connected a component in $\mcl{G}$, we can apply Claim \ref{claim:balanced-dist1} and sum over all paths in the connected component to deduce that 
\[
\norm{L(\mu_i)} ,  \norm{L(\Sigma_i) - I }_2 \leq \poly(\chi, \log 1/\eps)
\]
 for all $i \leq t$.  Thus, since $\eps$ is sufficiently small in terms of $k,A$, we can ensure that the transformed mixture $L(\mcl{M}_1)$ is in $(\alpha, \beta, \gamma)$-regular form for 
 \begin{itemize}
        \item $\alpha \leq \poly(\log 1/\eps)$
        \item  $\beta \leq \poly(\log 1/\eps)$
        \item $\gamma$ sufficiently small in terms of $k,A$
\end{itemize}
Now the above implies that the application of Theorem \ref{thm:estimate-regularform} in algorithm {\sc Learn Mixture to $\wt{O}(\eps)$ accuracy} is valid (with $c = \Omega_{k,A}(1)$) and guarantees that with high probability, after applying the inverse linear transformation, we obtain a function $f_1$ such that 
\[
\norm{f_1(x) - \mcl{M}_1(x)}_1 \leq (2 + \log 1/ \eps + \chi)^{O_{k,A}(1)} \eps \,.
\] 
Similarly, we compute functions $f_2, \dots , f_l$ that approximate the density functions of $\mcl{M}_2, \dots , \mcl{M}_l$.  Finally note that 
\[
\mcl{M} = \sum_{j=1}^l \left(\sum_{i \in R_j} w_i \right) \mcl{M}_j \,.
\]
We can simply guess the weights $\left(\sum_{i \in R_j} w_i \right)$ for $j \in [l]$ of the submixtures using an $\eps$-grid.  If our guesses, say $W_1, \dots , W_l$ are all within $\eps$ of the true values then the function $f = \sum_{j = 1}^l W_j f_j(x)$ satisfies the desired conditions because
\begin{align*}
\norm{\sum_{j=1}^l W_jf_j(x) -  \mcl{M}(x)}_1 \leq  \norm{\sum_{j=1}^l W_jf_j(x) -  \sum_{j=1}^l W_j \mcl{M}_j}_1  + \norm{\sum_{j=1}^l \left( W_j - \sum_{i \in R_j} w_i \right)  \mcl{M}_j}_1 \\ \leq (2 + \log 1/\eps + \chi)^{O_{k,A}(1)} \eps \,.
\end{align*}
Overall, the number of candidates that {\sc Full Algorithm} outputs is $(1/ \eps)^{O_{k,A}(1)}$  (enumerating over all initial estimates for the components, possible clusterings, and possible guesses for the weights) and with $0.99$ probability at least one of them satisfies
\[
\norm{f(x) - \mcl{M}(x)}_1 \leq (2 + \log 1/\eps + \chi)^{O_{k,A}(1)} \eps \,,
\]
completing the proof.
\end{proof}

\subsection{Hypothesis Testing}

Theorem \ref{thm:list-of-functions} guarantees that we can learn a list of candidate density functions at least one of which is close to the density function of the true mixture.  The last step is to hypothesis test to select an element of the list which is guaranteed to be close to the true density function.  For this we rely on the following lemma from \cite{liu2020settling}.
\begin{lemma}[Restated from \cite{liu2020settling}]\label{lem:hypothesis-test}
Let $\mcl{F}$ be a family of distributions on some domain $\mcl{X}$ with explicitly computable density functions that can be efficiently sampled from.  Let $V$ be the VC dimension of $\mcl{H}_{\mcl{F},2}$ (recall Definition \ref{def:yatracos}). Let $\mcl{D}$ be an unknown distribution in $\mcl{F}$.  Let $m$ be a parameter.  Let $X_1, \dots , X_n$ be an $\eps$-corrupted sample from $\mcl{D}$ with $n \geq  \poly(m,\eps, V)$ for some sufficiently large polynomial.  Let $H_1, \dots , H_m$ be distributions in $\mcl{F}$ given to us by an adversary with the promise that
\[
\min(d_{\TV}(\mcl{D}, H_i) ) \leq \eps \,.
\]
Then there exists an algorithm that runs in time $\poly(n, \eps)$ and outputs an $i$ with $1 \leq i \leq m$ such that with $0.999$ probability 
\[
d_{TV}(\mcl{D}, H_i) \leq O(\eps) \,.
\]
\end{lemma}

We can now complete the proof of the main theorem.
\begin{proof}[Proof of Theorem \ref{thm:main}]
Combining Theorem \ref{thm:list-of-functions} with Lemma \ref{lem:hypothesis-test} and Lemma \ref{lem:VCdim}, we immediately get the desired result.  Note that the density function of a constant degree MPG distribution is explicitly computable.  To see why it can be efficiently sampled from, note that all of the polynomials are always positive and the integral of $\int_{\R^d} P(x)G(x) dx$ for a polynomial $P$ and Gaussian $G$ can be computed explicitly using integration by parts.
\end{proof}

\bibliographystyle{alpha}
\bibliography{bibliography}

\appendix

\begin{center}
\Large{\textbf{Appendix}}
\end{center}
\section{Omitted Proofs from Section \ref{sec:tailbounds}}\label{appendix:tailbounds}

\begin{proof}[Proof of Claim \ref{claim:poly-gaussian-tail}]
First, we obtain high probability bounds on the sum of the largest $\eps$-fraction of the samples.  

For any $\alpha$, let $S_{\alpha}$ be the set of $x \in S$ with $|x| \geq \alpha  \log^{1/c}(1/\eps)$.  Let $C = (10/c)^{10/c} $. For $C \leq \alpha  \leq n^{0.01}$, using tail bounds on the binomial distribution, we get
\[
\Pr\left[ |S_{\alpha}| > \frac{\eps n}{\alpha^{10}} \right] \leq \binom{n}{\eps n/\alpha^{10}} \eps^{\alpha^{c} \cdot \frac{n \eps}{\alpha^{10}} } \leq \left(\frac{3\alpha^{10}}{\eps} \cdot \eps^{\alpha^{c}}\right)^{ \frac{n \eps}{\alpha^{10}}} \leq \frac{e^{-(10d/\eps)^2}}{\alpha^{10}} \,.
\]
Also, using a simple union bound, with probability at least $1 - e^{-(10d/\eps)^2}$, we have $|x| \leq n^{0.01}$ for all $x \in S$.

Consider a set of $\alpha$ forming a geometric series with ratio $2$, say $\{C, 2C, \dots ,  \}$.  Combining everything we've shown so far using a union bound, with probability at least $1 - e^{-(9d/\eps)^2}$, for all $\alpha = \{C, 2C, \dots , \}$, we have
\[
|S_{\alpha}| \leq \frac{\eps n}{\alpha^{10}} \,.
\]
This implies that with probability $1 - e^{-(9d/\eps)^2}$, for any set $T \subset S$ of size at most $\eps n$, 
\begin{equation}\label{eq:epstailbound}
\sum_{x \in T}|x| \leq \eps  C\log^{1/c}(1/\eps) n + \sum_{i=1}^{\infty} \frac{2^iC}{2^{10(i-1)}}\log^{1/c}(1/\eps) \eps n \leq 10\eps \log^{1/c}(1/\eps) C n \,.
\end{equation}

Let $\mcl{D}'$ be the distribution $\mcl{D}$ restricted to the interval $\left[ -\log^{1/c}(1/\eps), \log^{1/c}(1/\eps)\right] $. The assumption about the tail decay of $\mcl{D}$ implies that 
\begin{equation}\label{eq:meanbound}
| \mu_{\mcl{D'}} - \mu_{\mcl{D}}| \leq 10\eps \log^{1/c}(1/\eps) C \,.
\end{equation}
Let $\lambda$ be the probability mass that $\mcl{D}$ has in the interval $\left[ -\log^{1/c}(1/\eps), \log^{1/c}(1/\eps)\right] $.  Sampling from $\mcl{D}$ is equivalent to sampling using the following procedure.
\begin{itemize}
    \item First draw a Bernoulli random variable $\sigma$ that is $1$ with probability $\lambda$
    \item If $\sigma$ is $1$ then draw a sample from $\mcl{D'}$ and otherwise sample from the distribution $\mcl{D}$ restricted to outside the interval $\left[ -\log^{1/c}(1/\eps), \log^{1/c}(1/\eps)\right] $
\end{itemize}
Note that $\lambda \geq 1 - \eps$.  Assuming that the original set of samples $S$ is drawn in this manner, with probability at least $1 - e^{-(9d/\eps)^2}$, there are at least $(1-2\eps)n$ elements of $S$ that are drawn from $\mcl{D'}$.  By a Chernoff bound, their mean is within $10\eps \log^{1/c}(1/\eps) C $ of $\mu_{\mcl{D}'}$ with at least $1 - e^{-(9d/\eps)^2}$ probability.  Finally, using equations (\ref{eq:epstailbound}),  (\ref{eq:meanbound}) and the triangle inequality, we deduce that with probability at least $1 - e^{-(8d/\eps)^2}$, for all subset $S' \subset S$ of size at least $(1-\eps)n$,
\[
\left \lvert \mu_{\mcl{D}} -  \frac{1}{|S'|} \sum_{x \in S'} x   \right \rvert \leq \eps \log^{1/c} (1/\eps) \left(\frac{10^2}{c}\right)^{10/c} \,.
\]

\end{proof}

\begin{proof}[Proof of Claim \ref{claim:hermite-stability}]
Let $\kappa$ be the dimensionality of $v(H_m(X,z))$.  Note $\kappa = d^{O_{m}(1)}$.  Let $E$ be an $\eps/(10\kappa)$-net of the unit sphere in $\kappa$ dimensions.  We can ensure that 
\[
|E| \leq \left(\frac{10^2 \kappa }{\eps} \right)^\kappa  \,.
\]
For a fixed vector $v \in E$, we will compute the probability that either 
\[
\left \lvert v \cdot \left(\mu_{\mcl{D}} - \frac{1}{S'}\sum_{x \in S'} x \right)  \right \rvert    \geq  \frac{\delta}{2}
\]
or 
\[
\left \lvert v^T\left(\Sigma_{\mcl{D}} - \frac{1}{S'}\sum_{x \in S'}(x - \mu_{\mcl{D}}) (x - \mu_{\mcl{D}})^T   \right)v \right \rvert \geq  \frac{\delta^2}{2\eps} 
\]
for some subset $S'$ with $|S'| \geq (1-\eps)|S|$ and then we will union bound the failure probability over all vectors $v \in E$.

By Lemma \ref{lem:hermite-tail}, the distribution $v \cdot \mcl{D}$, scaled by a factor of 
\[
\frac{1}{(2 + \alpha + \beta)^{O_m(1)}} \,,
\]
satisfies the conditions of Claim \ref{claim:poly-gaussian-tail} with $c = \Omega_m(1)$.  Thus, by choosing $n$ sufficiently large, we can ensure that with $1 - e^{(8\kappa/\eps)^2}$ probability, we have
\[
\left \lvert v \cdot \left(\mu_{\mcl{D}} - \frac{1}{S'}\sum_{x \in S'} x \right)  \right \rvert    \leq  \frac{\delta}{2} 
\]
for all subsets $S'$ with $|S'| \geq (1-\eps)n$.  Next, 
\begin{align*}
&\left \lvert v^T\left(\Sigma_{\mcl{D}} - \frac{1}{S'}\sum_{x \in S'}(x - \mu_{\mcl{D}}) (x - \mu_{\mcl{D}})^T   \right)v \right \rvert  \\ &= \left \lvert v^T (\Sigma_{\mcl{D}} +\mu_{\mcl{D}} \mu_{\mcl{D}}^T )v - \frac{1}{S'}\sum_{x \in S'}(v \cdot x)^2 + \frac{2}{S'} v^T\left( \sum_{x \in S'}(x - \mu_{\mcl{D}})\right) \mu_{\mcl{D}}^T v  \right \rvert \\ & \leq \left \lvert v^T (\Sigma_{\mcl{D}} +\mu_{\mcl{D}} \mu_{\mcl{D}}^T )v - \frac{1}{S'}\sum_{x \in S'}(v \cdot x)^2 \right \rvert + 2\norm{\mu_{\mcl{D}}} \left \lvert v \cdot \left(\mu_{\mcl{D}} - \frac{1}{S'}\sum_{x \in S'} x \right)  \right \rvert  \,.
\end{align*}
Lemma \ref{lem:hermite-tail} implies that the variable $(v \cdot x)^2$ for $x \sim \mcl{D}$, scaled by a factor of 
\[
\frac{1}{(2 + \alpha + \beta)^{O_m(1)}} \,,
\]
also satisfies the conditions of Claim \ref{claim:poly-gaussian-tail} with $c = \Omega_m(1)$.  Also note that 
\[
\E_{x \sim \mcl{D}} (v \cdot x)^2 = v^T (\Sigma_{\mcl{D}} +\mu_{\mcl{D}} \mu_{\mcl{D}}^T )v \,.
\]
Thus, we can ensure that with $1 - e^{(8\kappa/\eps)^2}$ probability, we have
\[
\left \lvert v^T (\Sigma_{\mcl{D}} +\mu_{\mcl{D}} \mu_{\mcl{D}}^T )v - \frac{1}{S'}\sum_{x \in S'}(v \cdot x)^2  \right \rvert    \leq  \frac{\delta^2}{10\eps} \,,
\]
for all subsets $S'$ with $|S'| \geq (1-\eps)n$.  Lemma \ref{lem:hermite-tail} also implies that 
\[
\norm{\mu_{\mcl{D}}} \leq (2 + \alpha + \beta)^{O_m(1)}
\]
so overall, we conclude that with $1 - e^{(7\kappa/\eps)^2}$ probability, for each fixed vector $v$, we have both 
\[
\left \lvert v \cdot \left(\mu_{\mcl{D}} - \frac{1}{S'}\sum_{x \in S'} x \right)  \right \rvert    \leq  \frac{\delta}{2}
\]

\[
\left \lvert v^T\left(\Sigma_{\mcl{D}} - \frac{1}{S'}\sum_{x \in S'}(x - \mu_{\mcl{D}}) (x - \mu_{\mcl{D}})^T   \right)v \right \rvert \leq  \frac{\delta^2}{2\eps} 
\]
for all subsets $S'$ with $|S'| \geq (1-\eps)n$.  A union bound gives that with $1 - e^{(6\kappa/\eps)^2}$ probability, the above holds for all vectors $v \in E$.  Now let $b$ be the vector
\[
b = \left(\mu_{\mcl{D}} - \frac{1}{S'}\sum_{x \in S'} x \right)
\]
and let $Q$ be the matrix
\[
Q = \left(\Sigma_{\mcl{D}} - \frac{1}{S'}\sum_{x \in S'}(x - \mu_{\mcl{D}}) (x - \mu_{\mcl{D}})^T   \right) \,.
\]
We want to bound the $L^2$ norm of $b$ and the operator norm of $Q$.  First, let $v$ be the element of $E$ that is closest to the unit vector in the direction of $b$.  Since $E$ is a $\eps/(10\kappa)$-net, we must have $v \cdot b > 0.5\norm{b}_2$ which implies $\norm{b}_2 \leq \delta$.  Next let $u$ be a unit vector such that $u^TQu = \norm{Q}_{\textsf{op}}$ (such a $u$ exists since $Q$ is symmetric).  Let $v$ be the element of $E$ closest to $u$.  Then
\[
v^TQv = u^TQu + (v-u)^TQ u + v^TQ(v-u) \geq \norm{Q}_{\textsf{op}} - |(v-u)^TQ u| - |v^TQ(v-u)| \geq \norm{Q}_{\textsf{op}}(1 - 2\norm{v-u}_2) > \frac{\norm{Q}_{\textsf{op}}}{2}
\]
and thus we must actually have $\norm{Q}_{\textsf{op}}  \leq \delta^2/\eps$.  This completes the proof.
\end{proof}


\end{document}